\documentclass[sigconf]{acmart}
\AtBeginDocument{%
  \providecommand\BibTeX{{%
    \normalfont B\kern-0.5em{\scshape i\kern-0.25em b}\kern-0.8em\TeX}}}

\copyrightyear{2026}
\acmYear{2026}
\setcopyright{acmlicensed}
\acmConference[Conference acronym 'XX]{Make sure to enter the correct
  conference title from your rights confirmation emai}{June 03--05,
  2018}{Woodstock, NY}
\acmISBN{978-1-4503-XXXX-X/18/06}
\acmDOI{XXXXXXX.XXXXXXX}

\usepackage{amsfonts, amsmath, amsthm}
\usepackage[ruled, vlined, linesnumbered]{algorithm2e}
\usepackage{graphicx}
\usepackage{multirow}
\usepackage{nccmath}
\usepackage{subfig}
\usepackage{bm}
\usepackage{balance}
\usepackage{enumerate}
\usepackage{afterpage}
\usepackage{caption}
\usepackage{subcaption}
\usepackage{booktabs}
\usepackage{hyperref}
\usepackage{enumitem}
\usepackage[normalem]{ulem}
\useunder{\uline}{\ul}{}
\usepackage{adjustbox}

\newtheorem{definition}{Definition}

\newtheorem{lemma}{Lemma}

\newtheorem{theorem}{Theorem}
\newtheorem{fact}{Fact}
\newcounter{savelemma}

\newcommand{\etal}{\textit{et~al.}\xspace}
\newcommand{\eg}{\textit{e.g.}\xspace}
\newcommand{\ie}{\textit{i.e.}\xspace}
\newcommand{\aka}{\textit{a.k.a.}\xspace}

\newcommand\figref[1]{Fig.~\ref{#1}}
\newcommand\algref[1]{Algorithm~\ref{#1}}
\newcommand\tabref[1]{Tab.~\ref{#1}}
\newcommand\secref[1]{Sec.~\ref{#1}}
\newcommand\equref[1]{Eq.~(\ref{#1})}
\newcommand\thmref[1]{Theorem ~\ref{#1}}

\newcommand\lemmaref[1]{Lemma ~\ref{#1}}
\newcommand\factref[1]{Fact ~\ref{#1}}
\newcommand\appref[1]{Appendix~\ref{#1}}

\newcommand{\fakeparagraph}[1]{\vspace{1mm}\noindent\textbf{#1.}}

\ifodd 1
\newcommand{\TODO}[1]{\textbf{\color{red}{TODO: #1} }}

\else
\newcommand{\TODO}[1]{}

\fi

\newcommand{\sysname}{\textsf{GraphDLG}\xspace}

\begin{document}

\title{\sysname: Exploring Deep Leakage from Gradients in Federated Graph Learning}

\author{
Shuyue Wei$^{*}$}
\affiliation{%
  \department{C-FAIR \& School of Software}
  \institution{Shandong University}
  \city{Jinan}
  \country{China}
}
\email{weishuyue@sdu.edu.cn}

\author{Wantong Chen$^{*}$}
\affiliation{%
  \department{SKLCCSE Lab}
  \institution{Beihang University}
  \city{Beijing}
  \country{China}
}
\email{2306cwt@buaa.edu.cn}

\author{Tongyu Wei}
\affiliation{%
  \department{SKLCCSE Lab}
  \institution{Beihang University}
  \city{Beijing}
  \country{China}
}
\email{weitongyu@buaa.edu.cn}

\author{Chen Gong}
\affiliation{%
  \department{Automation \& Intelligent Sensing}
  \institution{Shanghai Jiao Tong
University}
  \city{Shanghai}
  \country{China}
}
\email{chen.gong@sjtu.edu.cn}

\author{Yongxin Tong}
\affiliation{%
  \department{SKLCCSE Lab}
  \institution{Beihang University}
  \city{Beijing}
  \country{China}
}
\email{yxtong@buaa.edu.cn}

\author{Lizhen Cui}
\affiliation{%
  \department{C-FAIR \& School of Software}
  \institution{Shandong University}
  \city{Jinan}
  \country{China}
}
\email{clz@sdu.edu.cn}




\begin{abstract}
Federated graph learning (FGL) has recently emerged as a promising privacy-preserving paradigm that enables distributed graph learning across multiple data owners.
A critical privacy concern in federated learning is whether an adversary can recover raw data from shared gradients, a vulnerability known as deep leakage from gradients (DLG).
However, most prior studies on the DLG problem focused on image or text data, and it remains an open question whether graphs can be effectively recovered, particularly when the graph structure and node features are uniquely entangled in GNNs.

In this work, we first theoretically analyze the components in FGL and derive a crucial insight, \ie \textit{once the graph structure is recovered, node features can be obtained through a closed-form recursive rule}.
Building on our analysis, we propose \sysname, a novel approach to recover raw training graphs from shared gradients in FGL, which can utilize randomly generated graphs or client-side training graphs as the auxiliaries to enhance the recovery.
Extensive experiments demonstrate that our \sysname outperforms existing solutions by successfully decoupling the graph structure and node features, achieving improvements of over $5.46\%$ (by MSE) for node feature reconstruction and over $25.04\%$ (by AUC) for graph structure reconstruction.

\end{abstract}



\maketitle

\section{Introduction} 
\label{sec:intro}
    Graphs naturally model the complex relationships or interactions among entities and serve as a prominent data for various real-world applications such as bioinformatics \cite{kdd23:app_drug, kdd22:app_mol}, transportation \cite{kdd24:fedgtp, kdd23:app_traffic, VLDB24_Fed_SM}, and healthcare \cite{kdd20:app_health, kdd17:app_health}.
    Notably, in many real scenarios,  graph data is distributed across multiple data owners and it is prohibited to share the raw graphs directly due to strict data privacy regulations (\eg GDPR~\cite{gdpr} and CPRA~\cite{ccpa}).
    For example, it is promising to train a graph learning model for drug discovery that utilizes extensive graph-based molecular data held by multiple individual drug research institutions.
    However, directly sharing the raw graphs among these institutions is usually infeasible due to privacy constraints.
    As a solution, federated graph learning (FGL)~\cite{kdd22:fedgnn_survey1,tnnls24:fedgnn_survey2,kdd23:fedgnn_survey3} has been emerging as a novel paradigm to tackle distributed graph data, which extends the concepts of federated learning~\cite{tist19:fl_survey, ftml21:fl_survey_long, FCS25_Survey} to graph data and enables the sharing of gradients in the learning process rather than raw graph data, thereby regarded as a privacy preserving way to protect the underlying graph information.
    

    
    In federated learning, a problem known as \textit{{deep leakage from gradients (DLG)}} has been identified as a fundamental privacy concern when sharing the gradients~\cite{nips19:dlg, corr20:idlg, ijcai23:dlg_survey2}, \ie \textit{is it possible to recover private training data from the shared gradients during the federated learning process?} 
    Prior works have studied the \textit{DLG} problem across various data types and have empirically or theoretically demonstrated that raw training data, such as images~\cite{nips19:dlg,nips20:ig,iclr21:rgap}, texts~\cite{emnlp21:tag,nips22:film}, and tabular data~\cite{icml23:tableak}, can be effectively recovered from the shared gradients in federated learning.
    Exploring the DLG issue on graph data is crucial for us to understanding the potential privacy risks in FGL.
    {\textit{Yet, it remains an unresolved question whether we can recover raw training graphs from shared gradients in FGL}}.
    In this work, we aim to study the DLG problem inherent in FGL.

\begin{figure}[t]
	\centering
	\includegraphics[width=0.48 \textwidth]{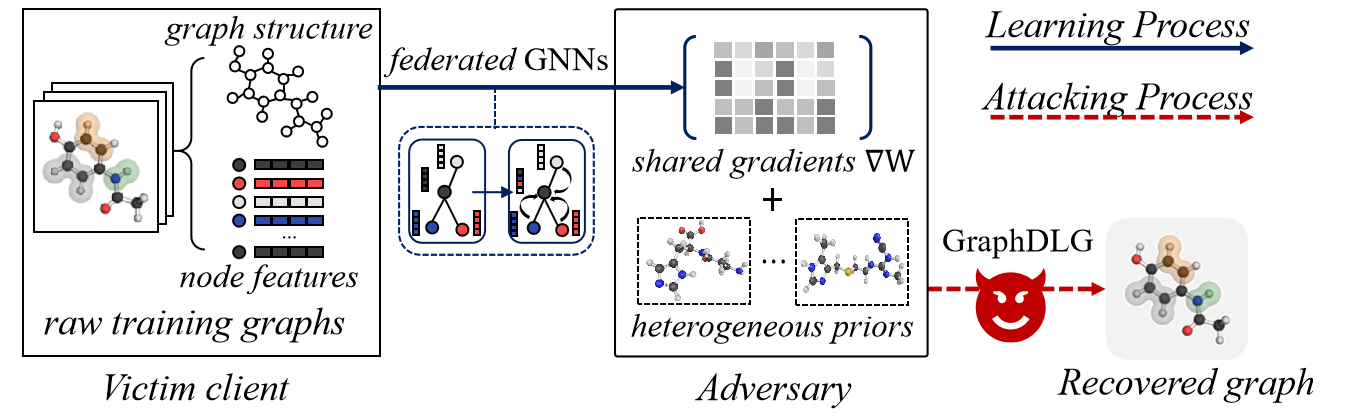}
    \caption{{\small{DLG for Graphs}: In federated graph learning, an adversary can recover training graphs from the victim client's gradients}.}
	\label{fig:intro}
\end{figure}


    As shown in \figref{fig:intro}, the DLG for graphs in FGL poses unique challenges inherent to graph data and existing solutions for recovering training data from gradients face two significant limitations.
    
    - \textit{Limitation 1: Existing recovering approaches are not tailored for the graph neural network models (GNNs).}
    Unlike the targeted images or text, graphs consist of node features and unique topologies. 
    It is common to adopt GNNs to tackle graph data, where the message aggregation operations intertwine node features and graph topologies within the gradients.
    Thus, prior solutions~\cite{nips19:dlg,nips20:ig,emnlp21:tag,nips22:film,icml23:tableak} for images or texts fail to disentangle node features and graph structures, making it ineffective to recover graphs (as in \tabref{tab:app_visualization}).
    
    - \textit{Limitation 2: Neglecting the heterogeneity between the prior knowledge and the target graphs to be recovered.}
    In existing studies~\cite{nips21:gias,cvpr22:ggl}, the adversary is often assumed to leverage its local dataset as auxiliary priors to enhance recovery capability in DLG, as it is natural for an FL client to use its local training graphs as the auxiliary freely.
    However, prior solutions overlook the fact that an adversary's auxiliary priors (\eg local graphs) are usually heterogeneous with the victim client's data to be recovered, a situation commonly encountered in real-world FGL scenarios ~\cite{curs24:hfl, WWW25_FGL_HFGL}.
    Thus, it is crucial to consider such heterogeneity between auxiliary and victim graphs when utilizing the auxiliary priors for DLG in FGL.

   To address the above limitations in existing DLG attacks and better understand the associated privacy risk, we explore deep leakage from gradients in FGL and propose \sysname, which is graph-tailored and can effectively recover both the node features and graph structures of training graphs simultaneously.
    In this work, we focus on widely-used GCNs~\cite{iclr17:gcn}.
    We first theoretically analyze the gradients of two essential components (MLP layers and {GCN layers}) in FGL and arrive at a crucial insight: \textit{if the graph structure is known, node features can be recovered through a closed-form rule}, simplifying the graph-tailored DLG via decoupling node features and graph structures (in \secref{sec:method_analysis}).
   {
   Then, we design a decoder utilizing heterogeneous auxiliary priors to transform graph embeddings derived from the MLP layers in FGL into graph structures, thereby allowing publicly available or even randomly generated graphs to contribute to the graph-tailored DLG attacks (in \secref{sec:method_structure}).
   }
   Finally, we substitute the obtained graph structures into our closed-form rules and recover node features of the target graph by recursively solving a set of linear equations (in \secref{sec:method_feature}). 
   The main contributions and results are summarized as follows.
   \begin{itemize}
        \item 
        We analyze gradients in FGL and derive a rule disentangling node features and graph structure, which provides an insight that node features can be recovered through a closed-form rule once graph structure is obtained.
        \item We propose \sysname, a novel approach to recover training graphs from the shared gradients and the FGL model.
        \sysname first recovers graph structures from graph embeddings with the FGL model and auxiliaries which are heterogeneous to victim's graphs and then recovers the node features utilizing the proposed closed-form recursive rule.
        \item We conduct extensive experiments to validate performance of the  proposed \sysname.
        The results show that our solution outperforms baselines by more than $5.46\%$ in node features (by MSE) and more than $25.04\%$ in graph structure (by AUC), \ie existing approaches for other data types perform poorly, especially in recovering graph structures (shown in \tabref{tab:app_visualization}).
        Besides, an ablation study demonstrates that tackling the heterogeneity between the auxiliary graphs and victim's graphs also improves \sysname's effectiveness. 
    \end{itemize}
    
    In the rest of this paper, we first review the representative related work in \secref{sec:related} and introduce preliminaries in \secref{sec:problem}.
    Then, we present our \sysname approach in \secref{sec:method}.
    Finally, we show the experimental results in \secref{sec:experiment} and conclude this work in \secref{sec:conclusion}.

\section{Related Work}
\label{sec:related}

Our study is mainly related to two lines of research: \textit{(i) the federated graph learning (FGL)} and \textit{(ii) the deep leakage from gradients (DLG)}. 
We review representative works in each research area as follows.

\subsection{Federated Graph Learning}
    Federated Graph Learning (FGL) addresses learning tasks where graph datasets are distributed across multiple data owners while ensuring compliance with strict data regulations such as GDPR~\cite{gdpr} and CCPA~\cite{ccpa}, which prohibit the direct data sharing.
    Existing research in FGL can be divided into three categories based on the level of the graph data held by clients, \ie \textit{(i)} \textit{Graph-level FGL}~\cite{cikm21:fkge,ijcai23:fedhgn, nips21:gcfl,math22:fedgcn,aaai23:fedstar,kdd24:fedspray, aaai24:fair}, \textit{(ii)} \textit{Subgraph-level FGL}~\cite{nips21:fedsage,vldb23:fedgta,sigmod24:fedaas,kdd24:fedgtp} and \textit{(iii)} \textit{Node-level FGL}~\cite{kdd21:CNFGNN,kdd24:hifgl}.
    In this work, we focus on the \textit{Graph-level FGL}, a common scenario for various real-world applications.

Next, we review the representative prior works in \textit{Graph-level FGL}, which mainly focus on addressing the data heterogeneity (\aka \textit{non-i.i.d.}) and the data privacy.
GCFL~\cite{nips21:gcfl} designs a gradient sequence based clustering mechanism to enhance model aggregation among heterogeneous graphs.
FedSpray~\cite{kdd24:fedspray} enables alignment of global structure proxies among multiple graphs in a personalized manner.
Besides, some works focus on privacy-preserving techniques in FGL as well.
For example, FKGE~\cite{cikm21:fkge} proposes an aggregation mechanism for federated graph embedding, using differential privacy to preserve the privacy of raw graph data.
Though prior works~\cite{math22:fedgcn, kdd22:fedgnn_survey1, WWW25_FGL4} emphasize the importance of privacy protection, they do not explore whether private graphs can be recovered via the shared information (\eg, gradients) in FGL, which is a crucial issue helps us to identify the privacy risk.
In this work, we study whether private graphs can be recovered through gradients and aim to fill gaps in understanding privacy leakage risk in FGL.

\subsection{Deep Leakage from Gradients}
In federated learning, recent studies~\cite{csur21:dlg_survey3, fgcs21:dlg_survey4} have shown that the private training data can be recovered from shared gradients, which is known as the \textit{deep {l}eakage from {g}radient (DLG)} problem.
Existing solutions for DLG attack can be classified into two categories~\cite{ijcai23:dlg_survey2}: the \textit{optimization-based DLG} and the \textit{analytics-based DLG}. 

In the seminal work, Zhu~\cite{nips19:dlg} \etal formally define the DLG problem and propose an optimization-based attack that treats the generated dummy data as optimized parameters, which recovers the raw data from dummy data by minimizing the differences between the gradients produced by the dummy data and the shared received gradients.
iDLG~\cite{corr20:idlg} enhances the reconstruction by extracting labels of the target raw data, while subsequent works improve performance by employing other loss functions in recovering process~\cite{nips20:ig,corr20:sapag,entropy23:wdlg}.
Some solutions~\cite{cvpr21:stg, nips21:gias, cvpr22:ggl, esoric20:cpl,cvpr22:gradvit} also utilize the prior knowledge to improve the recovery ability.
For example, GIAS~\cite{nips21:gias} and GGL~\cite{cvpr22:ggl} optimizing a dummy latent vector rather than a dummy data and leverages the prior knowledge embedded in the pre-trained GANs.
However, these solutions overlook the heterogeneous between the priors and the target data.
Analytics-based solutions aim to derive a closed-form formulation that describes the relationship between shared gradients and raw data.
The work in~\cite{tifs18:phong} demonstrates that the input to a fully connected layer can be inferred by dividing the weights' gradient by the bias gradient.
R-GAP~\cite{iclr21:rgap} subsequently proposes a closed-form recursive procedure for calculating the inputs of convolutional neural networks. 
Though these solutions perform well in MLPs or CNNs, which avoid extensive optimization iterations and achieve high accuracy, they can not be applied in federated graph learning, leaving the privacy risks associated with recovering graph data from its gradients unexplored.
Though recent work~\cite{tnse23:empirical_graph,arxiv24:Graph_Attacker} for privacy leakage in FGL study the DLG for graphs, they still applies the idea for image data~\cite{nips19:dlg} to recover node embeddings. 
In this work, we explore and devise solutions to the graph-tailored DLG problem, and also use prior DLG approaches for other data types as baselines.
\section{Preliminaries}
\label{sec:problem}
In this paper, we study the DLG issues in Graph-level FGL, where multiple clients jointly train graph learning model via a federated manner. 
We first introduce basic models and concepts in Graph-level FGL, and then illustrate the adopted threat model in this work.

\subsection{Basic Graph Learning Model}
\fakeparagraph{Notations for Graphs}
Given a set of graphs $\mathcal{G}$ $= \{G_1, \dots, G_N\}$, each graph $G_i$ is associated with a class label $Y_i$ for classification.
For each graph $G=(V, E)$, $V$ and $E$ denote the set of nodes and edges of the graph $G$, respectively.
The node features of a graph are represented by a matrix $X\in \mathbb{R}^{|V|\times d_X}$, where $d_X$ is the feature dimension.
The edges are typically depicted by the adjacency matrix $A \in \mathbb{R}^{|V|\times |V|}$, where $A_{ij}=1$ if two nodes $v_i$ and $v_j$ are connected.

\fakeparagraph{Graph Neural Networks}
GNNs are widely-used solutions to acquire effective representations of graphs~\cite{tnnls21:gnn_survey1,tkde22:gnn_survey2}, which usually is followed by the use of MLPs for various downstream tasks.
The main idea of GNNs is to learn node representations by propagating and aggregating neighbor information, thereby integrating both node features and the graph’s topological structure.
In this paper, we focus the representative \textit{graph convolutional networks} (GCNs)~\cite{iclr17:gcn}, where the information propagation rule is defined as follows,
\begin{equation}\label{eq:gcn_pro}
        \medop{
            {H}_{l+1} = \sigma (\Tilde{{D}}^{-\frac{1}{2}}\Tilde{{A}}\Tilde{{D}}^{-\frac{1}{2}}{H}_{l}{W}_{l}).
        }
\end{equation}
Here ${H}_{l}\in \mathbb{R}^{|V|\times d_{H}}$ is the node embedding matrix of the $l$-th layer in GCNs;
$\Tilde{{A}}={A}+{I}_{|V|}$ is the adjacency matrix with added self connections and $\Tilde{D}$ is the corresponding degree matrix of the input graph $G$;
$\sigma$ is the non-linear activation function, such as $\text{ReLU}(\cdot)$.
After GNNs, to obtain graph-level representation $H_G$ from node embeddings $H_l$, a pooling operation $\text{Pool}(\cdot)$ is commonly used to aggregate embeddings of all nodes, which can be implemented as,
\begin{equation}\label{eq:g_em}
      \medop{
          H_G = \text{Pool}(H_l)=\text{Pool}(\{h_{l,v}|v\in G\}),
       }
\end{equation}
where $h_{l,v}$ is node $v$'s embedding of the $l$-th layer.
The operations mentioned above enable the capture of the entire graph's embedding, which utilizes both node features and the graph structure.

\subsection{Federated Graph Learning}
\fakeparagraph{Graph-level FGL Setting}
We consider a FGL scenario with a central server $S$ and $M$ clients denoted as $C=\{C_1,\dots,C_M\}$. 
Each client $C_i$ holds a set of private graphs $\mathcal{G}_i=\{G_{i,N_1},\dots,G_{i,N_i}\}$, where $N_i$ is the number of local graphs and $N=\sum_{i=1}^M N_i$ is the total graphs in the federation.
The FGL aims to collaboratively train a graph learning model $F$ (\eg graph classification) across $M$ clients based on GNNs, without sharing the private raw graphs (including node features and edges).
The objective of the FGL is to optimize learnable weights $W$ in $F$ based on $\mathcal{G}_{1}, \dots, \mathcal{G}_{M}$ and to find an optimal $W^*$ minimizing the overall loss of all FGL clients as follows,
\begin{equation}\label{eq:fed_obj}
    \medop{
        W^* = \operatorname*{argmin} \limits_{W} \sum_{i=1}^{M}  \frac{N_i}{N} \mathcal{L} (F_i,W, \mathcal{G}_{i}).
        }
\end{equation}

\fakeparagraph{Graph-level FGL Training}
Generally, the training process of FGL comprises the following steps~\cite{tist19:fl_survey}: 
At each global epoch, each client $C_i$ update the local model based on private dataset and send gradients $\nabla{W_i}$ to the server $S$.
Then server $S$ updates the global model weight using aggregated gradients as $W = W - \eta\sum_{i=1}^M{\frac{N_i}{N}\nabla{W_i}}$, and forwards $W$ back to the clients for next epoch of updates. 
By this process, a global model that learned from all clients' data can be obtained without sharing private graphs.
In the DLG problem, the shared gradients $\nabla W_i$ in training process is taken as available information for recovering the private training sample. 

\subsection{Threat Model}
\fakeparagraph{Adversary's Knowledge and Capability}
In this work, we adopt the honest-but-curious adversary, who eavesdrops on the communication between the server $S$ and any client $C_i$ during the federated graph learning.
In FGL training process, the adversary is access to: \textit{(i)} the shared gradients $\nabla{W_{i,i\in[M]}}$ of each clients, \textit{(ii)} the model parameters $W$ at each training epoch, and \textit{(iii)} an auxiliary graph dataset $\mathcal{G}_{aux}$ as the prior knowledge. 
Note that the last assumption is reasonable as an adversary may be an FGL clients with local graphs or collect some public graph datasets from similar scenarios.
In this work, we make a deeper understanding of information leakage from gradients in FGL under the three assumptions outlined above.
As next, we formally introduce the DLG attack for graphs (\ie the adversary's objective), where an adversary aims to recover the victim client's private graph based on the shared gradients in the training process of the federated graph learning.


\begin{definition} [DLG Attack for Graphs]
    In FGL scenario with $M$ clients, where each client $C_i, i\in [M]$ owns a set of graphs $\mathcal{G}_i$ as its private dataset, an adversary $\mathcal{A}$ in DLG is allowed to access to the victim client $C_v$'s gradients $\nabla{W_{v}}$, the federated graph model $F$ (\ie model parameters $W$) and an auxiliary graph dataset $\mathcal{G}_{aux}$.
    Then the adversary attempts to recover node features $\hat X$ and the graph structure (represented as adjacency matrix $\hat A$) close to a private graph $G$ (\ie the ground truth node features and graph structures) in $\mathcal{G}_v$.
\end{definition}
    

\section{Method}
\label{sec:method}

\begin{figure*}[t]
	\centering
	\includegraphics[scale=0.45]{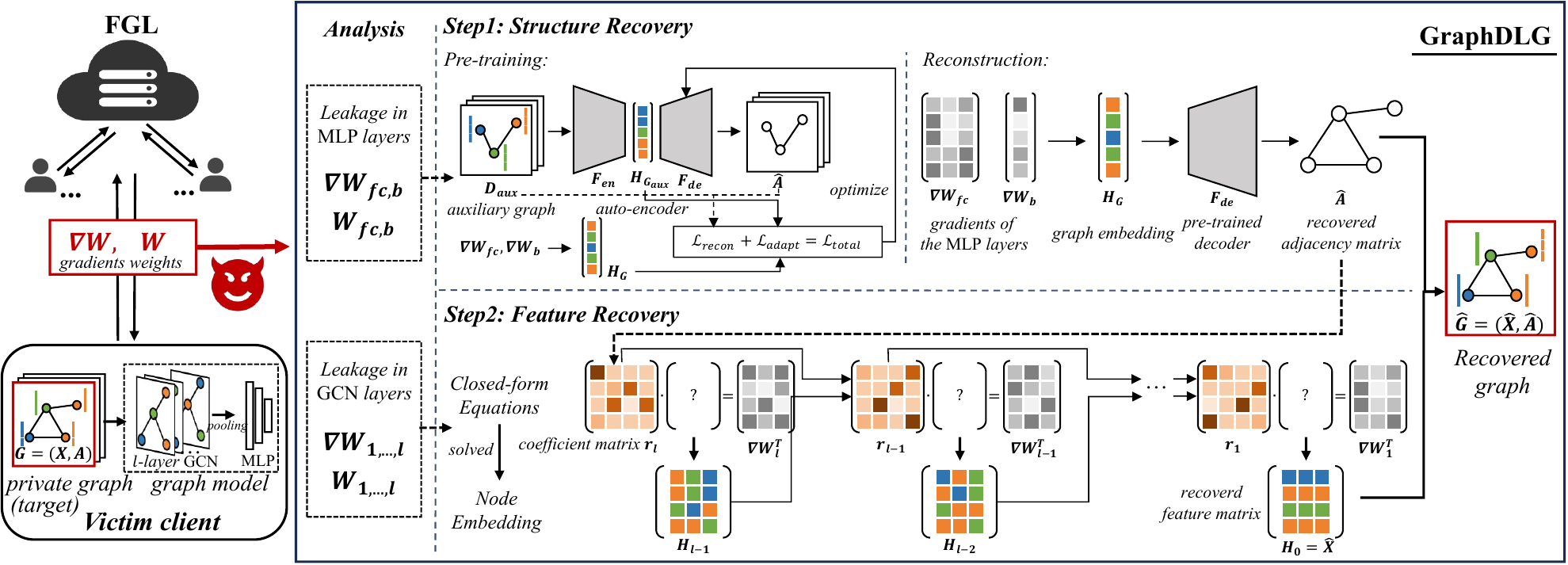}
	\caption{
        \small{ 
        Overview of \sysname. Firstly, a learned decoder can recover graph structure from graph embeddings leaked in MLP layers' gradients. 
        Then, node features can be recovered by recursively solving closed-form equations.
        Finally, the attacker obtained the whole graph.
        }
    }
	\label{fig:overview}
\end{figure*}

In this section, we propose \sysname, a novel framework to recover graphs from shared gradients in FGL (in \figref{fig:overview}).
We first analyze the data leakage through gradients from the theoretical view, providing a foundation for the main idea of \sysname (\secref{sec:method_analysis}).
Next, we devise an auto-encoder to recover the graph structure that utilizes the precisely leaked graph embeddings from our analysis and the auxiliary graphs (\secref{sec:method_structure}).
Finally, we integrate the recovered graph structure into the found closed-form recursive rules and propose a recovery algorithm for node features, thereby recovering an entire private graph of the victim FGL client (\secref{sec:method_feature}).


\subsection{Graph Data Leakage Analysis}
\label{sec:method_analysis}

As mentioned above, a key challenge to recovering graphs in the DLG is that the graph structure and node features are intertwined in the gradients by the forward process in FGL.
Exploring the relations between the graph structure and node features is beneficial to decoupling and recovering them.
Next, we briefly review the forward process in graph-level GNNs and analyze the information leaked through the shared gradients and model parameters.


\fakeparagraph{Forward Process} We begin the analysis with the model forward process, where three parts in the graph-level FGL (including GCN layers $H_{l}$, MLP layers $\hat{Y}$ and loss function $L$) are calculated as:

\begin{equation}\label{eq:modeling_gcn}
    \medop{
    H_l=\sigma_l(\bar A\overbrace{\sigma_{l-1}(\bar A\underbrace{\sigma_{l-2}(\bar A\psi(X,A)W_{l-2})}_{H_{l-2}}W_{l-1})}^{H_{l-1}}W_l)
    }
\end{equation}
\begin{equation}\label{eq:modeling_mlp}
    \medop{
        \hat Y=W_{fc}(M_{p}H_l)^\top+W_{b}
    }
    \hspace{9.0em}
\end{equation}
\begin{equation}\label{eq:modeling_loss}
    \medop{
        L=\mathcal{L} (\hat Y,Y)
        }
        \hspace{12.8em}
\end{equation}
In GCN layers, we follow the widely-used model in~\cite{iclr17:gcn} and rewrite the forward process in a recursive manner, which is more convenient for analysis.
Specifically, $H_{l}$, $H_{l-1}$ and $H_{l-2}$ are the different GCN layers with $\bar A=\Tilde{D}^{-\frac{1}{2}}\Tilde{A}\Tilde{D}^{-\frac{1}{2}}$ and $\psi(X,A)$ denotes input of the $(l-2)_{\mathrm{th}}$ GCN layer $H_{l-2}$.
In MLP layers, the pooling operation (\eg sum pooling or mean pooling) can be denoted by multiplying the matrix 
$M_p$ (details in \appref{sec:app_pool}). 
$W_{fc}$ and $W_b$ are parameters for the linear layer (our analysis can also be extended to multiple layers in following discussion).
After that, the loss function $\mathcal{L}(\cdot)$ is calculated.
In this work, we calculate the gradients of the model parameters $W$ in each layer, \ie, $\nabla{W}=\frac{\partial L}{\partial W}$, through which we can derive the closed-form relations between layer inputs and the corresponding gradients for both the MLP layers and the GCN layers.



\subsubsection{Leakage in MLP Layers}
Firstly, we can calculate the gradients of parameters $W_{fc}$ and $W_{b}$ in the MLP layer by treating the pooling operation as matrix multiplication as follows,

\begin{equation}\label{eq:mlp_g}
    \medop{
        \nabla{W_{fc}} = (\frac{\partial L}{\partial \hat Y})^\top M_p H_l, \quad \nabla{W_{b}} = (\frac{\partial L}{\partial \hat Y})^\top.
    }
\end{equation}
Then the pooled graph embedding $M_p H_l$, \ie, {input of the MLP layer}, can be recovered by any row of the gradients $\nabla{W_{fc}}$ and $\nabla{W_b}$,
\begin{equation}\label{eq:mlp_h_g}
    \medop{
        M_p H_l = \nabla{W_{fc}}^j / \nabla{W_{b}}^j
    },
\end{equation}
where $\nabla{W_{fc}}^j$ and $\nabla{W_{b}}^j$ denotes any row of $\nabla{W_{fc}}$ and $\nabla{W_{b}}$, respectively.
As for multiple MLP layers, once the last layer's input (\ie previous layers' output) is determined, the input of the first layer can be retrieved by solving linear equations recursively $W_{fc, t}X_{t-1}^\top+W_{b, t}=X_{t}$ $(t=1,\dots,l_{MLP})$.
\textit{{
As a result, the MLP layer can easily leak its input $H_{G}$ through the gradients of its parameters.
}}

\subsubsection{Leakage in GCN Layers}
Next, we further study the information that might be leaked through the gradients in GCN layers.
Similarly, we can derive gradients based on \equref{eq:modeling_gcn}$\sim$\equref{eq:modeling_loss}, and we have the following analytical expressions for the $l$-GCN layers.


\begin{lemma} [Gradients in an $l$-layer GCN]
    \label{lemma:gradients}
    For the GCN layers in FGL, the gradients $\nabla W_l$, $\nabla W_{l-1}$, $\nabla W_{l-2}$ can be formally expressed by the node embeddings from their previous layer $H_{l-1}$, $H_{l-2}$, $H_{l-3}$ respectively and the normalized graph adjacency matrix $\bar A$,
    \begin{equation}\label{eq:gcn_g_l}
        \medop{
            \nabla{W_{l}} = H_{l-1}^\top \bar{A}^\top (M_p^\top\frac{\partial L}{\partial \hat Y}W_{fc}\odot\sigma_l') 
        }
        \hspace{12.8em}
    \end{equation}
    \begin{equation}\label{eq:gcn_g_l-1}
       \medop{
            \nabla{W_{l-1}} = H_{l-2}^\top \bar{A}^\top (\bar{A}^\top (M_p^\top\frac{\partial L}{\partial \hat Y}W_{fc}\odot\sigma_l')W_l^\top\odot \sigma_{l-1}') 
        }
        \hspace{6em}
    \end{equation}
    \begin{equation}\label{eq:gcn_g_l-2}
        \medop{
            \nabla{W_{l-2}} = H_{l-3}^\top \bar{A}^\top (\bar{A}^\top(\bar{A}^\top (M_p^\top\frac{\partial L}{\partial \hat Y}W_{fc}\odot\sigma_l')W_l^\top\odot \sigma_{l-1}')W_{l-1}^\top\odot \sigma_{l-2}')
        }
    \end{equation}
    where $\sigma'$ represents the derivative of the activate function $\sigma$.
    The term $\frac{\partial L}{\partial \hat Y}$ in \equref{eq:gcn_g_l} can be derived from \equref{eq:mlp_g}, while $M_p$, $W_{fc}$ and $W_l$ are already known model parameters. 
    The proof of \lemmaref{lemma:gradients} is in \appref{sec:app_gradient}. The gradients of activate function $\sigma_{l}'$ can be derived from the output of $l$-th layer and more details are in \appref{sec:app_pool}.
\end{lemma}

Observing repeated components for each layer in GCNs, we can rewrite them and derive a closed-form recursive rule as follows.

    \begin{equation}\label{eq:gcn_r_l}
        \medop{
            \nabla{W_{l}} = H_{l-1}^\top \mathbf{r}_l, 
            \hspace{2.9em}
            \mathbf{r}_l=\bar{A}^\top(M_p^\top\frac{\partial L}{\partial \hat Y}W_{fc}\odot\sigma_l')
        } 
        \hspace{1.0em}
    \end{equation}
    \begin{equation}\label{eq:gcn_r_l-1}
        \medop{
            \nabla{W_{l-1}} = H_{l-2}^\top \mathbf{r}_{l-1}, 
            \hspace{1em}
            \mathbf{r}_{l-1}=\bar{A}^\top (\mathbf{r}_l W_l^\top\odot \sigma_{l-1}') 
        }
        \hspace{1.5em}
    \end{equation}
    \begin{equation}\label{eq:gcn_r_l-2}
        \medop{
            \nabla{W_{l-2}} = H_{l-3}^\top \mathbf{r}_{l-2}, 
            \hspace{1em}
            \mathbf{r}_{l-2}=\bar{A}^\top (\mathbf{r}_{l-1} W_{l-1}^\top\odot \sigma_{l-2}')
            \hspace{0.5em}
        }
    \end{equation}
    For the $i_{th}$ layer, we conjecture that a similar equation still holds,
    \begin{equation}\label{eq:gcn_r_i}
        \medop{
            \nabla{W_{i}} = H_{i-1}^\top \mathbf{r}_{i}, 
            \hspace{1em}
            \mathbf{r}_{i}=\bar{A}^\top (\mathbf{r}_{i+1} W_{i+1}^\top\odot \sigma_{i}'), i=1,\dots,l-1
        }.
    \end{equation}
    
    In this work, we can prove the above conjecture and derive the following recursive rule for GCN layers (proof in \appref{sec:app_gradient}).
\begin{theorem} [Recursive Rule for Gradients]
        \label{thm:gradients}
    The gradients $\nabla W_i$ in GCN layers can be depicted by its input embedding $H_{l-1}$ and a coefficient matrix $\mathbf{r_i}$ in a closed-form recursive equation, \ie,

    \begin{equation}\label{eq:gcn_r_i_thm1}
       \begin{aligned}
        \medop{
           \nabla{W_{i}} = H_{i-1}^\top \mathbf{r}_{i}
        },
       \end{aligned}
    \end{equation}
    where the coefficient matrix $\mathbf{r}_i$ is defined as follows,
    \begin{equation}\label{eq:r_i_thm1}
        \begin{aligned} 
        \medop{
            \mathbf{r}_{i}
        }&
        \medop{
            =
        }
                \left \{
                \begin{array}{ll}
                     \medop{
                        \bar{A}^\top(M_p^\top\frac{\partial L}{\partial \hat Y}W_{fc}\odot\sigma_i'),
                     } & \medop{i=l}\\
                     \medop{
                        \bar{A}^\top (\mathbf{r}_{i+1} W_{i+1}^\top\odot \sigma_{i}'),
                    } & \medop{i=1,\dots, l-1}
                \end{array} \right.
        \end{aligned}
    \end{equation}
    
\end{theorem}

\noindent\thmref{thm:gradients} implies that \textit{once graph structure $\bar{A}$ is obtained, $\mathbf{r}_i$ can be then calculated by solving a linear equations}, $\nabla{W_i}^\top=\mathbf{r}_i^\top H_{i-1}$ and then we can recursively compute $r_i$ and $H_i$ alternatively until input of the first GCN layer $H_0$ (\ie, node features $X$) is recovered.




\subsubsection{Takeaways}

Above analysis provide us two crucial insights. 

\noindent\textit{(i) Graph embeddings $H_{G}$ can be easily recovered from gradients.
    }
    {
    Thus, we can utilize the $H_{G}$, an intermediate closer to the original graph, rather than the output loss, to recover the original graph.  
    }

\noindent\textit{(ii) 
\thmref{thm:gradients} not only decouples node features and graph structure in gradients but also establishes a recovery order for them.}
{This simplification aids in the task of simultaneously recovering both graph structure and node features from gradients.
}
Based on the two results, we introduce our \sysname, which first recovers the graph structure from the easily obtained graph embeddings $H_G$ (in \secref{sec:method_structure}) and then recovers the node features using \thmref{thm:gradients} (in \secref{sec:method_feature}).

\subsection{Structure Recovery with Hetero-Priors}
Following the established recovery order, we design an auto-encoder based on FGL models and use its decoder to transform obtained graph embeddings $H_{G}$ (in \secref{sec:method_analysis}) into the graph structure ${\hat A}$.

\label{sec:method_structure}
\fakeparagraph{Main idea} 
    As above, we have successfully recovered the graph embeddings $H_{G}$ (\ie, final output of GCN layers), which integrates both the graph structure and node features.
    Thus, we can extract the graph structure from $H_G$. 
    Specifically, we design an auto-encoder to fully utilize the knowledge in the FGL model and priors in auxiliary graphs $\mathcal{G}_{aux}$, where the FGL model is fixed as the encoder to transform the original input graph into the embeddings $H_{G}$ and then the decoder is with the capability to recover the graph structure from the graph embeddings.
    We further design an adapter to handle the heterogeneity between the prior and target graphs.

\fakeparagraph{{Auto-Encoder for Graph Structure}}
We introduce the proposed auto-encoder to recover the graph structure as follows.

\begin{equation}\label{eq:gae}
   \medop{
        H_G = F_{en}(X, A), \quad \hat A = F_{de}(H_G)
    }
\end{equation}
\textit{(i) Encoder Module}: 
Here $F_{en}(\cdot)$ is the encoder, which shares the same parameters as the federated graph model in FGL except for the MLP layers.
This design aims to utilize the knowledge in the federated graph model to enhance recovery.
The encoder takes the graph's node features $X$ and adjacency matrix $A$ as inputs and outputs the graph embedding $H_{G}$.
\textit{(ii) Decoder Module}: 
Here $F_{de}(\cdot)$ is the decoder which takes graph embedding $H_G$ as input and outputs the recovered adjacency matrix $\hat A$.
We employ a three-layer MLP ending with the Sigmoid function as the decoder to transform the graph embedding vector back to the adjacency matrix.
\textit{(iii) Loss function}: 
    We use the auxiliary graph datasets $\mathcal{G}_{aux}$ to train the decoder $F_{de}$ with the reconstruction loss as follows,

\begin{equation}\label{eq:gae_loss_r}
    \medop{
        L_{recon} = \frac{1}{|\mathcal{G}_{aux}|}\sum_{(X, A)\in \mathcal{G}_{aux}}\left \| F_{de}(F_{en}(X, A)) - A\ \right \|^2
    }
\end{equation}
Thereby, the \sysname can use the auxiliary priors $\mathcal{G}_{aux}$ to train the decoder and recover the graph structure of the target graphs.

\fakeparagraph{Utilizing Heterogeneous Priors}
    As we have obtained the ground-truth embeddings $H_{G}$ from target graphs $\mathcal{G}_{v}$, we can align the distribution of the target graphs $\mathcal{G}_v$ and the auxiliary graphs $\mathcal{G}_{aux}$ using their graph embeddings from the same FGL model instead of the unavailable private original graphs.
    Specifically, we use the domain adaption~\cite{aaai08:mmd_transfer} and adopt the widely-used \textit{Maximum Mean Discrepancy} (MMD) as the distance metric to minimize divergences between graph embeddings from $\mathcal{G}_{aux}$ and $\mathcal{G}_{v}$, so that the decoder trained on auxiliaries $\mathcal{G}_{aux}$ can still perform well on $\mathcal{G}_{v}$.
    Formally, the distance metric $MMD(\cdot, \cdot)$ is defined as follows, 

\begin{equation}\label{eq:mmd}
    \medop{
        MMD(\mathcal{G}_{i}, \mathcal{G}_{j}) = \Vert \sum_{G_{i,k}\in \mathcal{G}_{i}}{\phi(H_{G_{i,k}})}/{|\mathcal{G}_{i}|}-\sum_{G_{j,k}\in \mathcal{G}_{j}}{\phi(H_{G_{j,k}})}/{|\mathcal{G}_{j}|}  \Vert
    },
\end{equation}
where $\phi$ denotes the mapping to a Reproducing Kernel Hilbert Space (RKHS).
We add $MMD(\mathcal{G}_{aux}, \mathcal{G}_{v})$ as the regularization to the reconstruction loss and the total loss function is as follows,

\begin{equation}\label{eq:gae_loss_total}
    \medop{
        L_{total} = L_{recon} + MMD(\mathcal{G}_{aux}, \mathcal{G}_{v})
    }
\end{equation}
Here, $\lambda$ serves as the weight coefficient for normalizing $L_{adapt}$. 

\begin{algorithm}[htb]
    \caption{Graph Node Features Recovery (GNFR)}
    \label{alg:attack}
    \SetKwInOut{Input}{input}
    \SetKwInOut{Output}{output}
    \Input{
        FGL models $W$, victim clients gradients $\nabla{W}$;\\
        recovered adjacency matrix $\hat{A}$;\\
    }
    \Output{The victim's node features $\hat{X}_0$;\\}
    $\Tilde{A}=\hat{A}+I$, $\bar A=\Tilde{D}^{-\frac{1}{2}}\Tilde{A}\Tilde{D}^{-\frac{1}{2}}$;\\
    \For{layer $k \leftarrow l$ to $1$}{
        \If{$k == l$}{
            Initialized $\mathbf{r}_k$ using \equref{eq:mlp_g} and \equref{eq:gcn_r_l};\\
        }
        \Else{
            Derive $\sigma_k$ from $\hat X_k$;\\
            $\mathbf{r}_k\leftarrow \bar{A}^\top(\mathbf{r}_{k+1}W_{k+1}^\top\odot \sigma_{k})$;\\
        }
        $B \leftarrow flatten(\nabla{W_k}^\top)$;\\
        $R \leftarrow \mathbf{r}_k^\top\otimes I_{\hat d_{X_{k-1}}}$;\\
        Solving linear equations $Rx=B$;\\
        $\hat X_{k-1}\leftarrow unflatten(x)$;\\
    }
    \Return $\hat X_0$;
\end{algorithm}


\fakeparagraph{Discussion} Next, we provide a brief discussion on the auxiliary dataset $\mathcal{G}_{aux}$. 
A key advantage of our approach is that it does not require auxiliary graphs to follow the same distribution as the victim’s dataset, making it broadly applicable in practice.
Specifically, an attacker may exploit its local training graphs, publicly available graph datasets, or randomly generated graph structures.
To validate this flexibility, we conduct experiments to examine three auxiliary data scenarios, \ie auxiliaries from the same distribution, from heterogeneous distributions, and from random generation.


\subsection{Feature Recovery via Closed-form Rules}
\label{sec:method_feature}
In this subsection, we introduce the \textit{Graph Node Features Recovery} (\textsl{GNFR}) algorithm, which utilizes the estimated graph structure $\hat{A}$ to recover the node features $\hat{X}$ based on our recursive closed-form rule in \thmref{thm:gradients}.
Specifically, we illustrate the graph node features recovery algorithm in \algref{alg:attack}.
The main idea of \algref{alg:attack} is that if we can recursively calculate node embeddings in each GCN layer of the FGL, we can ultimately recover the original node features $\medop{\hat X_0}$.
The algorithm starts from the final GCN layer and proceeds layer by layer.
For the $k_{th}$ GCN layer, the coefficient matrix $\mathbf{r}_k$ is computed based on values in the $(k+1)_{th}$ layer (in lines 3–7).
In lines 8-11, we calculate node embedding matrix of the previous layer (\ie $\medop{\hat X_{k-1}}$) by solving linear equations.
To convert the matrix equation into a standard linear system, we flatten $\medop{\nabla{W_k}^\top}$ into a vector in line 8 and convert the coefficient matrix into $\medop{\mathbf{r}_k^\top \otimes I_{d_{\hat X_{k-1}}}}$ in line 9.
Here $\otimes$ denotes the Kronecker product and $\medop{I_{d_{\hat X_{k-1}}}}$ is the identity matrix with dimensions $\medop{d_{\hat X_{k-1}} \times d_{\hat X_{k-1}}}$, and 
$\medop{d_{\hat X_{k-1}}}$ is the last dimension of $\medop{\hat X_{k-1}}$.
After solving the standard linear equations in line 10, we reshape the vectorized results $x$ to obtain the node embedding matrix $\medop{\hat  X_{k-1}}$ in line 11. 
Through the above iterative process in \algref{alg:attack}, we can effectively reconstruct node features, ultimately retrieving the complete graph from the shared gradients.

\section{Experiments}
\label{sec:experiment}
We conduct extensive experiments to verify the effectiveness of proposed \sysname, aiming to answer the following three questions.
\textbf{Q1}: State-of-the-art solutions are not tailored for the graph data, prompting us to evaluate their performance on graphs, \ie, \textit{can \sysname achieve better performance for graph recovery?}
\textbf{Q2}: Prior solutions overlook the heterogeneity between priors and target data, so we question that \textit{how heterogeneity affects graph recovery and whether our \sysname can improve recovery by the adapter designed to handle heterogeneous priors?}
\textbf{Q3}: Existing defense methods (\eg gradient compression~\cite{nips19:dlg} and differential privacy~\cite{iclr18:dp}) are commonly employed to counter the DLG attacks, so we question that \textit{to what extent are they against the proposed \sysname?}

\subsection{Experimental Setup} 


\fakeparagraph{Datasets} 
We use five graph classification datasets~\cite{tud}, including 
\textit{(i) Molecule datasets} (MUTAG, PTC\_MR, and AIDS), where each node represents an atom (with a one-hot vector as its feature), and edges denote chemical bonds.
\textit{(ii) Bio-informatics datasets} (ENZYMES and PROTEINS), where each node represents an amino acid with features indicating type and edges connecting neighboring amino acids.
Unlike previous experimental setups in DLG~\cite{nips21:gias,cvpr22:ggl} which randomly split the graph datasets, we adopt the Dirichlet distribution to establish the heterogeneity between auxiliaries and targets.
In ablation studies, we further conduct experiments in two scenarios: \textit{(i)} target and auxiliary graphs are from different datasets and \textit{(ii)} the auxiliary graphs are randomly generated.

\fakeparagraph{Baselines} 
We compare the proposed \sysname with eight baselines for DLG problem, including Random, DLG~\cite{nips19:dlg}, iDLG~\cite{corr20:idlg}, InverGrad~\cite{nips20:ig}, GI-GAN~\cite{nips21:gias,cvpr22:ggl}, GRA-GFA~\cite{tnse23:empirical_graph}, TabLeak~\cite{icml23:tableak} and  Graph Attacker~\cite{arxiv24:Graph_Attacker}.
More details of baselines are in \appref{sec:app_baselines}.

\fakeparagraph{Configurations}
In following experiments, we implement global graph learning model as a two-layer GCN with a linear layer as the classifier.
The hidden size of the GCN layers is set to $16$.
For our proposed \sysname, the hidden size of $F_{de}$ is set to $[100, 250]$, and the output size is determined by the maximum number of nodes in the target graphs.
We use the identity function to approximate $\phi$ since the encoder has projected the graph feature into a high-dimension space.
We set $\lambda$ to $0.2$ based on a parameter search, and the auto-encoder model is optimized using Adam with a learning rate of $0.001$.
For any baselines that include optimization process, the maximum number of recovery iterations is set to $100$.
Other hyperparameters follow the settings reported in the original paper.
Due to limited space, more setups are described in \appref{sec:app_setup} and more results of various GCN configurations (including batch sizes, GCN layers and hidden sizes) are presented in \appref{sec:app_others}.

\fakeparagraph{Evaluation Metrics}
\textit{(i) Node Features}: we use the mean square error (MSE) to evaluate recovery performance.
Since node features represent node types, we use node class accuracy (ACC) to evaluate results semantically.
\textit{(ii) Graph Structure}: we employ area under the ROC curve (AUC) and average precision (AP) as metrics, which are widely used to measure binary classification performance across various thresholds. 
For a straightforward interpretation, we also present the accuracy of edges using a threshold of $0.5$.
\textit{Specifically, a lower MSE and a higher AUC indicate better recovery performance.}

\subsection{Performance Comparison}
\label{sec:exp_baseline}


\begin{table}[]
\vspace{-1ex}
\setlength{\tabcolsep}{2pt}
\caption{Comparison of performance on feature and structure recoveries between \sysname and baselines. The best result is \textbf{bold} and the second result is \uline{underlined}.}
\label{tab:baseline}
\vspace{-2ex}
\resizebox{\columnwidth}{!}{%
\begin{tabular}{@{}c|c|cc|ccc@{}}
\toprule[1.5pt]
 &  & \multicolumn{2}{c|}{\textbf{Node Feature}} & \multicolumn{3}{c}{\textbf{Adjacency Matrix}} \\ \cmidrule(l){3-7} 
\multirow{-2}{*}{\textbf{Dataset}} & \multirow{-2}{*}{\textbf{Method}} & \textbf{MSE} & \textbf{ACC(\%)} & \textbf{AUC} & \textbf{AP} & \textbf{ACC (\%)} \\ \toprule[1.5pt]
 & Random & 0.3366 & 13.03 & 0.4933 & 0.1301 & 49.45 \\
 & DLG~\cite{nips19:dlg} & 1.2265 & 18.47 & 0.5226 & 0.1415 & 66.10 \\
 & iDLG~\cite{corr20:idlg} & 1.1063 & 19.66 & 0.5243 & 0.1409 & 65.88 \\
 & InverGrad~\cite{nips20:ig} & 0.9164 & 28.04 & 0.5183 & 0.1417 & 51.76 \\
 & GI-GAN~\cite{nips21:gias,cvpr22:ggl} & {\ul 0.1053} & 51.86 & 0.5081 & 0.1246 & 13.96 \\
 & GRA-GRF~\cite{tnse23:empirical_graph} & 0.1082 & {\ul 71.37} & 0.4891 & 0.1255 & 49.74 \\
   & TabLeak~\cite{icml23:tableak} & 0.3914 & 21.16 & {\ul 0.5323} & {\ul 0.1436} & {\ul 75.40} \\
 & Graph Attacker~\cite{arxiv24:Graph_Attacker} & 1.0366 & 27.74 & 0.5223 & 0.1287 & 35.32 \\
\multirow{-7}{*}{MUTAG} & \textbf{\sysname (ours)} & \textbf{0.0844} & \textbf{74.35} & \textbf{0.9042} & \textbf{0.7556} & \textbf{93.41} \\ \midrule
 & Random & 0.3334 & 6.95 & 0.4858 & 0.1661 & 48.75 \\
 & DLG~\cite{nips19:dlg} & 1.0289 & 12.33 & 0.5372 & 0.1992 & 62.63 \\
 & iDLG~\cite{corr20:idlg} & 1.0608 & 11.13 & 0.5307 & 0.1864 & 65.41 \\
 & InverGrad~\cite{nips20:ig} & 0.9154 & 16.19 & 0.5469 & 0.1939 & 52.46 \\
 & GI-GAN~\cite{nips21:gias,cvpr22:ggl} & 0.0586 & 5.81 & 0.4919 & 0.1529 & 28.05 \\
 & GRA-GRF~\cite{tnse23:empirical_graph} & {\ul 0.0372} & {\ul 70.45} & 0.4753 & 0.1506 & 48.64 \\
 & TabLeak~\cite{icml23:tableak} & 0.3308 & 9.75 & {\ul 0.5656} & {\ul 0.2040} & {\ul 73.32} \\ 
 & Graph Attacker~\cite{arxiv24:Graph_Attacker} & 0.9893 & 15.15 & 0.5495 & 0.1779 & 38.29 \\
\multirow{-7}{*}{PTC\_MR} & \textbf{\sysname (ours)} & \textbf{0.0329} & \textbf{72.53} & \textbf{0.8893} & \textbf{0.6808} & \textbf{91.08} \\ \midrule
 & Random & 0.3329 & 2.44 & 0.5022 & 0.2057 & 50.19 \\
 & DLG~\cite{nips19:dlg} & 1.0834 & 5.78 & 0.5475 & 0.2192 & 61.07 \\
 & iDLG~\cite{corr20:idlg} & 1.1559 & 5.99 & 0.5437 & 0.2187 & 61.49 \\
 & InverGrad~\cite{nips20:ig} & 0.8535 & 8.79 & {\ul 0.5526} & 0.2221 & 44.88 \\
 & GI-GAN~\cite{nips21:gias,cvpr22:ggl} & 0.0262 & 6.17 & 0.5072 & 0.1793 & 30.53 \\
 & GRA-GRF~\cite{tnse23:empirical_graph} & {\ul 0.0238} & {\ul 53.79} & 0.4754 & 0.1719 & 48.45 \\
 & TabLeak~\cite{icml23:tableak} & 0.3015 & 3.67 & 0.5505 & {\ul 0.2229} & {\ul 72.13} \\ 
 & Graph Attacker~\cite{arxiv24:Graph_Attacker} & 0.9150 & 7.86 & 0.5383 & 0.1901 & 38.20 \\
\multirow{-7}{*}{AIDS} & \textbf{\sysname (ours)} & \textbf{0.0225} & \textbf{62.18} & \textbf{0.7677} & \textbf{0.5113} & \textbf{83.23} \\ \midrule
 & Random & 0.3315 & 32.53 & 0.5000 & 0.1725 & 49.91 \\
 & DLG~\cite{nips19:dlg} & 2.8092 & 37.43 & 0.5273 & 0.1834 & 62.56 \\
 & iDLG~\cite{corr20:idlg} & 1.5751 & 36.40 & 0.5263 & 0.1817 & 64.53 \\
 & InverGrad~\cite{nips20:ig} & 1.1449 & 40.14 & 0.5290 & {\ul 0.1836} & 55.83 \\
 & GI-GAN~\cite{nips21:gias,cvpr22:ggl} & {\ul 0.2083} & 41.19 & 0.4978 & 0.1627 & 58.98 \\
 & GRA-GRF~\cite{tnse23:empirical_graph} & 0.2229 & {\ul 52.26} & 0.4819 & 0.1595 & 49.32 \\
 & TabLeak~\cite{icml23:tableak} & 0.5706 & 34.88 & {\ul 0.5310} & 0.1835 & {\ul 72.30} \\ 
 & Graph Attacker~\cite{arxiv24:Graph_Attacker} & 1.3590 & 40.64 & 0.5251 & 0.1734 & 37.69 \\
\multirow{-7}{*}{ENZYMES} & \textbf{\sysname (ours)} & \textbf{0.1772} & \textbf{60.60} & \textbf{0.6615} & \textbf{0.3911} & \textbf{85.59} \\ \midrule
 & Random & 0.3354 & 31.71 & 0.4969 & 0.1733 & 50.00 \\
 & DLG~\cite{nips19:dlg} & 1.5237 & 38.05 & 0.5286 & 0.1879 & 63.68 \\
 & iDLG~\cite{corr20:idlg} & 1.4736 & 36.67 & 0.5288 & 0.1875 & 64.41 \\
 & InverGrad~\cite{nips20:ig} & 1.1160 & 41.47 & {\ul 0.5301} & {\ul 0.1889} & 55.03 \\
 & GI-GAN~\cite{nips21:gias,cvpr22:ggl} & {\ul 0.2192} & {\ul 49.15} & 0.5123 & 0.1710 & 43.26 \\
 & GRA-GRF~\cite{tnse23:empirical_graph} & 0.2198 & 40.44 & 0.4802 & 0.1603 & 48.90 \\
 & TabLeak~\cite{icml23:tableak} & 0.5904 & 35.11 & 0.5255 & 0.1856 & {\ul 72.54} \\
 & Graph Attacker~\cite{arxiv24:Graph_Attacker} & 1.2923 & 42.51 & 0.5242 & 0.1745 & 37.46 \\
\multirow{-7}{*}{PROTEINS} & \textbf{\sysname (ours)} & \textbf{0.1935} & \textbf{62.23} & \textbf{0.6630} & \textbf{0.3744} & \textbf{85.24} \\ \bottomrule[1.5pt]
\end{tabular}%
}
\vspace{-5.6ex}
\end{table}

\begin{table*}[htb]
\caption{\small{Visualization of the recovery performance for \sysname and baselines}.}
\label{tab:baseline_visual}
\vspace{-2ex}
\resizebox{0.95\textwidth}{!}{%
\begin{tabular}{@{}ccccccccccc@{}}
\toprule[1.5pt]
\textbf{} & \textbf{} & \textbf{Random} & \textbf{DLG} & \textbf{iDLG} & \textbf{InverGrad} & \textbf{GI-GAN} & \textbf{GRA-GRF} & \textbf{Tableak} & \textbf{Graph-Attacker}& \textbf{\sysname (ours)} \\ \toprule[1.5pt]
\textbf{Node Feature} & \multicolumn{1}{c|}{\textbf{MSE/ACC(\%)}} & 0.3039/7.14 & 1.3186/14.29 & 1.0636/21.43 & 0.8969/35.71 & 0.1008/57.14 & 0.1068/\textbf{78.57} & 0.4203/21.43 & 1.0295/28.57 & \textbf{0.0816/78.57} \\
\textbf{Adjacency Matrix} & \multicolumn{1}{c|}{\textbf{AUC/AP/ACC(\%)}} & 0.4926/0.1604/46.94 & 0.6058/0.1902/66.33 & 0.5108/0.1599/59.69 & 0.4865/0.1585/44.90 & 0.5120/0.1563/17.35 & 0.5863/0.1817/53.06 & 0.5382/0.1614/74.49 & 0.4978/0.1525/31.12 & \textbf{0.9622/0.8832/93.88} \\
\textbf{\begin{tabular}[c]{@{}c@{}}Ground-truth/\\ Reconstructed graph\end{tabular}} & \multicolumn{1}{c|}{\adjustbox{valign=c}{\includegraphics[scale=0.135]{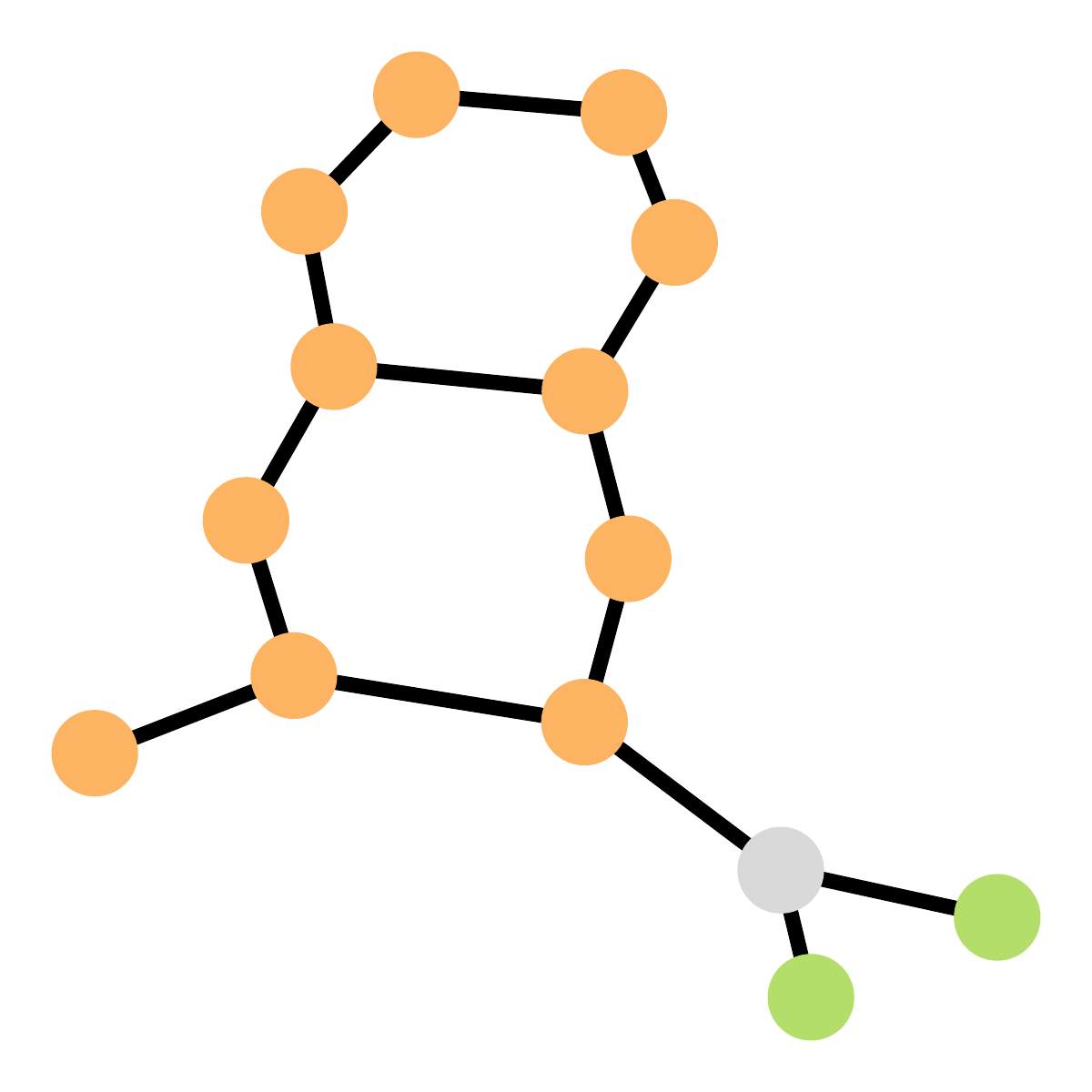}}} & \adjustbox{valign=c}{\includegraphics[scale=0.135]{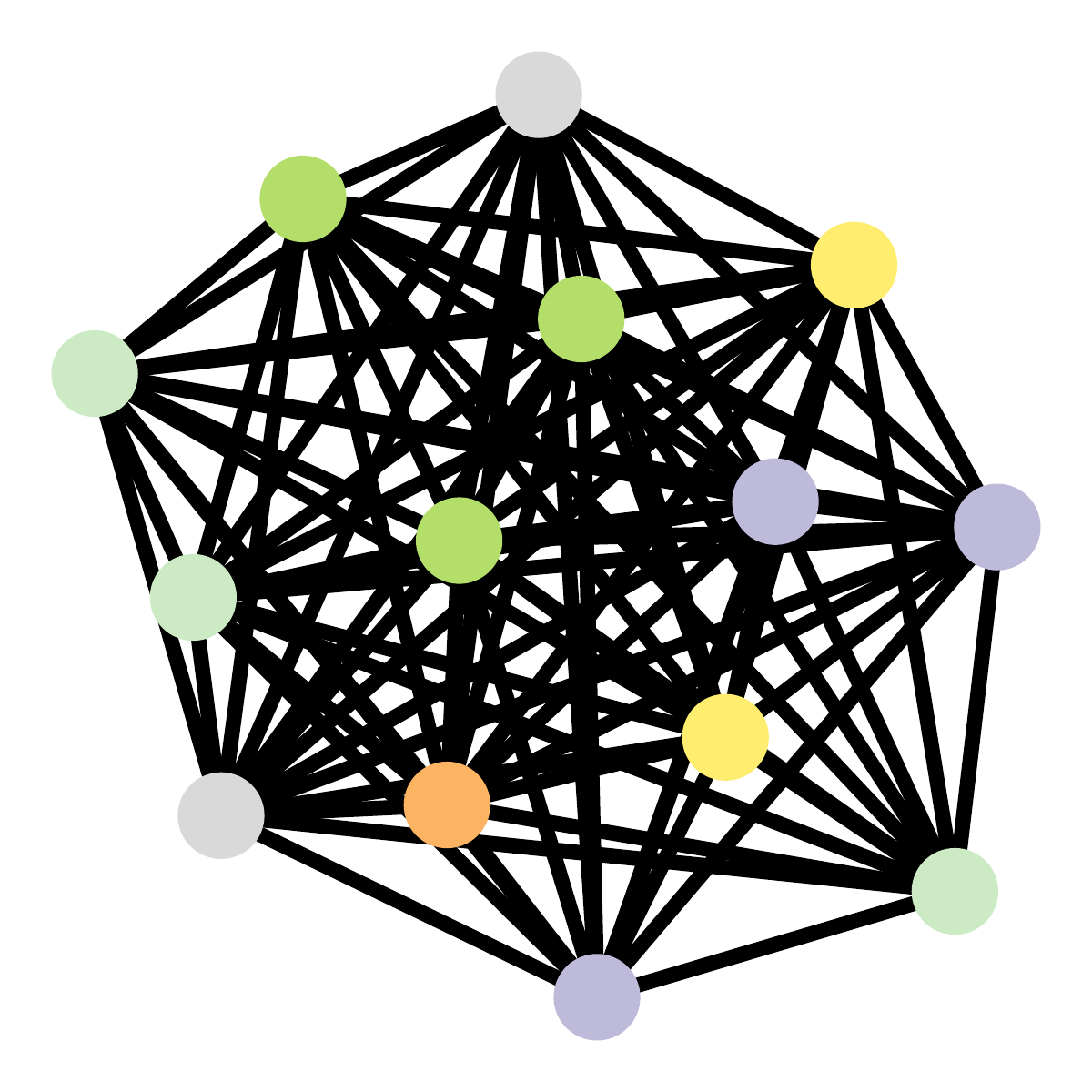}} & \adjustbox{valign=c}{\includegraphics[scale=0.135]{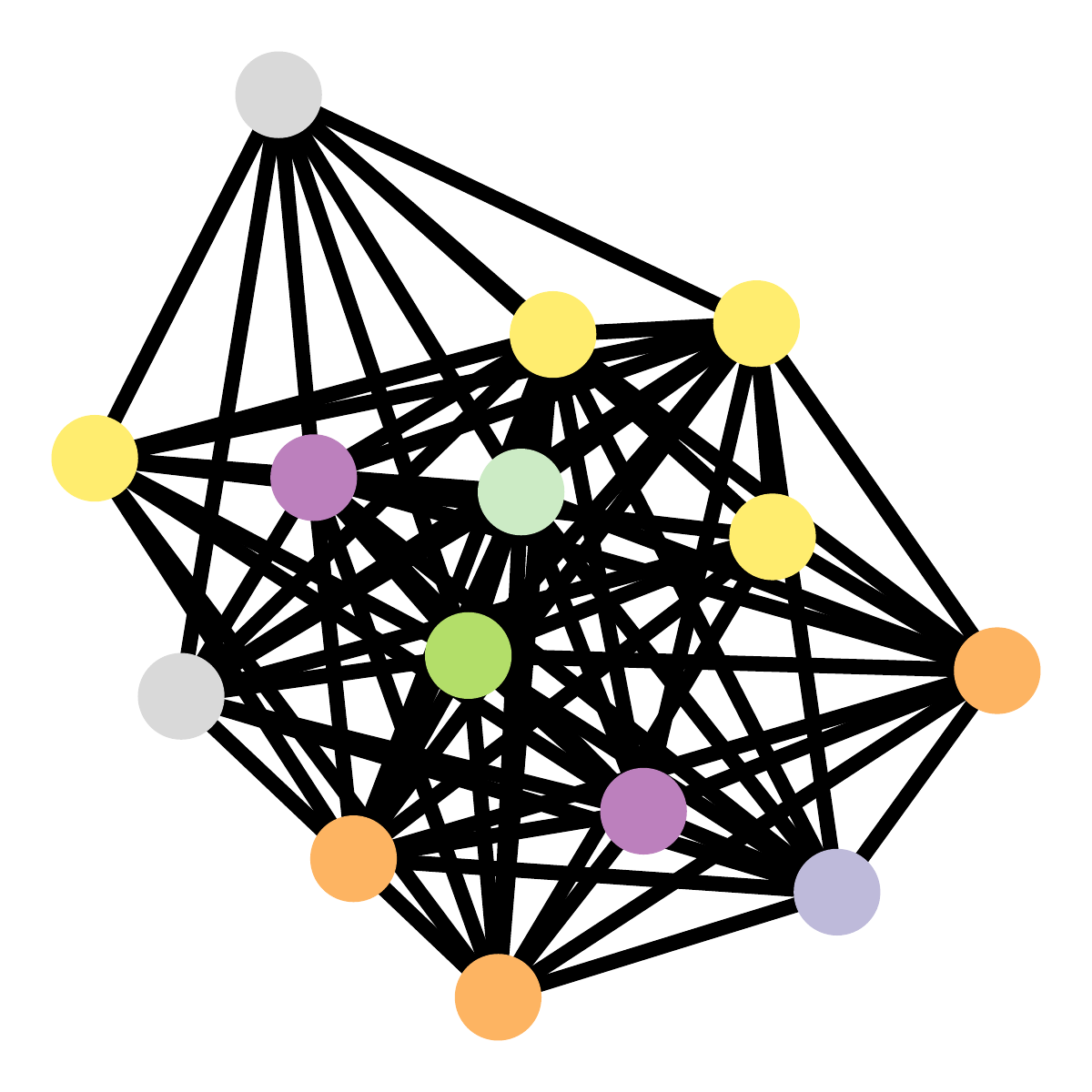}} & \adjustbox{valign=c}{\includegraphics[scale=0.135]{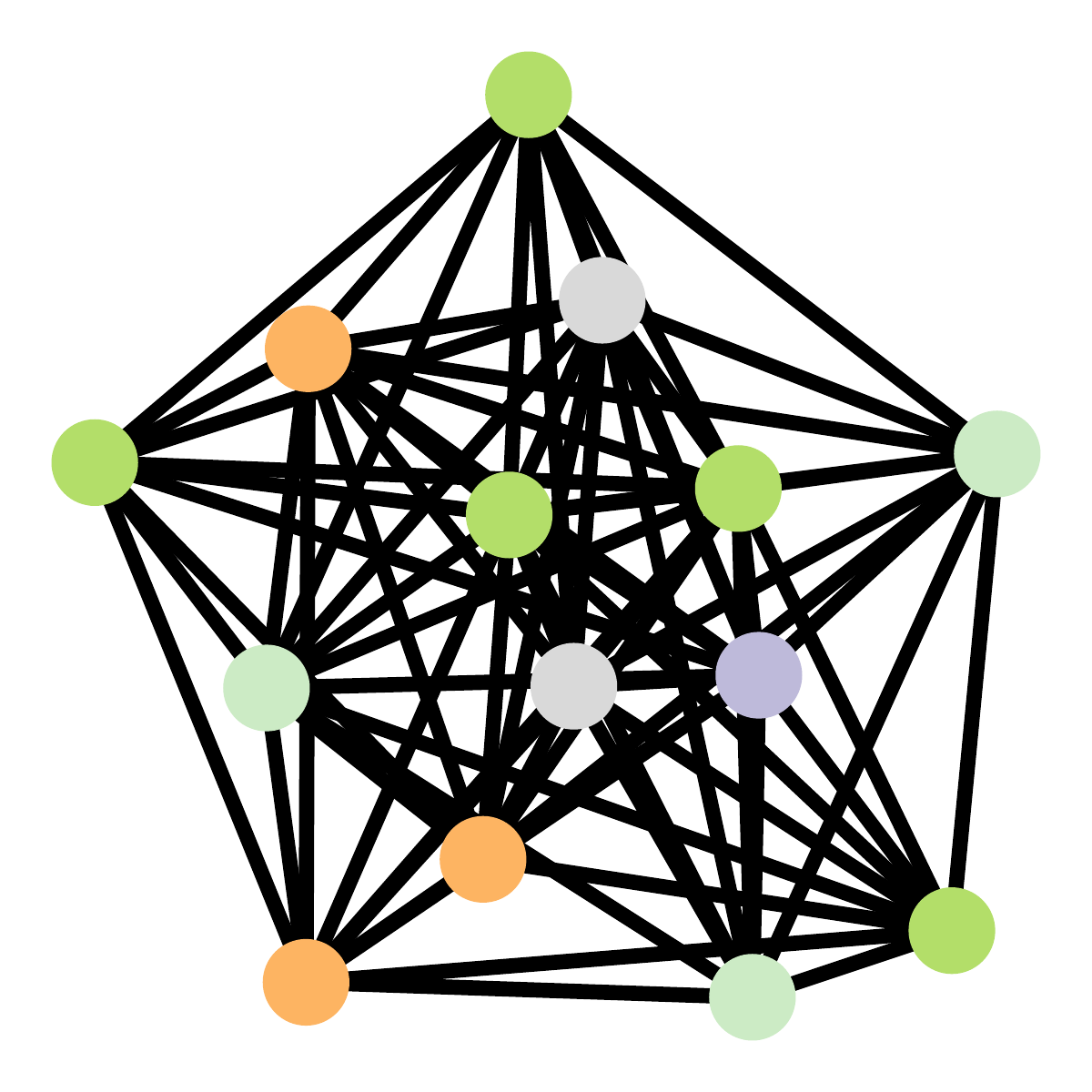}} & \adjustbox{valign=c}{\includegraphics[scale=0.135]{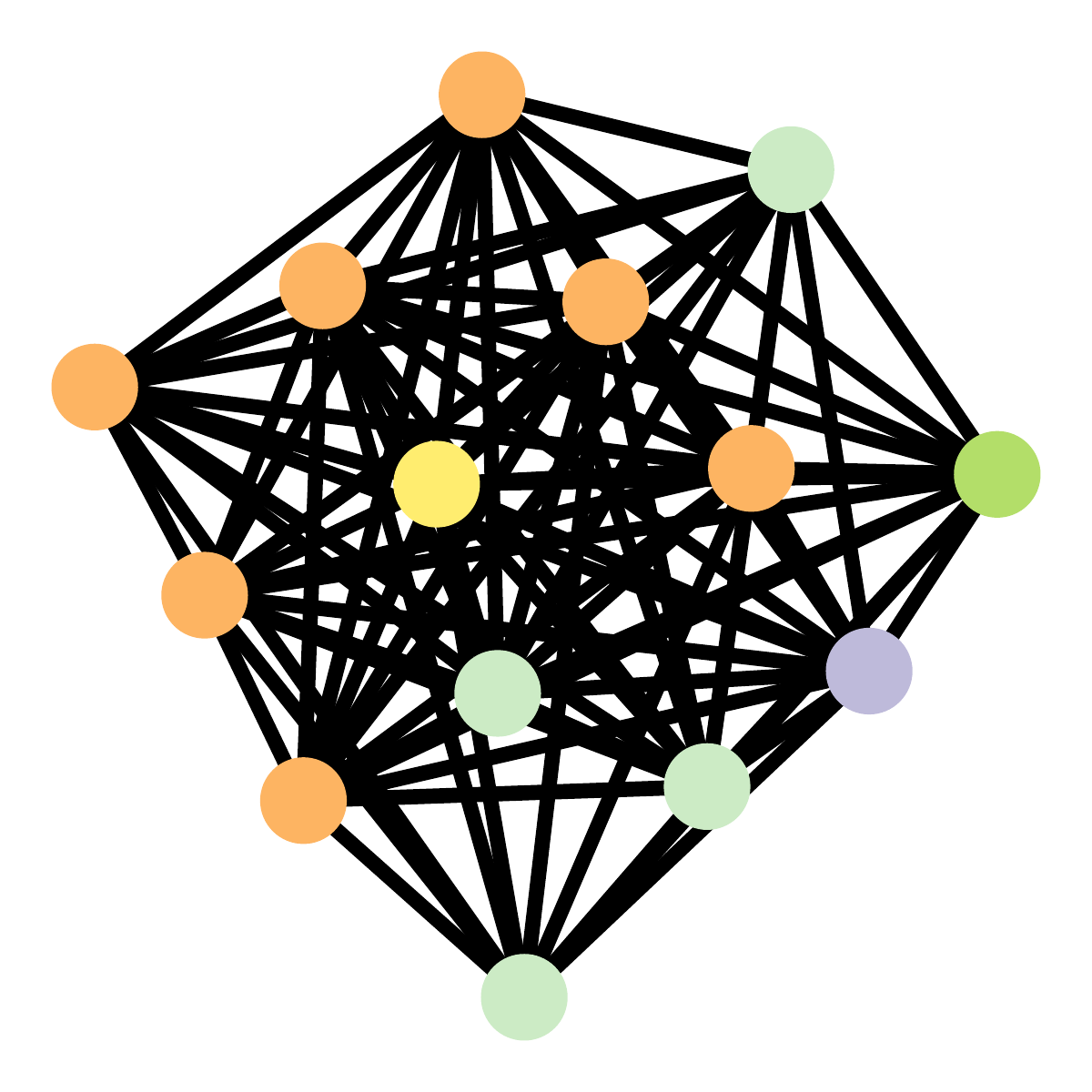}} & \adjustbox{valign=c}{\includegraphics[scale=0.135]{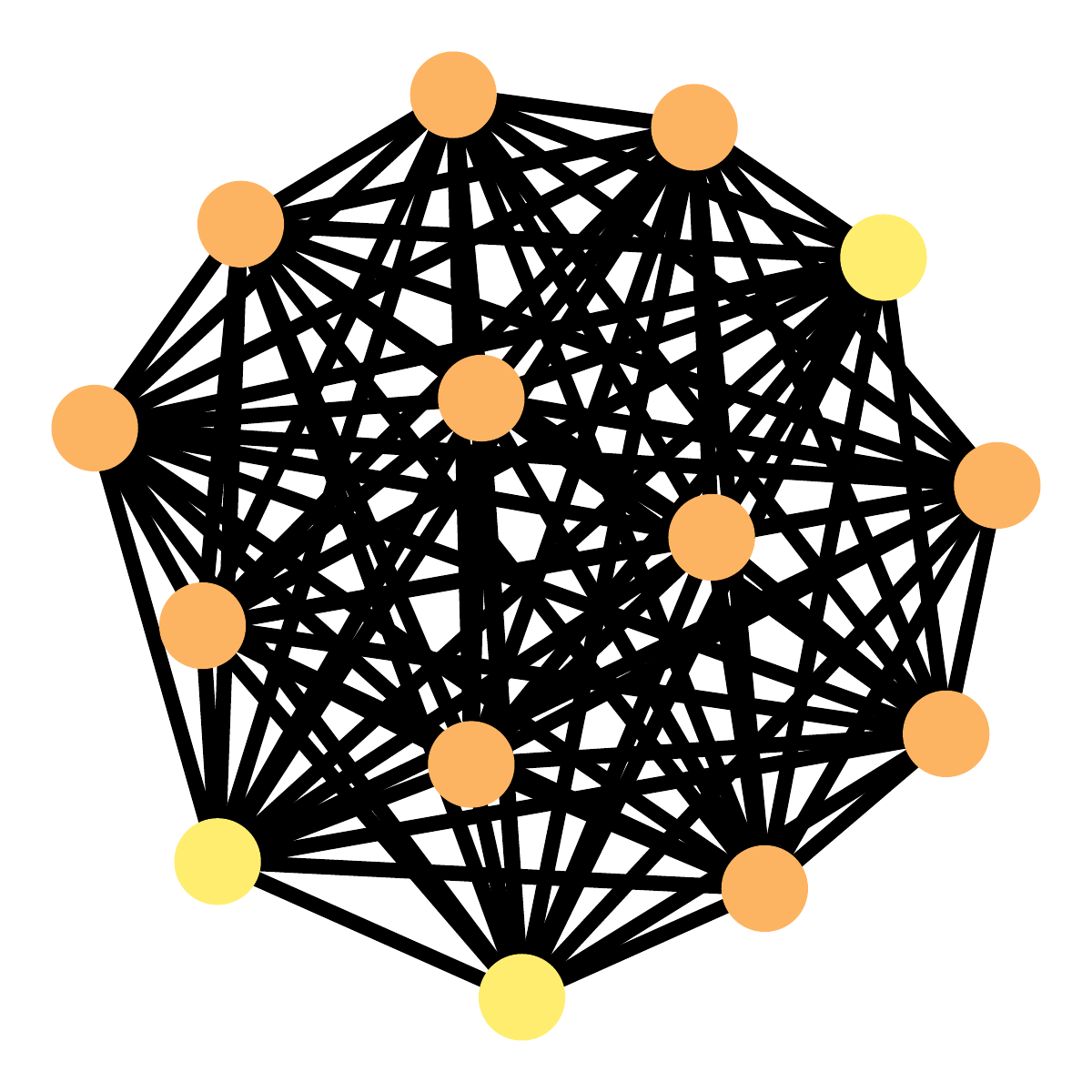}} & \adjustbox{valign=c}{\includegraphics[scale=0.135]{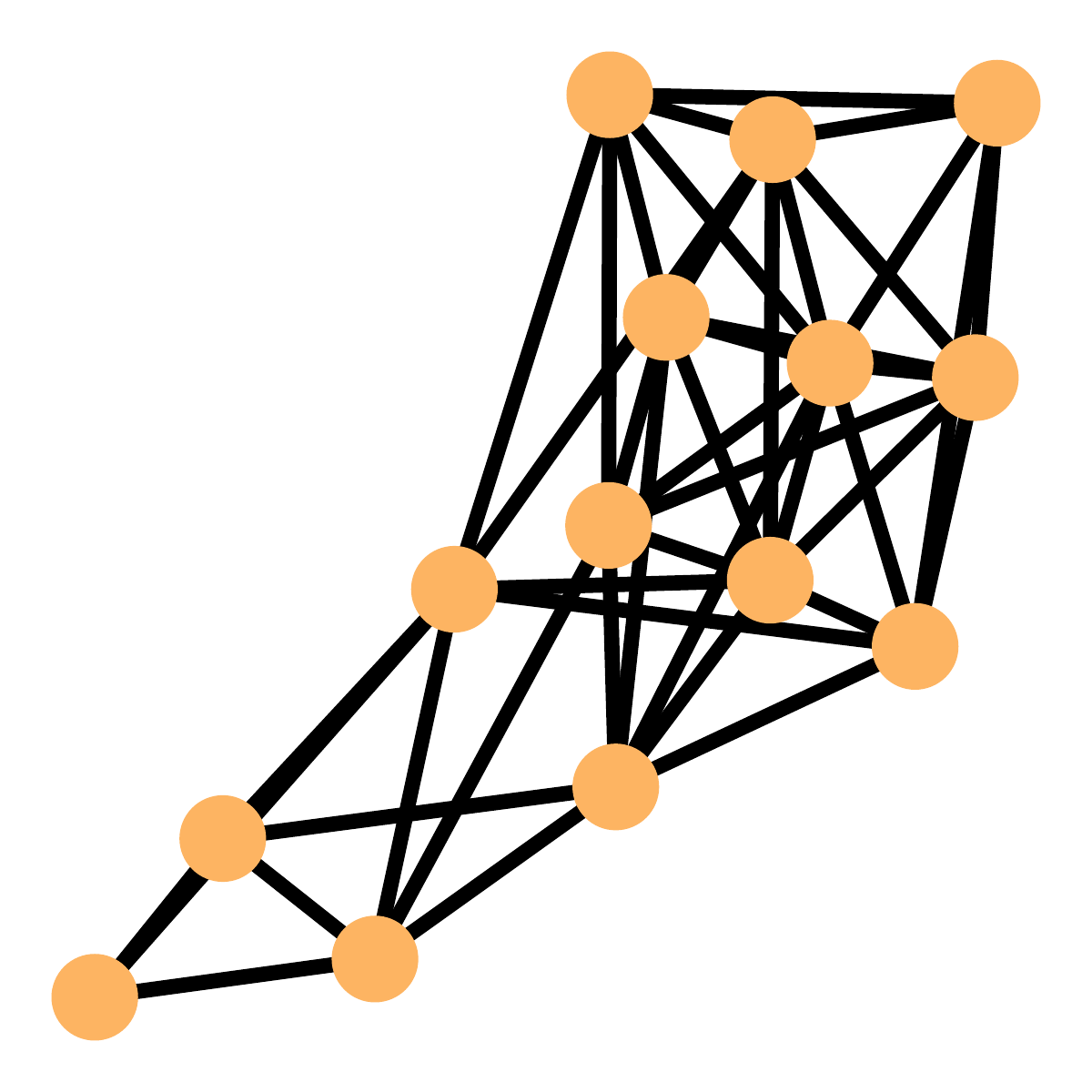}} & \adjustbox{valign=c}{\includegraphics[scale=0.135]
{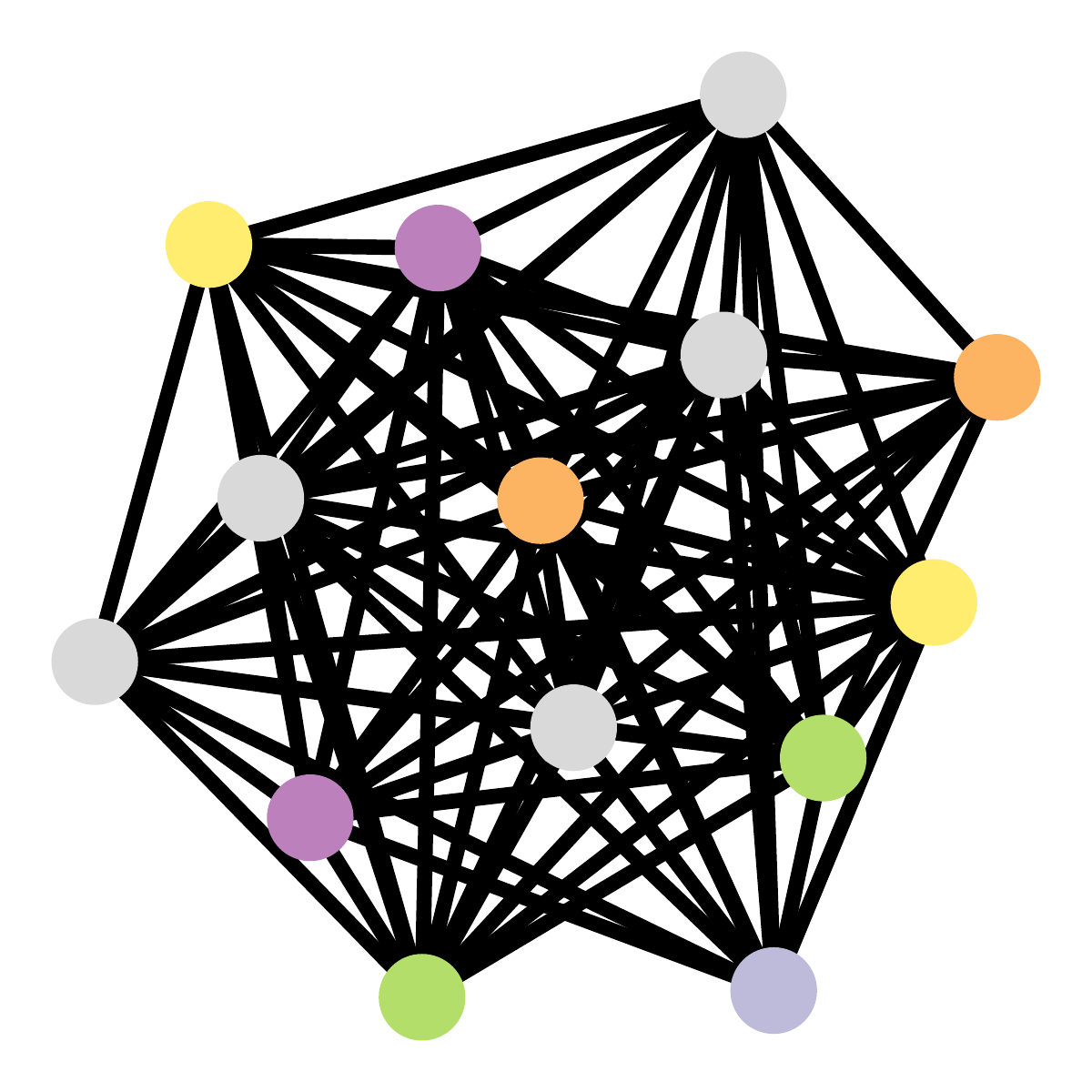}} & \adjustbox{valign=c}{\includegraphics[scale=0.135]
{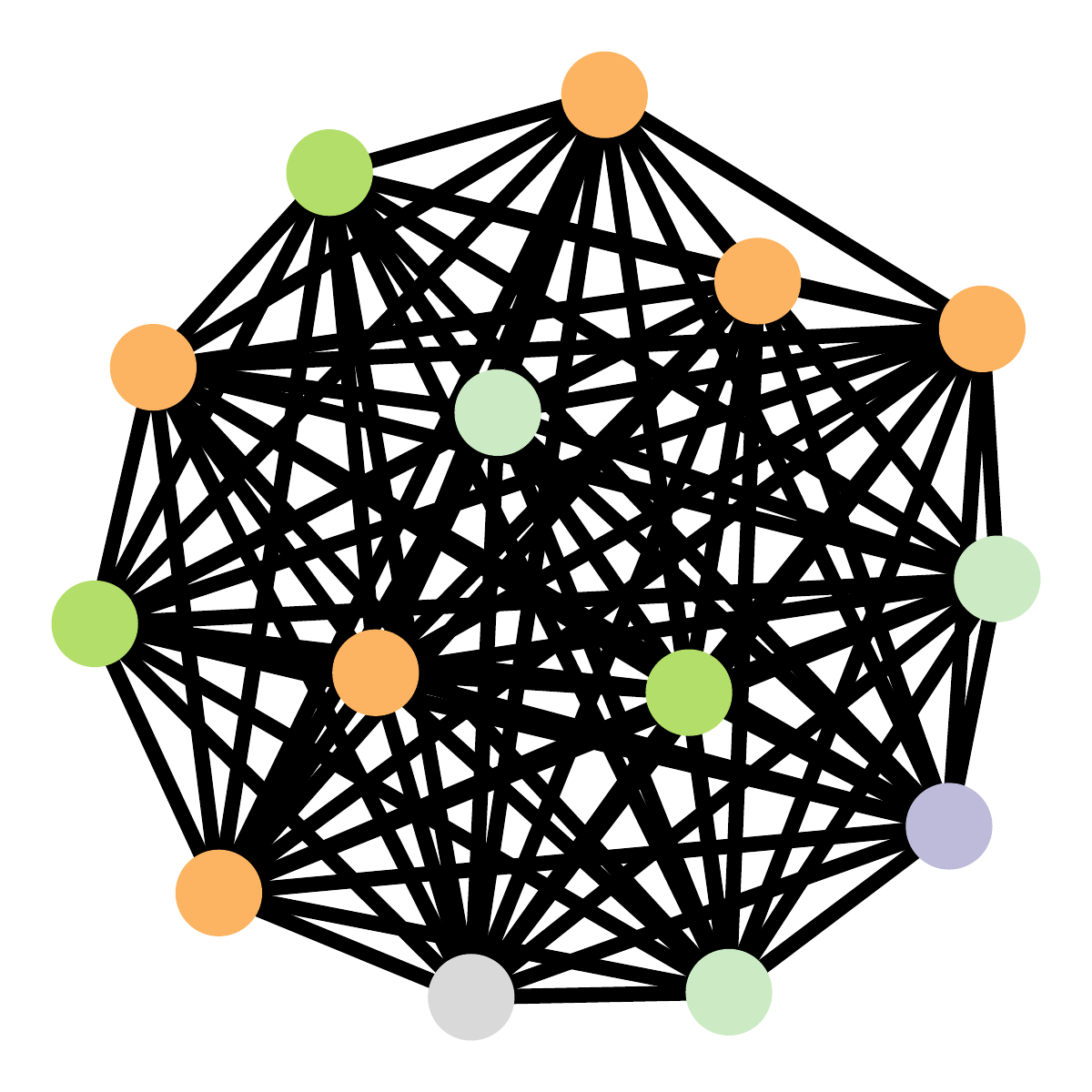}} & \adjustbox{valign=c}{\includegraphics[scale=0.135] 
{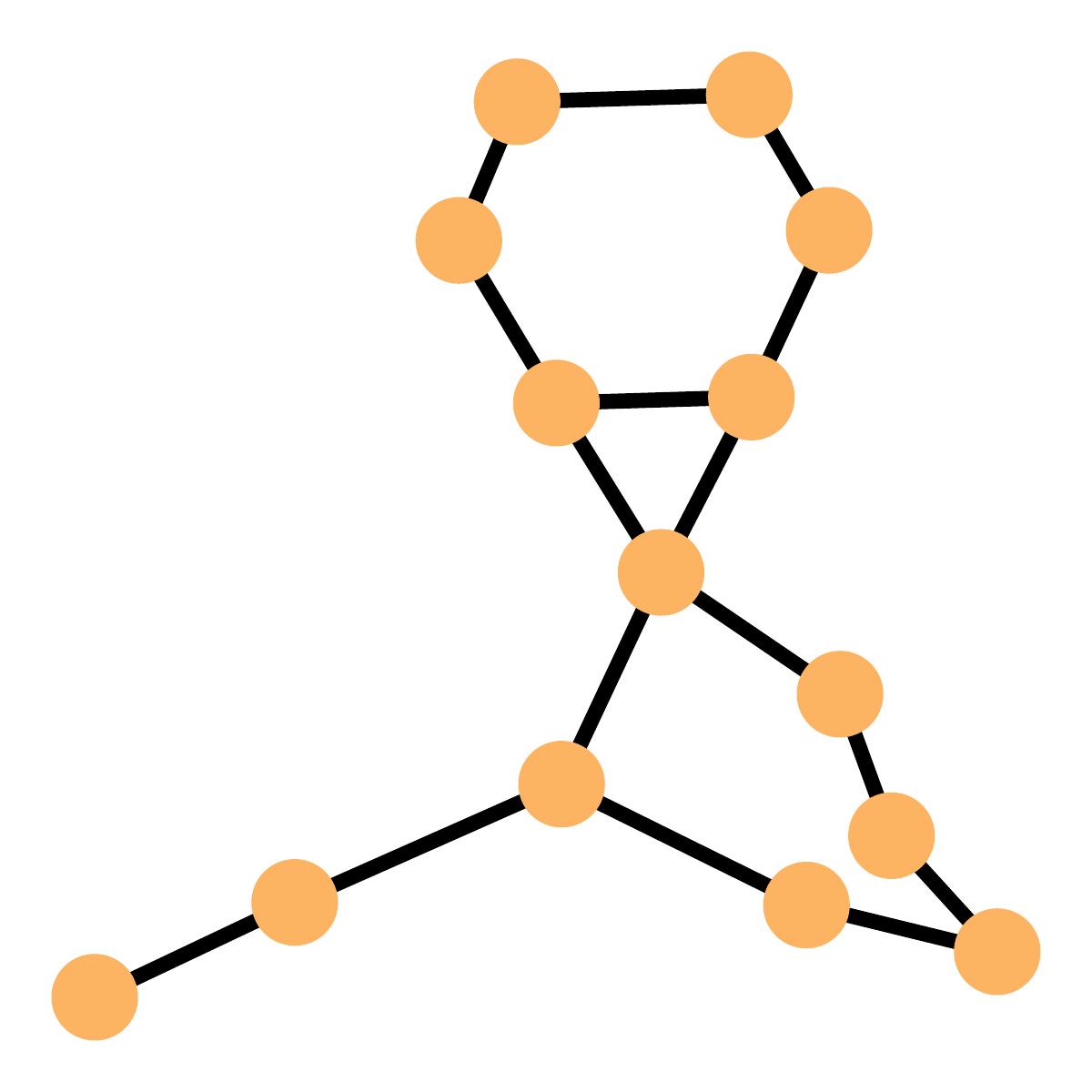}} \\ \bottomrule[1.5pt]
\end{tabular}%
}
\vspace{-1.5ex}
\end{table*}

To answer \textbf{Q1}, we compare our method with eight baselines in terms of recovering graph structure and features.
The results are presented in \tabref{tab:baseline}.
Overall, our \sysname outperforms the baselines in structure and feature recoveries, with a reduction of over $5.46\%$ in node features by MSE and an increase of over $25.04\%$ in the adjacency matrix by AUC.
A case study is also provided in \tabref{tab:baseline_visual} to demonstrate this superiority.
Detailed analysis is as follows.

\fakeparagraph{Results on Graph Structure}
Regarding the graph structure recovery, significant improvements are observed across all metrics compared to baselines.
As depicted in \tabref{tab:baseline}, the AUC scores of the baselines are around $0.5$ across all datasets, indicating performance akin to random guessing, while our proposed \sysname achieves AUC values ranging from $0.66$ to $0.9$.
This validates the effectiveness of our design in \secref{sec:method}, where we focus on initially recovering only the graph structure with leaked graph embeddings and priors, thereby simplifying the recovery process.
As baseline methods rely on iterative optimizations to recover edges from gradients, this task remains highly challenging due to the mixed feature and structural information in gradients, leading to poor performance.

\fakeparagraph{Results on Node Features}
We have observed notable improvements in feature recovery, as shown in \tabref{tab:baseline}, particularly due to a reduction in MSE.
Methods employing priors, such as GI-GAN, GRA-GRF, and our proposed \sysname, consistently outperform those based solely on optimization process with gradients, reducing MSE errors by at least $0.1\times$.
Among these three methods, our \sysname stands out by solving closed-form equations that clarify the relationship between gradients and features, leading to an average MSE reduction of $17.6\%$.
Our method also achieves the highest accuracy in terms of ACC over all datasets, particularly excelling with the AIDS dataset, which includes 37 node types.

\fakeparagraph{Case Study}
The visualizations of graph recovery in \tabref{tab:baseline_visual} demonstrate the strong capability of our approach in recovering the graph data.
Specifically, the AUC of our \sysname is 0.96, whereas others range from 0.49$\sim$to 0.60.
From the visualizations, we can observe that other baselines tend to reconstruct fully connected graphs, which are far away from the ground-truth graph structure to be reconstructed. 
Furthermore, the visualization results support that higher AUC, AP, and ACC metrics consistently correspond to better graph structure recovery performance.
Similar observations can also be found in additional cases, as shown in \appref{sec:app_visualization}.

\subsection{Ablation Studies}

\begin{figure}[t]
        \subfloat[\scriptsize{AUC of recovered structure}]{
		\includegraphics[scale=0.235]{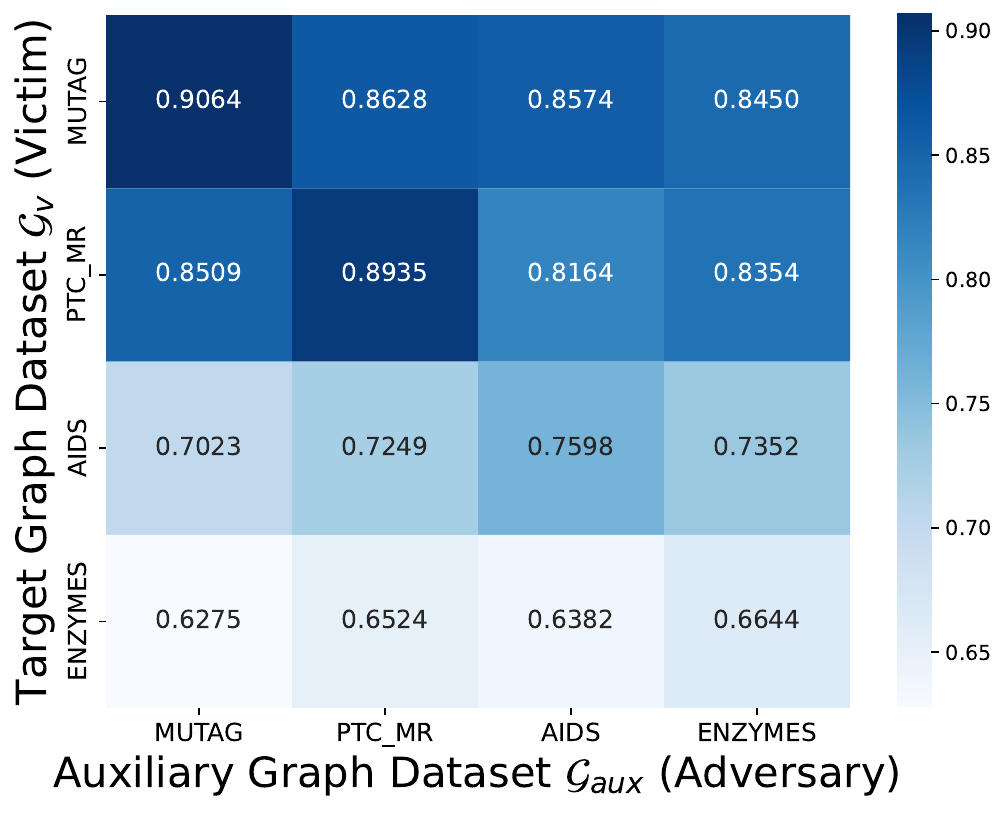}
		\label{fig:abla_auc}
	}
	\subfloat[\scriptsize{AP of recovered structure}]{
		\includegraphics[scale=0.235]{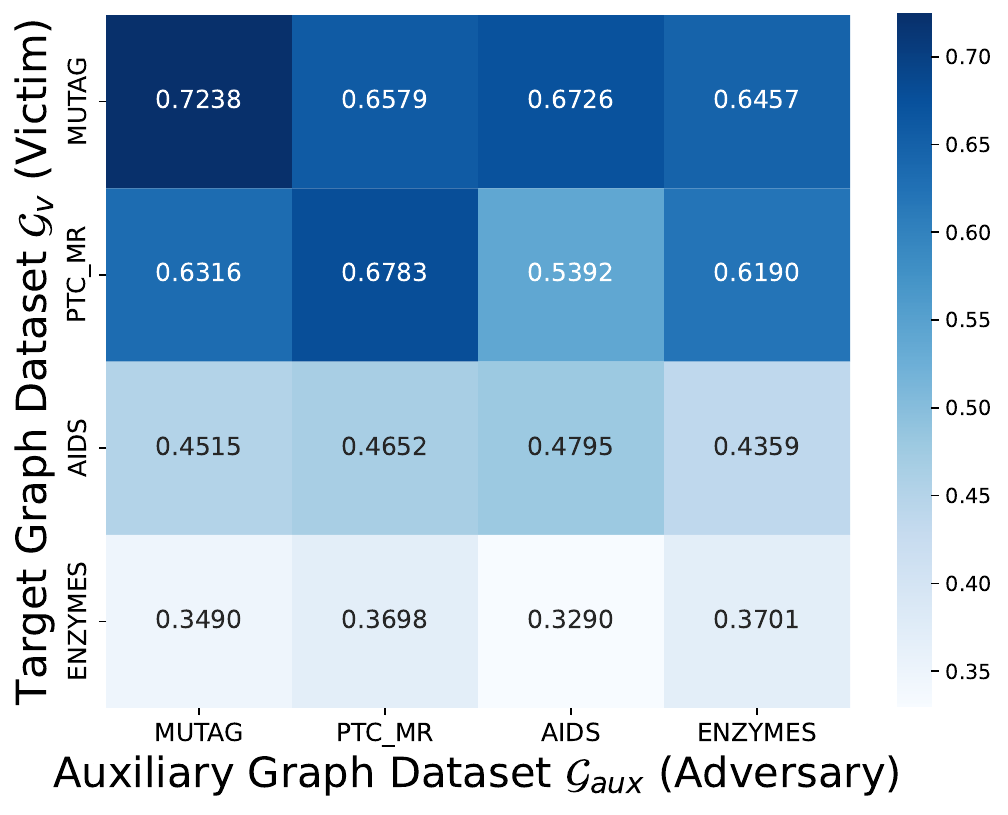}
		\label{fig:abla_ap}
	}
    \caption{\small{Structure recovery performance with hetero-priors, \ie target and auxiliary graphs datasets are different.}}
	\label{fig:abla}
    
\end{figure}


To answer \textbf{Q2}, we utilize four datasets to explore more heterogeneous setting scenarios, where the adversary employs \sysname with diverse auxiliary graphs (\eg from MUTAG) to recover victim clients' private graphs from a different dataset (\eg AIDS).
Unlike previous studies, which only used prior auxiliaries from the same target dataset, we fill the gap by exploring the effects of heterogeneous priors. 
Our findings show that this heterogeneity affects the recovery of graph data, and we further confirm that our designed adapter enhances recovery effectiveness in such scenarios.

\fakeparagraph{Impact of Hetero-Priors}
We assess the recovery performance using \sysname without the designed adapter to determine the impact of heterogeneous priors. 
As depicted in \figref{fig:abla}, the AUC and AP metrics of recovered structures reach their peak values for the same target dataset only when the target and auxiliary datasets are drawn from the same distribution (\ie, values along the diagonal of the heatmap), showing that diversity between priors and target graphs could hinder the recovery of the graph structure.
We also observe that \sysname consistently achieves higher AUC and AP scores than all other baselines, when using various datasets (MUTAG, PTC\_MR, AIDS, ENZYMES) as the auxiliary graphs.

\fakeparagraph{Results on Random Auxiliary Graph}
We conduct experiments where the MUTAG dataset serves as target dataset, and the auxiliaries are randomly generated graphs.
Specifically, we employ the widely-used Erdős–Rényi graph model~\cite{ER60_Graph_Model} to generate graph structures ($\mathcal{G}(0.5)$ denotes that any two nodes are connected independently with a probability of 0.5), while node features are sampled from Gaussian or Uniform distributions.
As shown in \tabref{tab:random}, when the MUTAG auxiliary dataset is replaced with these randomly generated ones, the adjacency matrix accuracy of our method decreases slightly from 93.41\% to 87.73\% and it still outperforms other baselines.
Moreover, the accuracy of node features of our \sysname remains stable, showing its robustness in feature recovery against variations in recovery results of the graph structures.


\begin{table}[h]
    \setlength{\tabcolsep}{2pt} 
    \caption{\small{Results on randomly generated auxiliary graphs.}} \label{tab:random} 
    \resizebox{0.95\columnwidth}{!}{
    \begin{tabular}{@{}c|c|cc|ccc@{}} 
    \toprule[1.5pt] 
      \multirow{2}{*}{\textbf{Auxilary Feature}} &
      \multirow{2}{*}{\textbf{Auxilary Structure}} &
      \multicolumn{2}{c|}{\textbf{Node Feature}} &
      \multicolumn{3}{c}{\textbf{Adjacency Matrix}} \\ \cline{3-7} 
     &
       &
      \textbf{MSE} &
      \multicolumn{1}{c|}{\textbf{ACC(\%)}} &
      \textbf{AUC} &
      \textbf{AP} &
      \textbf{ACC(\%)} \\ 
    
    \toprule[1.5pt]
    MUTAG               & MUTAG  & 0.0844 & 74.35 & 0.9042 & 0.7556 & 93.41 \\ \hline
    $\mathcal{N}(0,1)$  & $\mathcal{G}(0.1)$  & 0.0844 & 74.35 & 0.5917 & 0.1691 & 87.73 \\ \hline
    $\mathcal{N}(0,1)$  &  $\mathcal{G}(0.05)$ & 0.0844 & 74.35 & 0.6155 & 0.1765 & 87.73 \\ \hline
    $\mathcal{U}[-1,1]$ &  $\mathcal{G}(0.1)$ & 0.0844 & 74.35 & 0.5762 & 0.1645 & 87.73 \\ \hline
    $\mathcal{U}[-1,1]$ & $\mathcal{G}(0.05)$ & 0.0844 & 74.35 & 0.6080 & 0.1862 & 87.73 \\
    \bottomrule[1.5pt] 
    \end{tabular}
    }
\end{table}

\fakeparagraph{Effectiveness of Adapter}
Since the adapter primarily impacts structure recovery, we present the results of an ablation study on the recovered structure in \tabref{tab:abla_adapt}.
By introducing the adapter $L_{adapt}$ into our proposed method, we observe higher AUC and AP values across all settings, with increases of up to $5.19\%$ in AUC and up to $16.43\%$ in AP. 
This demonstrates the effectiveness and robustness of the proposed regularization, which reduces the disparity between target and auxiliary graphs,  allowing the decoder to perform better on target graphs.
The above observations suggest that the graph data leakage from gradients is risky, even when target and auxiliary graphs are from heterogenous datasets.

\begin{table}[htb]
\caption{\small{Ablation results on the impact of adapter in \sysname}.}
\label{tab:abla_adapt}
\resizebox{0.80\columnwidth}{!}{%
    \begin{tabular}{@{}c|c|c|cc@{}}
    \toprule[1.5pt]
    \textbf{Targets} & \textbf{Auxiliaries} & \textbf{Method} & \textbf{AUC} & \textbf{AP} \\
    \toprule[1.5pt]
    \multirow{6}{*}{MUTAG} & \multirow{2}{*}{PTC\_MR} & w/o $L_{adapt}$ & 0.8628 & 0.6579 \\
     &  & \textbf{\sysname (\textbf{ours})} & \textbf{0.8778} & \textbf{0.6898} \\ \cmidrule(l){2-5} 
     & \multirow{2}{*}{AIDS} & w/o $L_{adapt}$ & 0.8574 & 0.6726 \\
     &  & \textbf{\sysname (\textbf{ours})} & \textbf{0.8589} & \textbf{0.6833} \\ \cmidrule(l){2-5} 
     & \multirow{2}{*}{ENZYMES} & w/o $L_{adapt}$ & 0.8450 & 0.6457 \\
     &  & \textbf{\sysname (\textbf{ours})} & \textbf{0.8604} & \textbf{0.6522} \\ \midrule
    \multirow{6}{*}{PTC\_MR} & \multirow{2}{*}{MUTAG} & w/o $L_{adapt}$ & 0.8509 & 0.6316 \\
     &  & \textbf{\sysname (\textbf{ours})} & \textbf{0.8543} & \textbf{0.6348} \\ \cmidrule(l){2-5} 
     & \multirow{2}{*}{AIDS} & w/o $L_{adapt}$ & 0.8164 & 0.5392 \\
     &  & \textbf{\sysname (\textbf{ours})} & \textbf{0.8588} & \textbf{0.6278} \\ \cmidrule(l){2-5} 
     & \multirow{2}{*}{ENZYMES} & w/o $L_{adapt}$ & 0.8354 & 0.6190 \\
     &  & \textbf{\sysname (\textbf{ours})} & \textbf{0.8614} & \textbf{0.6447} \\ \midrule
    \multirow{6}{*}{AIDS} & \multirow{2}{*}{MUTAG} & w/o $L_{adapt}$ & 0.7023 & 0.4515 \\
     &  & \textbf{\sysname (\textbf{ours})} & \textbf{0.7048} & \textbf{0.4611} \\ \cmidrule(l){2-5} 
     & \multirow{2}{*}{PTC\_MR} & w/o $L_{adapt}$ & 0.7249 & 0.4652 \\
     &  & \textbf{\sysname (\textbf{ours})} & \textbf{0.7561} & \textbf{0.4830} \\ \cmidrule(l){2-5} 
     & \multirow{2}{*}{ENZYMES} & w/o $L_{adapt}$ & 0.7352 & 0.4359 \\
     &  & \textbf{\sysname (\textbf{ours})} & \textbf{0.7416} & \textbf{0.4401} \\ \midrule
    \multirow{6}{*}{ENZYMES} & \multirow{2}{*}{MUTAG} & w/o $L_{adapt}$ & 0.6275 & 0.3490 \\
     &  & \textbf{\sysname (\textbf{ours})} & \textbf{0.6278} & \textbf{0.3499} \\ \cmidrule(l){2-5} 
     & \multirow{2}{*}{PTC\_MR} & w/o $L_{adapt}$ & 0.6524 & 0.3698 \\
     &  & \textbf{\sysname (\textbf{ours})} & \textbf{0.6562} & \textbf{0.3795} \\ \cmidrule(l){2-5} 
     & \multirow{2}{*}{AIDS} & w/o $L_{adapt}$ & 0.6382 & 0.3290 \\
     &  & \textbf{\sysname (\textbf{ours})} & \textbf{0.6555} & \textbf{0.3727} \\ \bottomrule[1.5pt]
    \end{tabular}%
}
\end{table}

\subsection{Evaluation under Defenses}

\begin{figure}[htb]
	  \centering
	\includegraphics[width=1.0\columnwidth,trim=0 0 0 60,clip]{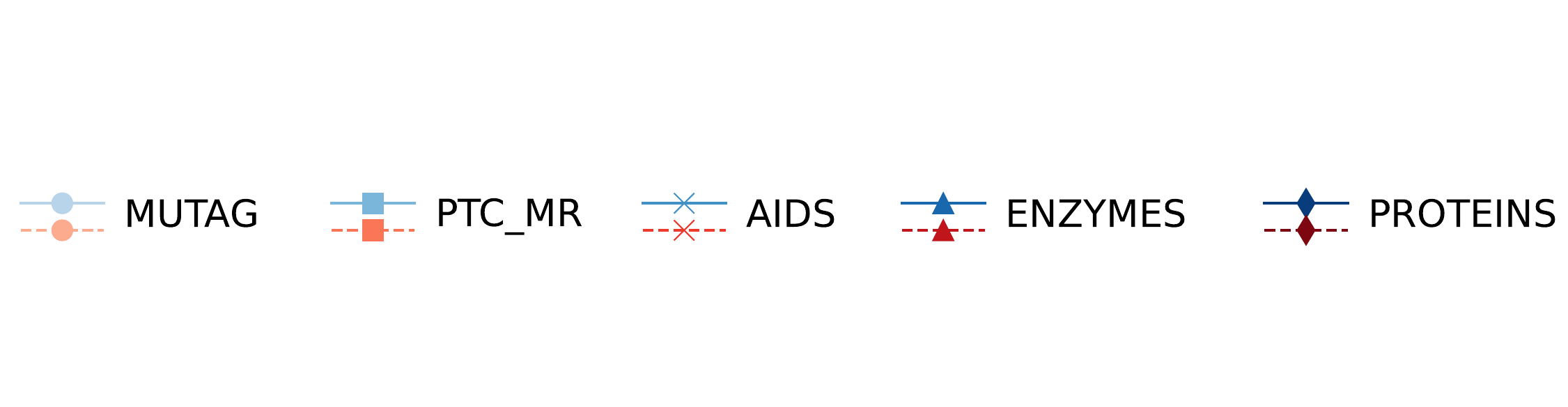}
    
    \subfloat[\fontsize{7}{6}\selectfont{Results on feature recovery \\ with gradient compression}]{
		\includegraphics[scale=0.22]{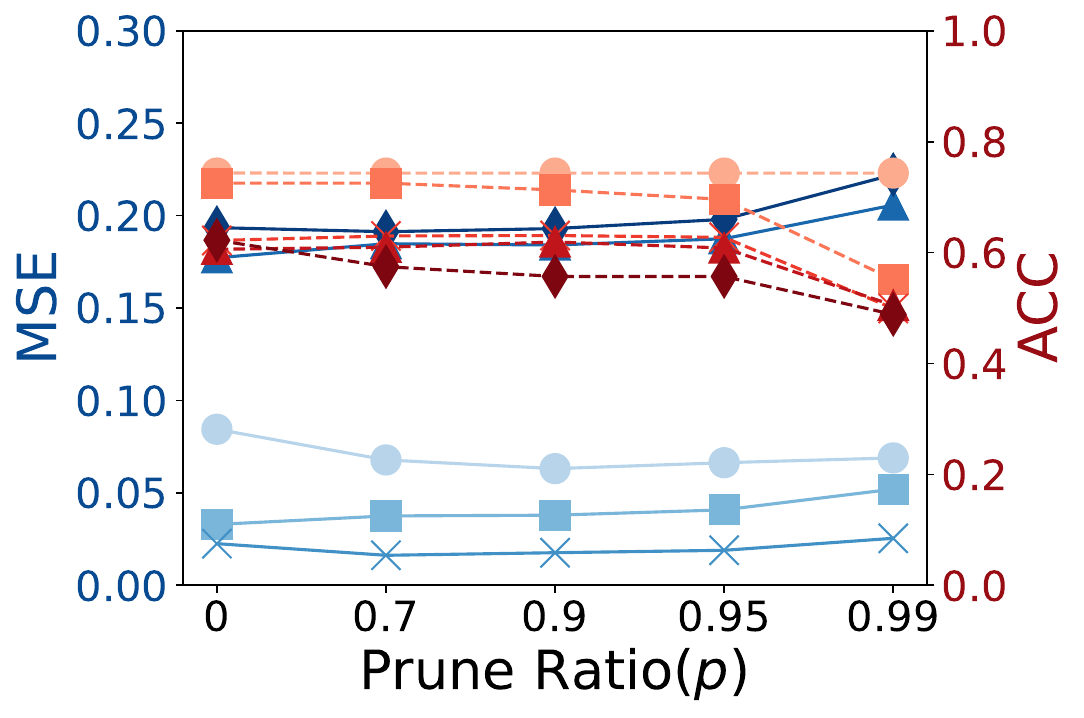}
		\label{fig:gc_feature}
	}
    \subfloat[\fontsize{7}{6}\selectfont{Results on feature recovery  \\ with differential privacy}]{
		\includegraphics[scale=0.22]{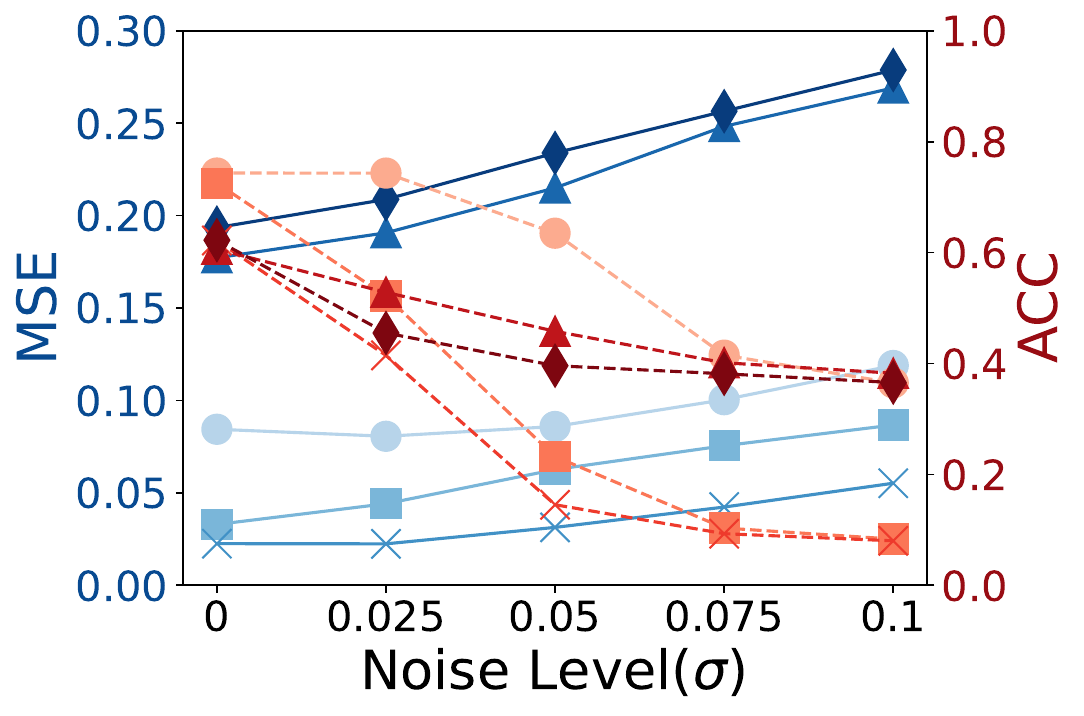}
		\label{fig:dp_feature}
	}
	\hfill
	  \centering
	\subfloat[\fontsize{7}{6}\selectfont{Results on structure recovery  \\ with gradient compression}]{
		\includegraphics[scale=0.22]{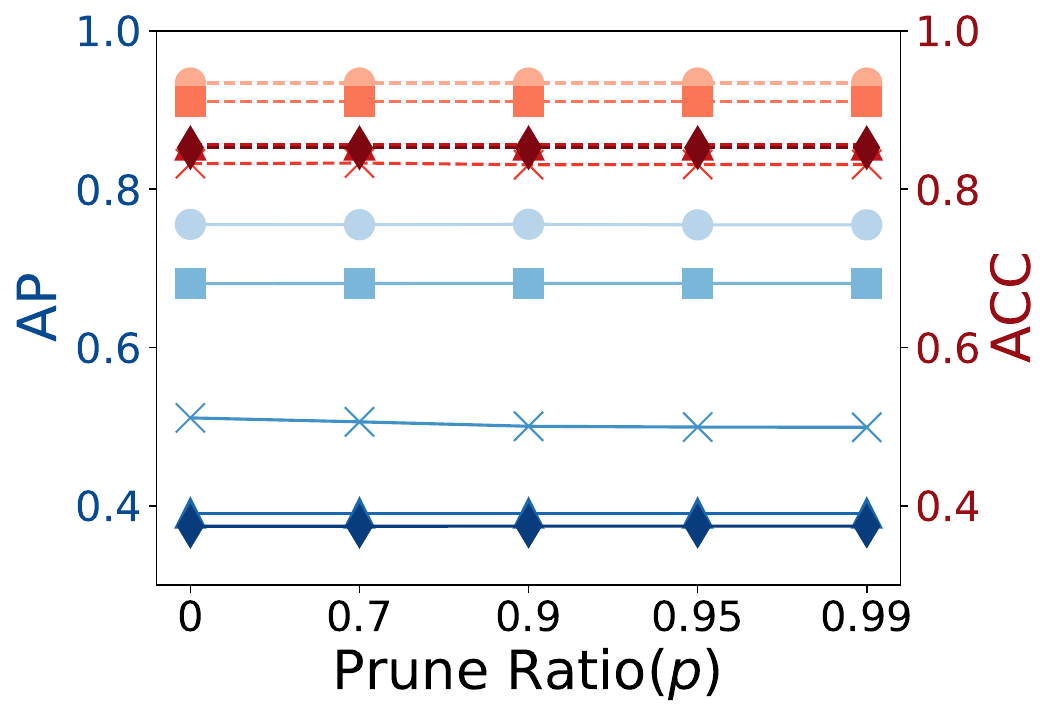}
		\label{fig:gc_structure}
	}
	\subfloat[\fontsize{7}{6}\selectfont{Results on structure recovery \\ with differential privacy}]{
		\includegraphics[scale=0.22]{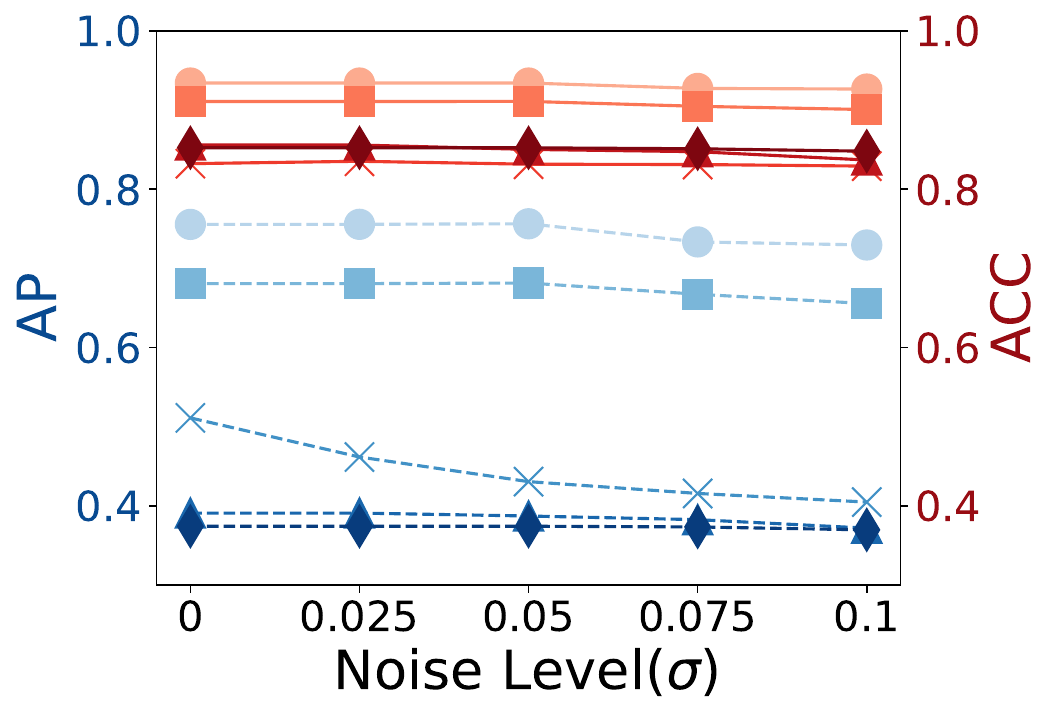}
		\label{fig:dp_structure}
	}
	\caption{\small{The effectiveness of different defense strategies.}}
	\label{fig:defense}
\end{figure}

	

To answer \textbf{Q3}, we evaluate the performance of \sysname under two commonly used defenses: \textit{(i)} Gradient compression\cite{nips19:dlg}, which prunes small values in gradients with prune ratio $p$; \textit{(ii)} Differential privacy\cite{iclr18:dp}, which adds noise to gradients, and here we use Laplacian noise with variance $\sigma$. 
The results are shown in \figref{fig:defense},
where solid lines denote MSE and dashed lines denote ACC.

\fakeparagraph{Defense with Gradient Compression}
The results are in \figref{fig:gc_feature}  indicate that the performance of feature recovery only exhibits a noticeable decrease when the prune ratio reaches as high as $0.99$. 
However, even with a high prune ratio, it is difficult to observe a performance drop in structure recovery, as shown in \figref{fig:gc_structure}.
This suggests that gradient compression has limited effectiveness against \sysname, potentially because the small values in gradients do not play a primary role in the privacy leakage exploited by \sysname.

\fakeparagraph{Defense with Differential Privacy}
\figref{fig:dp_feature} and \figref{fig:dp_structure} are results of different noise variances, where a greater variance means adding more noise, \ie, a higher level of privacy protection.
It is evident that as the noise level increases, the reconstruction performance decreases, particularly concerning the ACC of node features.
Yet, the decrease in AP and ACC of recovered graph structure is not pronounced, indicating that while differential privacy may be less effective in protecting graph structures than node features, requiring stronger defense designs.
We further discuss the potential defense methods for \sysname in \appref{sec:app_defense}.

\section{Conclusion}
\label{sec:conclusion}

In this paper, we systematically study the deep leakage from gradients (DLG) in federated graph learning (FGL).
Firstly, we uncover the relationship between graph structure and node features in DLG on graphs through the closed-form rules.
Then, we propose a novel framework, \sysname, to recover the original graph structure and node features of private training graphs from shared gradients in FGL.
\sysname advances the state-of-the-art by successfully disentangling structural and feature information from gradients based on the derived closed-form recursive rules.
We further enhance the reconstruction performance by incorporating the auto-encoder model with an adaptation regularization designed to leverage the heterogeneous auxiliary graph dataset.
Extensive experiments on five datasets demonstrate significant performance improvements compared with existing works, highlighting the substantial risks of data leakage from gradients in graph-level FGL.


\bibliographystyle{ACM-Reference-Format}
\balance
\bibliography{ref}

\newpage

\appendix

\section{appendix}
\label{sec:appendix}


\subsection{More Details for Data Leakage Analysis}
\label{sec:app_pool}

\fakeparagraph{Pooling as matrix operations}
Common pooling methods can be represented as the following matrix operation on the node embedding matrix of the $l$-th layer $H_l\in \mathbb{R}^{|V|\times d_{H}}$.
Specifically, for the sum pooling where $H_G = \sum_{v\in G}h_v$, it can be described as the sum of each row (\ie, each node) in $H_l$.
This operation is performed using the all-ones vector $M_{p\_sum}^\top  = \mathbf{1}^{|V|}\in \mathbb{R}^{|V|}$ and $H_G=M_{p\_sum}H_l$.
Similarly, for mean pooling where $H_G = \frac{1}{|V|}\sum_{v\in G}h_v$, it can be calculated as $H_G=\frac{1}{|V|}M_{p\_sum}H_l$.
As for the max pooling, we initially approximate the maximum function by $max(X)=\frac{1}{K}\log\sum_{i=1}^{D}e^{KX_i}, X\in \mathbb{R}^{D}$, where $K$ is an integer.
By incorporating it into max pooling where $H_G = \max_{v\in G}h_v$, the max pooling can be expressed as $\medop{H_G=\frac{1}{K}f_{\log}(M_{p\_sum}f_{\exp}(KH_l))}$, where $f_{\log}(\cdot)$ and $f_{\exp}(\cdot)$ represent the element-wise functions $\log(\cdot)$ and $\exp(\cdot)$, which only impacts $\nabla W_l$ in \lemmaref{lemma:gradients}.
Thus, we have $\medop{\nabla W_l=H_{l-1}^\top  \bar{A}^\top  ((M_p^\top (\frac{\partial L}{\partial \hat Y}W_{fc}\odot F_{\log}')\odot F_{\exp}')\odot\sigma_l')}$ and the recursive rule remains valid.

\fakeparagraph{Discussion on deriving $\sigma'$}
$\sigma'$ represents the derivation of the activation function, which in GNN is typically the ReLU function.
Since we calculate the node embeddings in a reverse manner, we can derive $\sigma'_{l-1}$ from its output $H_{l}$.
Specifically, for the ReLU function, the derivative is $1$ for each positive element in $H_{l}$ and $0$ otherwise.
In the last layer, obtaining node embeddings from the pooling result is challenging, so we initialize it as all ones.
This aims to leverage as much information as possible, considering that other components in the equations, such as the adjacency matrix, are often sparse.
While this may introduce some error, the experimental results in the paper demonstrate that our proposed approach achieves better performance compared with other baselines.


\subsection{Proofs of \lemmaref{lemma:gradients} and \thmref{thm:gradients}}
\label{sec:app_gradient}

We first introduce two basic facts~\cite{book17:matrix} that are widely used in matrix derivatives to aid the gradient computation in this work.

\begin{fact}
\label{lemma:trace}
The full differential of a function $f:\mathbb{R}^{n\times m}\rightarrow \mathbb{R}$ w.r.t. an $n\times m$ matrix $X$ can be expressed by trace and matrix derivatives as,
\begin{equation}
    \label{eq:trace}
    \medop{
        df(X) = \text{tr} (\frac{\partial f(X)}{\partial X}^\top  dX) = \text{tr} (\frac{\partial f(X)}{\partial X} dX^\top )
    }
\end{equation}
Here $\frac{\partial f(X)}{\partial X}$ is the derivative matrix of $f$ with respect to $X$ and $dX$ denotes the infinitesimal change in $X$.
$\text{tr}(\cdot)$ is the trace operation.
\end{fact}

\begin{fact}
\label{lemma:start}
For a function $f:\mathbb{R}^{n\times m}\rightarrow \mathbb{R}$ with respect to an $n\times m$ matrix $X$, the following equation holds:
\begin{equation}
    \medop{ 
       df(X) = \text{tr} (df(X))
    }
\end{equation}
\end{fact}

Based on above two facts, we can prove the \lemmaref{lemma:gradients} in \secref{sec:method_analysis}.
\setcounter{lemma}{\value{savelemma}}
\begin{lemma} [Gradients in an $l$-layer GCN]
    For the GCN layers in FGL, the gradients $\nabla W_l$, $\nabla W_{l-1}$, $\nabla W_{l-2}$ can be formally expressed by the node embeddings from their previous layer $H_{l-1}$, $H_{l-2}$, $H_{l-3}$ respectively and the normalized graph adjacency matrix $\bar A$,

        \begin{equation}\label{eq:gcn_g_l_app}
        \medop{ 
        \nabla{W_{l}} = H_{l-1}^\top \bar{A}^\top (M_p^\top\frac{\partial L}{\partial \hat Y}W_{fc}\odot\sigma_l') 
        }
        \hspace{13.6em}
    \end{equation}
    \begin{equation}\label{eq:gcn_g_l-1_app}
       \medop{
            \nabla{W_{l-1}} = H_{l-2}^\top \bar{A}^\top (\bar{A}^\top (M_p^\top\frac{\partial L}{\partial \hat Y}W_{fc}\odot\sigma_l')W_l^\top\odot \sigma_{l-1}') 
        }
        \hspace{6em}
    \end{equation}
    \begin{equation} \label{eq:gcn_g_l-2_app}
        \medop{
            \nabla{W_{l-2}} = H_{l-3}^\top \bar{A}^\top (\bar{A}^\top(\bar{A}^\top (M_p^\top\frac{\partial L}{\partial \hat Y}W_{fc}\odot\sigma_l')W_l^\top\odot \sigma_{l-1}')W_{l-1}^\top\odot \sigma_{l-2}')
        }
    \end{equation}
    
    
\end{lemma}

\begin{proof}
Following \factref{lemma:start}, we calculate the full differential of the loss function w.r.t $W_l$ as $ d L = \text{tr}(dL)= \text{tr}(\frac{dL}{d\hat Y}d \hat Y)$.
By \equref{eq:modeling_gcn}-\equref{eq:modeling_mlp}, the above loss function can be further transformed as follows:
\begin{equation}
    \begin{aligned}
    \medop{d L 
        =\text{tr}(\frac{dL}{d\hat Y} W_{fc} d(H_l^\top ) M_p^\top )  
        =\text{tr}(M_p^\top \frac{dL}{d\hat Y} W_{fc} d(\sigma_l(W_l^\top  H_{l-1}^\top \bar A^\top )) )
     }
    \end{aligned} 
\end{equation}

As the activation function $\sigma$ includes the element-wise Hadamard product, we can use the property of the trace of the product, $tr(A(B\odot C))=tr((A\odot B^\top ) C)$, to further derive the loss function as:
\begin{equation}
    \label{eq:fornext}
    \begin{aligned}
    &\medop{d L 
        = \text{tr}(M_p^\top \frac{dL}{d\hat Y} W_{fc} ((\sigma_l^\top )' \odot d(W_l^\top  H_{l-1}^\top \bar A^\top ))) 
    }
    \\
    &
    \medop{ = 
    \text{tr}((M_p^\top \frac{dL}{d\hat Y} W_{fc} \odot \sigma_l') d(W_l^\top  H_{l-1}^\top \bar A^\top ))  
     = \text{tr}(H_{l-1}^\top \bar A^\top (M_p^\top \frac{dL}{d\hat Y} W_{fc} \odot \sigma_l') dW_l^\top )
        }
    \end{aligned}
\end{equation}


According to \factref{lemma:trace}, we have following result,
\begin{equation}
    \medop{\nabla{W_l}=\frac{\partial L}{\partial W_l}
               =H_{l-1}^\top \bar A^\top (M_p^\top \frac{dL}{d\hat Y} W_{fc} \odot \sigma_l')
    }
\end{equation}

Extending $H_{l-1}$ in the last line of  \equref{eq:fornext} to the previous layer using the propagation rule of GCN, we can further calculate the full differential of the loss function with respect to $W_{l-1}$:
\begin{equation}
    \begin{aligned}
    &\medop{d L 
        = \text{tr}(\bar A^\top (M_p^\top \frac{dL}{d\hat Y} W_{fc} \odot \sigma_l') W_l^\top dH_{l-1}^\top )}
    \\
    &\medop{
        = \text{tr}(\bar A^\top (M_p^\top \frac{dL}{d\hat Y} W_{fc} \odot \sigma_l') W_l^\top d(\sigma_{l-1}(W_{l-1}^\top H_{l-2}^\top \bar A^\top ))) 
    }
    \end{aligned}
\end{equation}

Similar to \equref{eq:fornext}, we can derive $dW_{l-1}$ as follows:
\begin{equation}\label{eq:fornextl-2}
    \begin{aligned}
    &\medop{d L 
        = \text{tr}(\bar A^\top (M_p^\top \frac{dL}{d\hat Y} W_{fc} \odot \sigma_l') W_l^\top ((\sigma_{l-1}^\top)' \odot d(W_{l-1}^\top H_{l-2}^\top \bar A^\top )))
    } 
    \\
    & \medop{
    = \text{tr}((\bar A^\top (M_p^\top \frac{dL}{d\hat Y} W_{fc} \odot \sigma_l') W_l^\top \odot \sigma_{l-1}') d(W_{l-1}^\top H_{l-2}^\top \bar A^\top ))
    }
    \\
    & \medop{ 
        = \text{tr}(H_{l-2}^\top \bar A^\top (\bar A^\top (M_p^\top \frac{dL}{d\hat Y} W_{fc} \odot \sigma_l') W_l^\top \odot \sigma_{l-1}') dW_{l-1}^\top )
    }
    \\
    \end{aligned}
\end{equation}

Therefore, we can get the following equation based on \factref{lemma:trace}:
\begin{equation}
    \medop{\nabla{W_{l-1}}=\frac{\partial L}{\partial W_{l-1}}
               =H_{l-2}^\top  \bar{A}^\top  (\bar{A}^\top  (M_p^\top \frac{\partial L}{\partial \hat Y}W_{fc}\odot\sigma_l')W_l^\top \odot \sigma_{l-1}')
         }
\end{equation}

Similarly, gradients for $W_{l-2}$ can be obtained by extending $H_{l-2}$,
\begin{equation}\label{eq:base_case_trace}
    \begin{aligned}
    &\medop{d L 
        = \text{tr}(\bar A^\top(\bar A^\top (M_p^\top \frac{dL}{d\hat Y} W_{fc} \odot \sigma_l') W_l^\top \odot \sigma_{l-1}') W_{l-1}^\top dH_{l-2}^\top ))
        }\\
        &\medop{= \text{tr}(\mathbf{r}_{l-1} W_{l-1}^\top dH_{l-2}^\top ) 
        = \text{tr}(\mathbf{r}_{l-1} W_{l-1}^\top d(\sigma_{l-2}(W_{l-2}^\top H_{l-3}^\top \bar A^\top)) ),
        }
    \end{aligned}
\end{equation}
where $\mathbf{r}_{l-1}$ is defined as $\bar A^\top(\bar A^\top (M_p^\top \frac{dL}{d\hat Y} W_{fc} \odot \sigma_l') W_l^\top \odot \sigma_{l-1}')$ to enhance the clarity of the equation. 
Once again, leveraging the noted property of the Hadamard operation, we derive $dW_{l-2}$ as follows:
\begin{equation}\label{eq:induc_hadam}
    \begin{aligned}
    &\medop{d L 
        = \text{tr}(\mathbf{r}_{l-1} W_{l-1}^\top ((\sigma_{l-2}^\top)'\odot d(W_{l-2}^\top H_{l-3}^\top \bar A^\top)) )
    }
    \\
    &\medop{
        = \text{tr}((\mathbf{r}_{l-1} W_{l-1}^\top \odot \sigma_{l-2}')d(W_{l-2}^\top H_{l-3}^\top \bar A^\top) )
        = \text{tr}(H_{l-3}^\top \bar A^\top(\mathbf{r}_{l-1} W_{l-1}^\top \odot \sigma_{l-2}')dW_{l-2}^\top ).
        }\\
    \end{aligned}
\end{equation}

Then, we obtain the gradients $\nabla{W_{l-2}}$,
\begin{equation}
    \medop{
        \nabla{W_{l-2}}= \frac{\partial L}{\partial W_{l-2}}
               =H_{l-3}^\top \bar A^\top(\bar A^\top(\bar A^\top (M_p^\top \frac{dL}{d\hat Y} W_{fc} \odot \sigma_l') W_l^\top \odot \sigma_{l-1}')W_{l-1}^\top \odot \sigma_{l-2}')
        }.
        \notag
\end{equation}
\end{proof}

Next, we prove the \thmref{thm:gradients} in \secref{sec:method_analysis}, focusing on extending the recursive rule observed in \lemmaref{lemma:gradients} to any layer of GCNs.

\begin{table*}[bth]
\caption{Visualization of the reconstruction performance for \sysname and baselines.
The first graph for each row represents the ground truth.
In each graph, nodes of different colors indicate different node types.}
\vspace{-1em}
\label{tab:app_visualization}
\resizebox{0.85\textwidth}{!}{%
\begin{tabular}{@{}ccccccccccc@{}}
\toprule[1.5pt]
\textbf{} & \textbf{} & \textbf{Random} & \textbf{DLG} & \textbf{iDLG} & \textbf{InverGrad} & \textbf{GI-GAN} & \textbf{GRA-GRF} & \textbf{TabLeak} & \textbf{Graph-Attacker} & \textbf{\sysname (ours)} \\ \toprule[1.5pt]
\textbf{Node Feature} & \multicolumn{1}{c|}{\textbf{MSE/ACC(\%)}} & 0.3441/20.00 & 1.0625/40.00 & 0.5000/60.00 & 0.3879/60.00 & 0.2680/20.00 & 0.2209/60.00  & 0.4717/40.00 & 0.9629/40.00 & \textbf{0.0988/100.00} \\
\textbf{Adjacency Matrix} & \multicolumn{1}{c|}{\textbf{AUC/AP}} & 0.4444/0.7336 & 0.8730/0.9135 & 0.6984/0.8822 & 0.8254/0.9115 & 0.4444/0.6984 & 0.3651/0.6735 & 0.8254/0.9041 & 0.7897/0.8706 & \textbf{1.0000/1.0000} \\
\textbf{\begin{tabular}[c]{@{}c@{}}Ground-truth/\\ Reconstructed graph\end{tabular}} & \multicolumn{1}{c|}{\adjustbox{valign=c}{\includegraphics[scale=0.105]{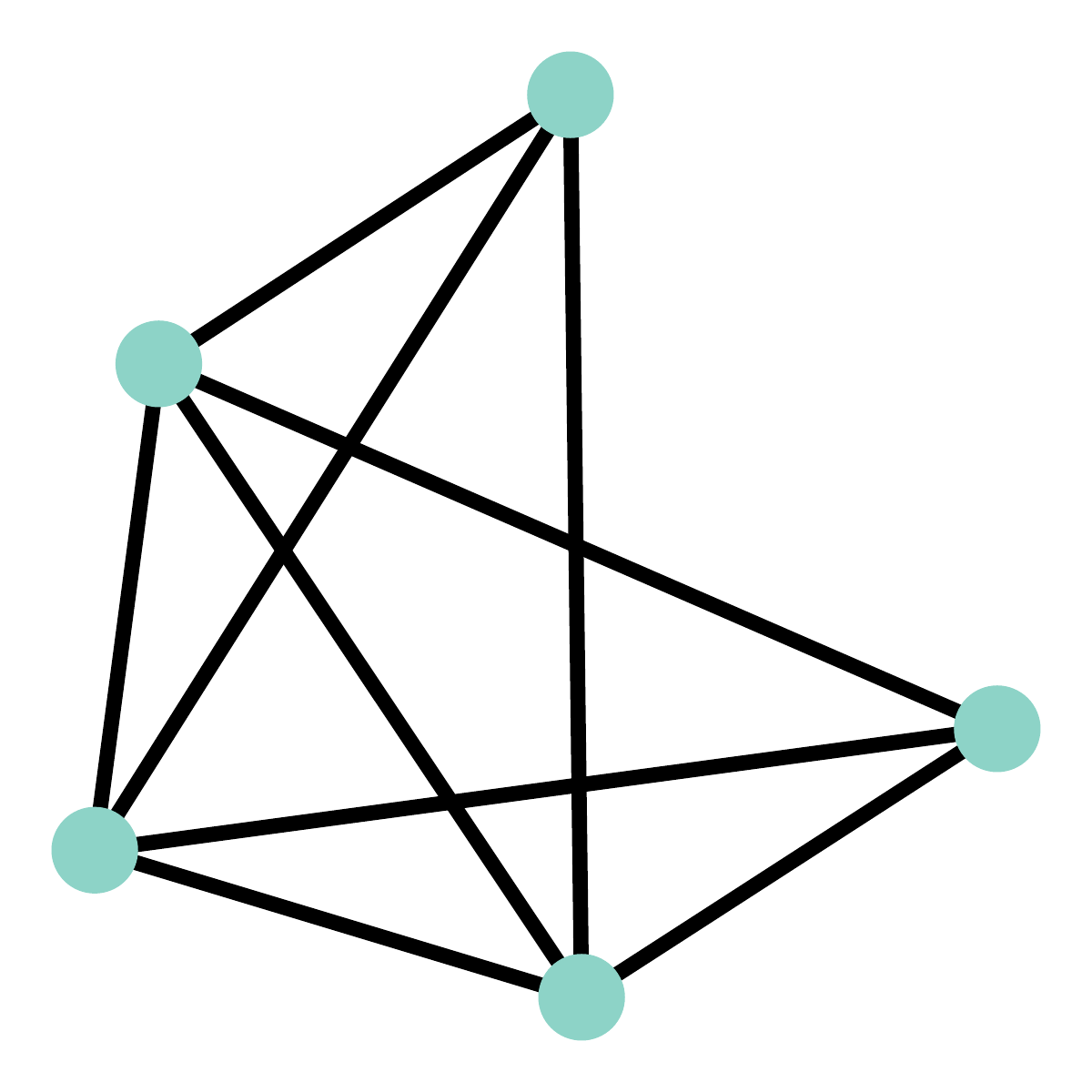}}} & \adjustbox{valign=c}{\includegraphics[scale=0.105]{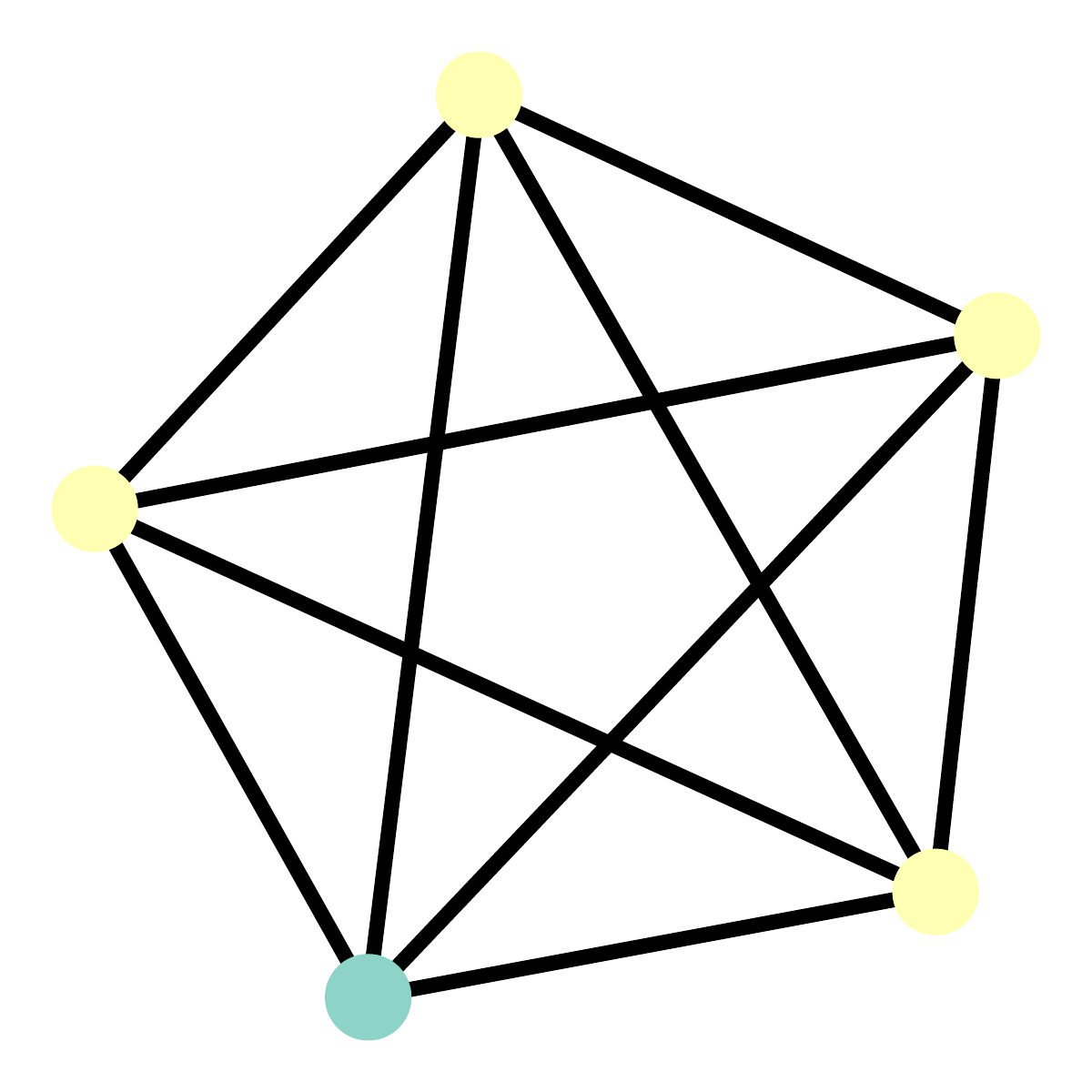}} & \adjustbox{valign=c}{\includegraphics[scale=0.105]{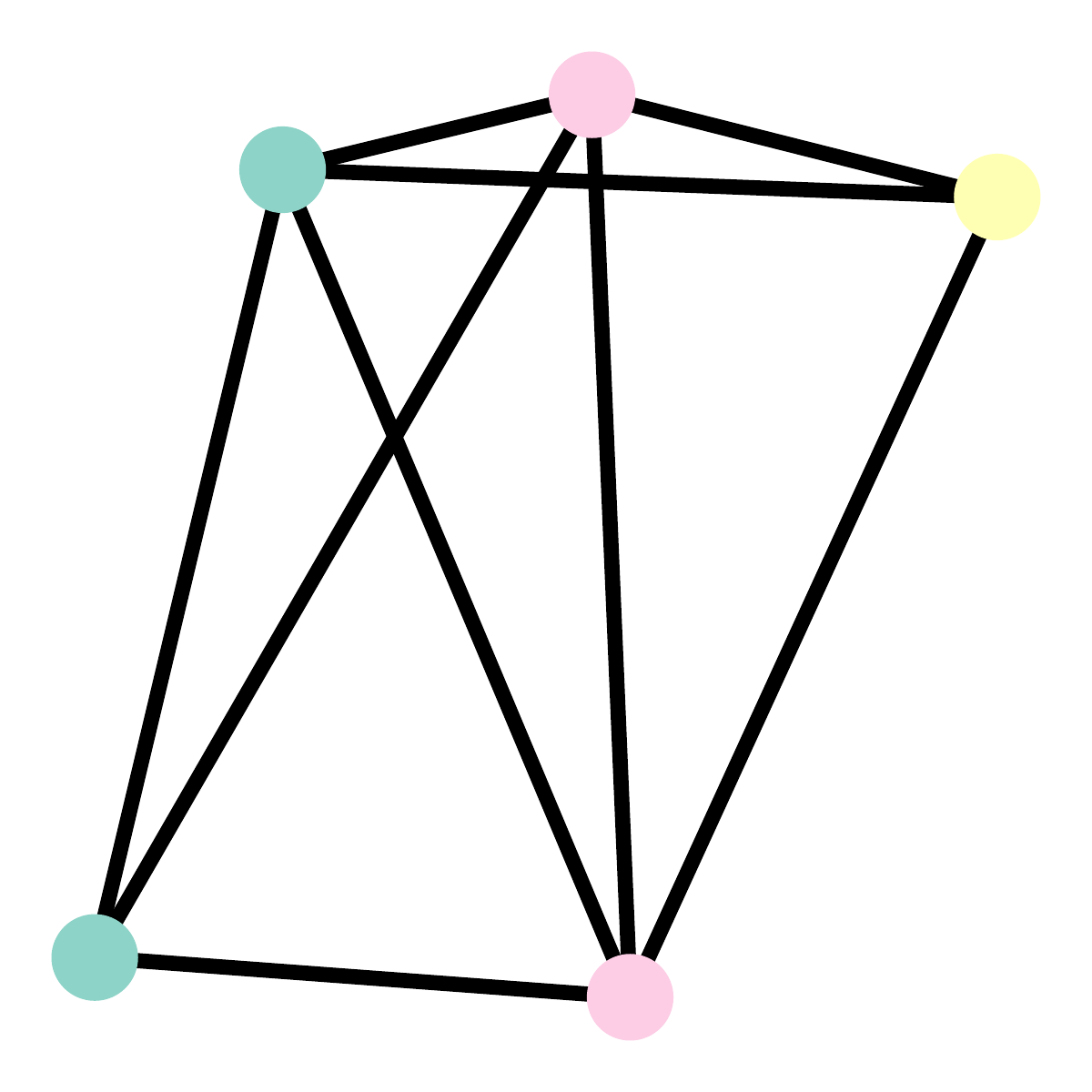}} & \adjustbox{valign=c}{\includegraphics[scale=0.105]{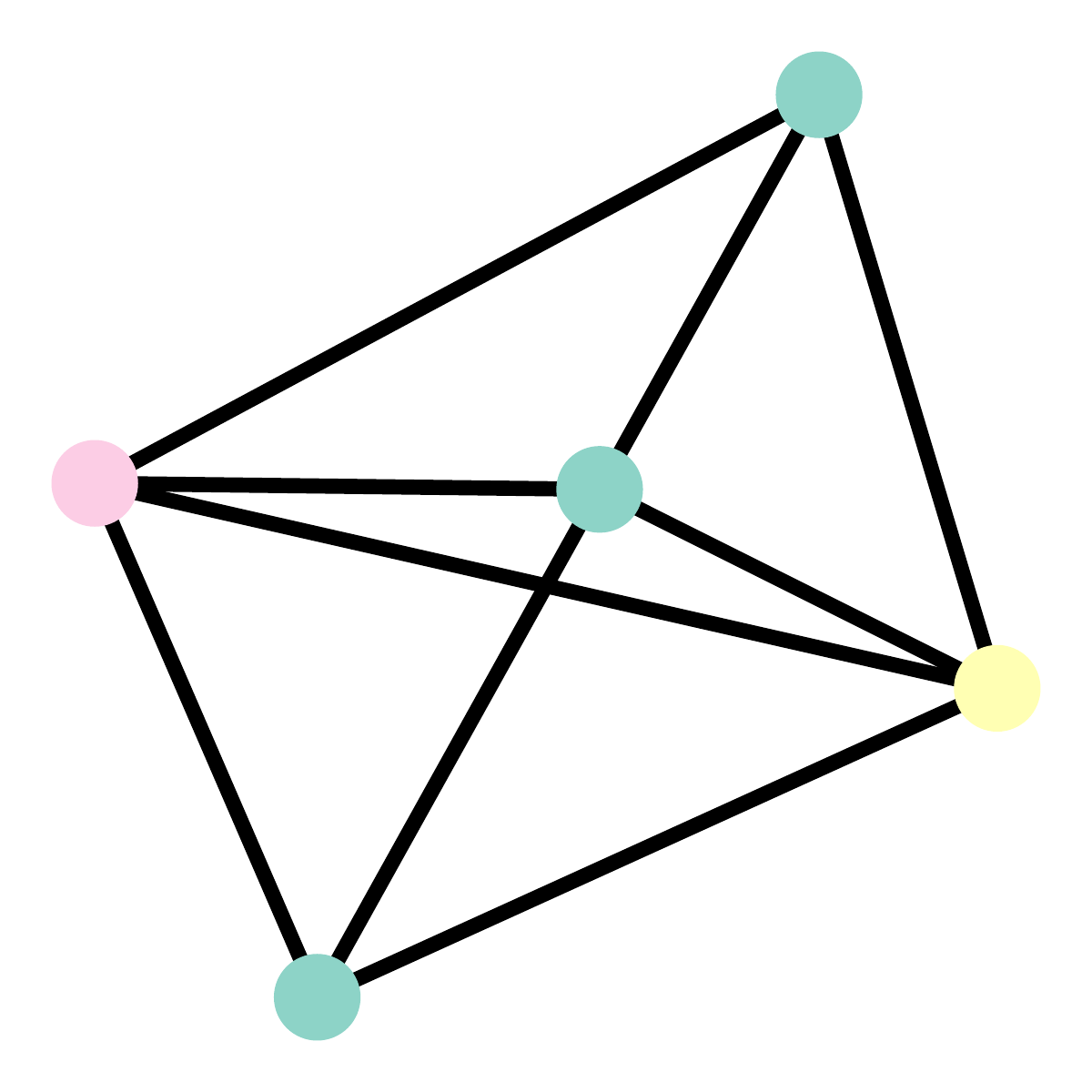}} & \adjustbox{valign=c}{\includegraphics[scale=0.105]{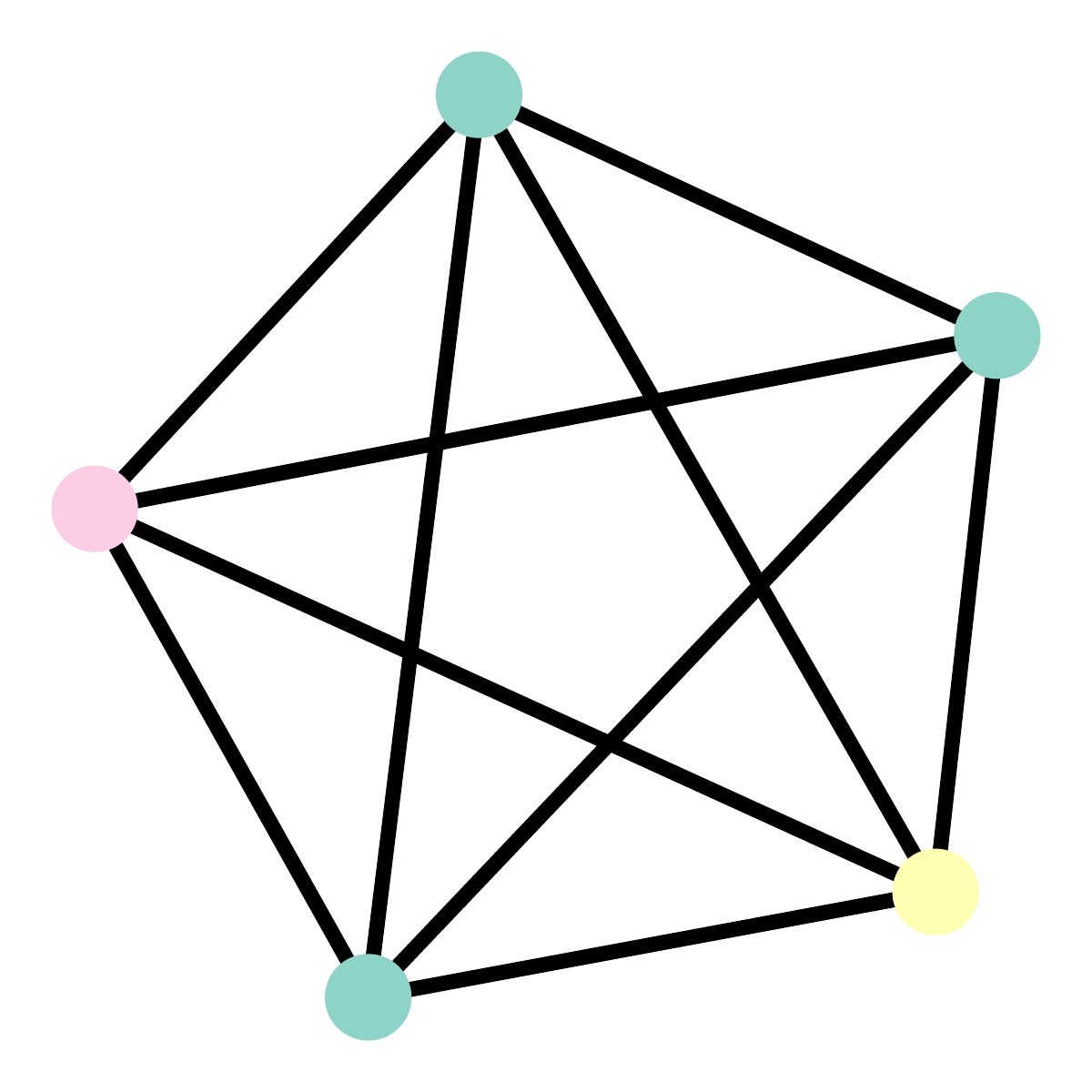}} & \adjustbox{valign=c}{\includegraphics[scale=0.105]{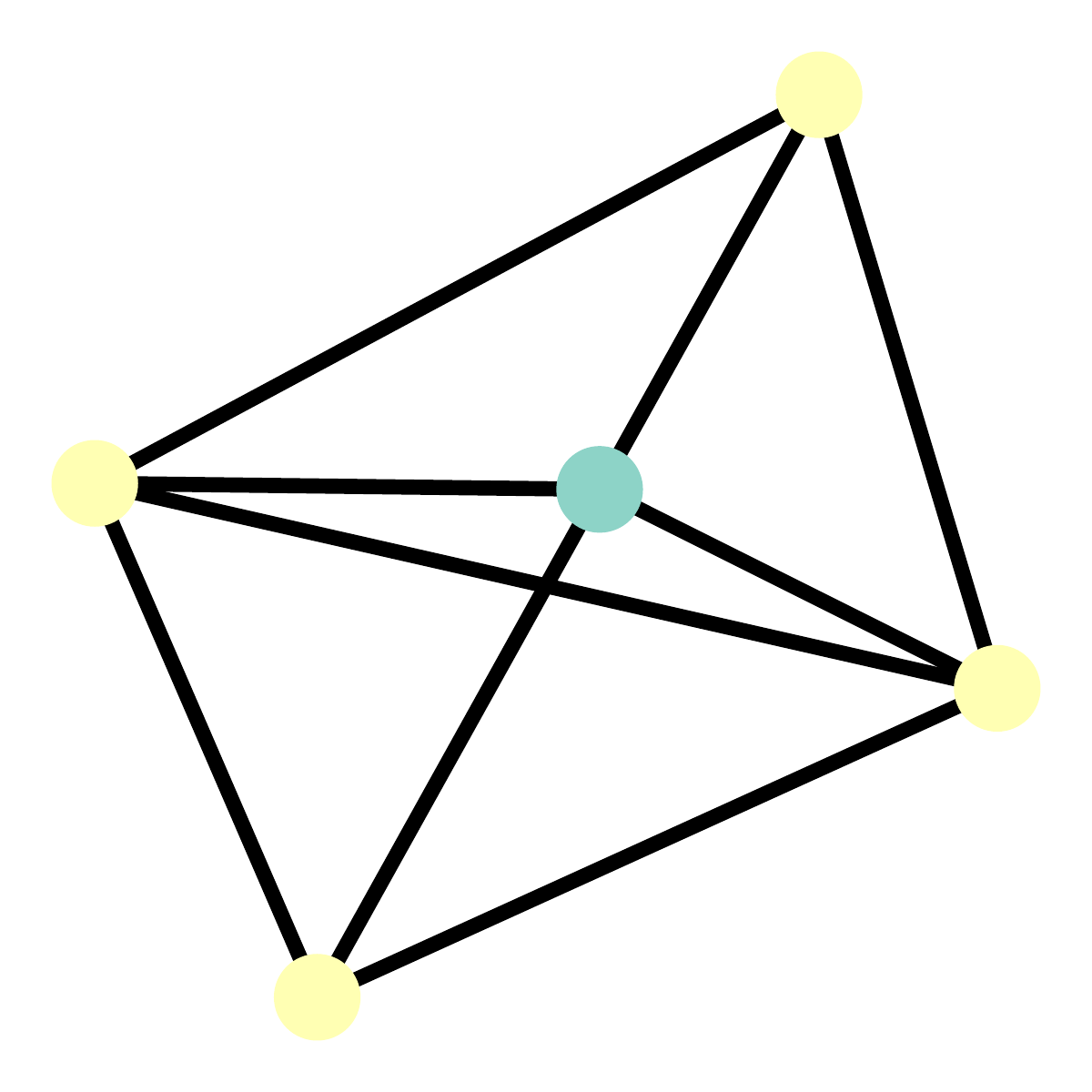}} & \adjustbox{valign=c}{\includegraphics[scale=0.105]{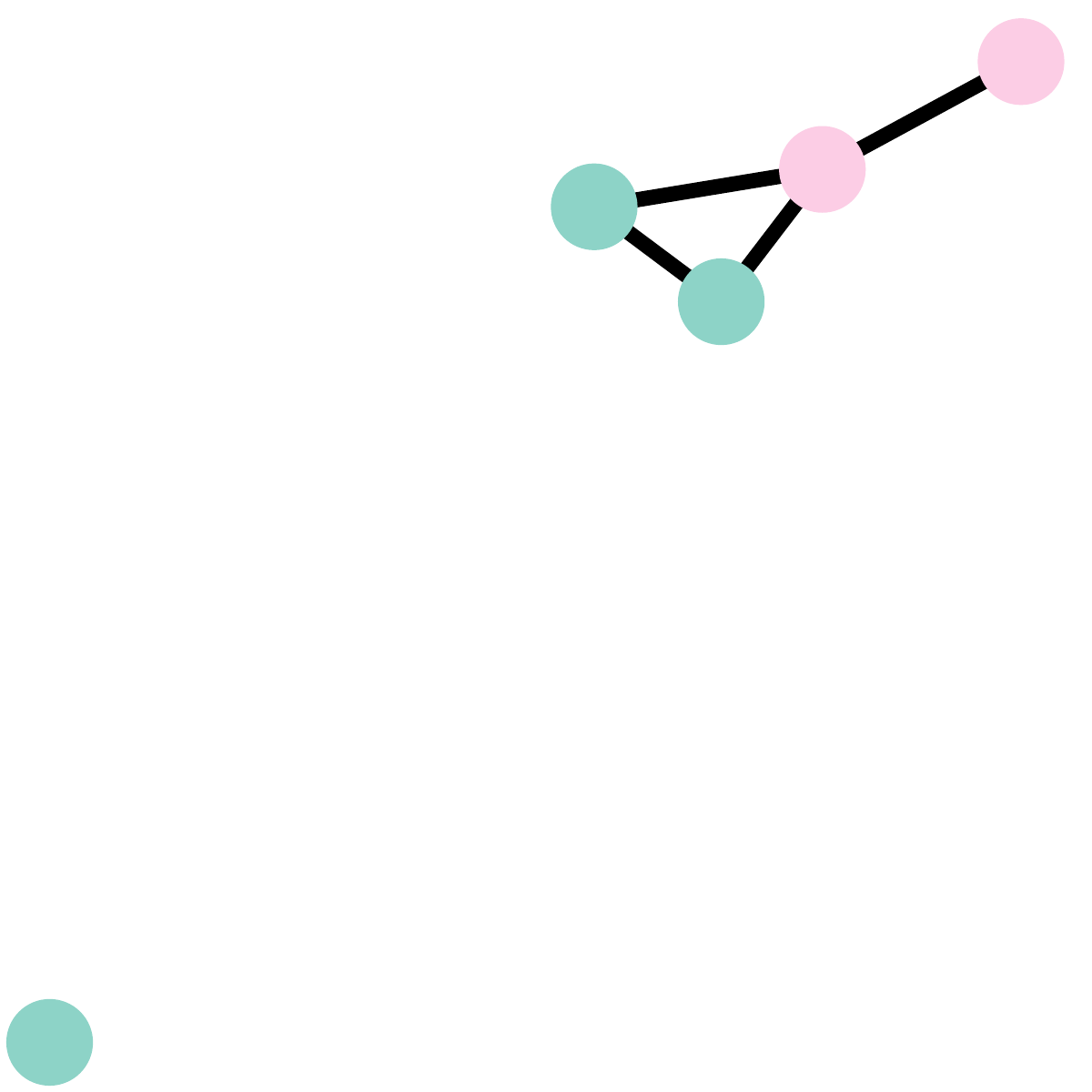}} & 
\adjustbox{valign=c}{\includegraphics[scale=0.105]{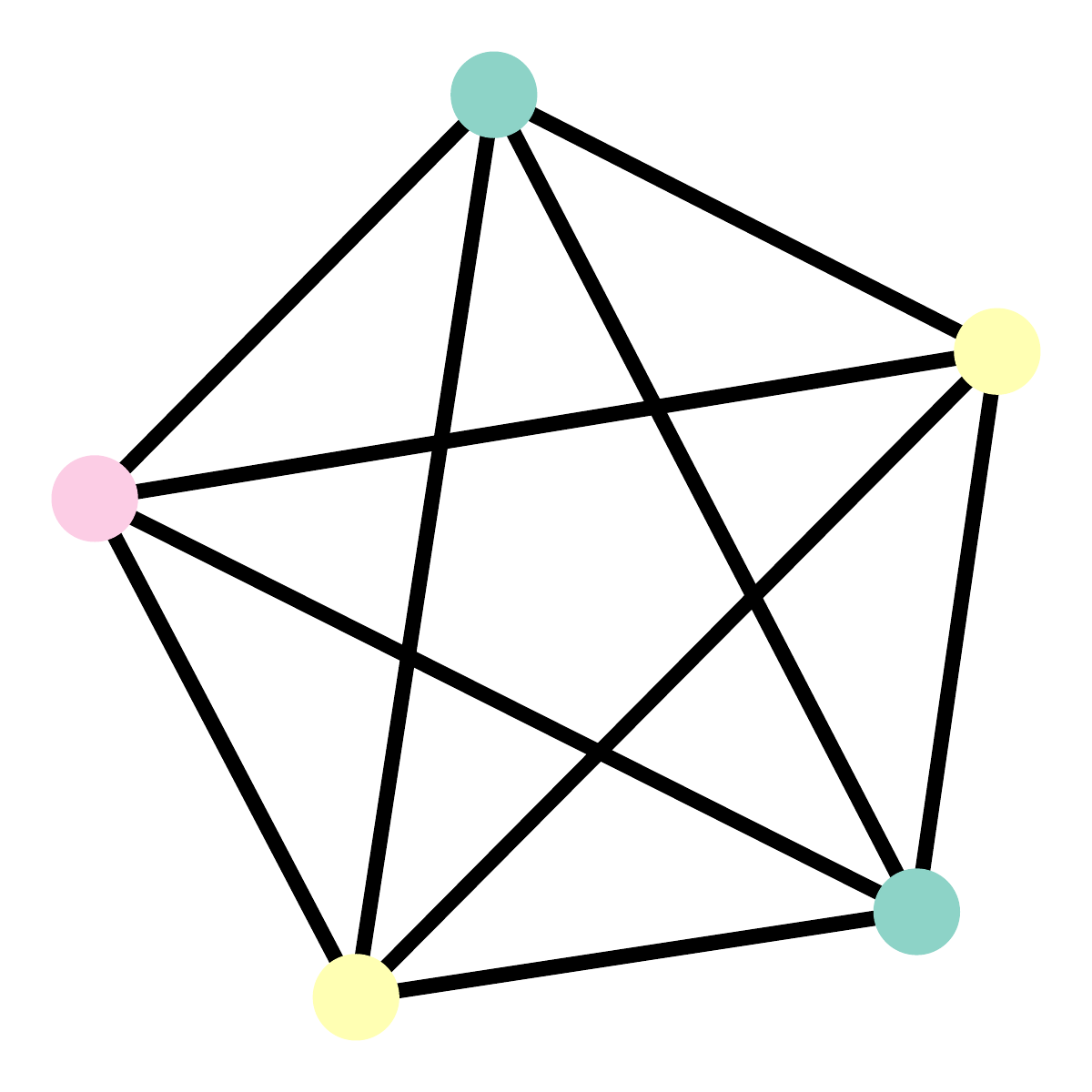}} &
\adjustbox{valign=c}{\includegraphics[scale=0.105]{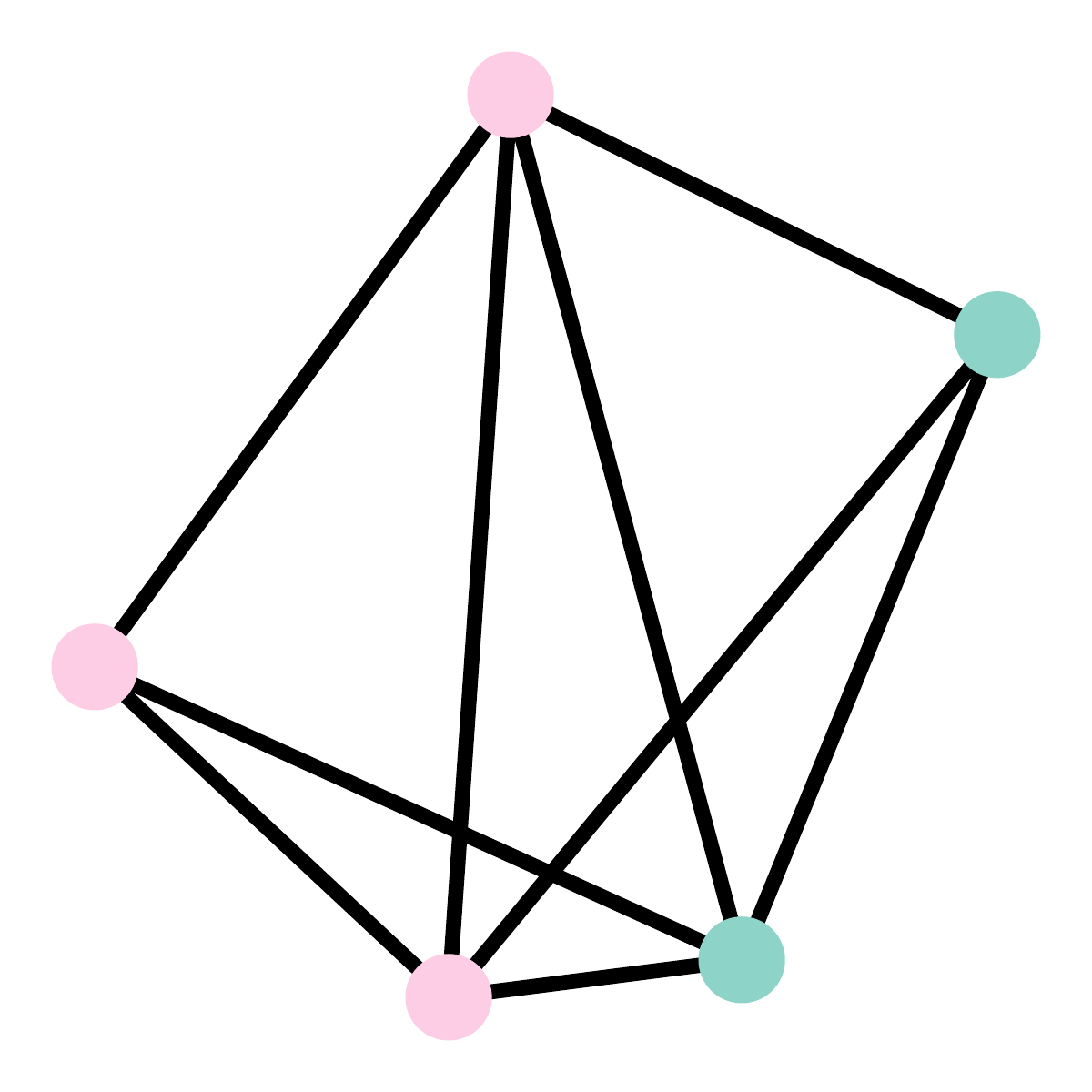}} &
\adjustbox{valign=c}{\includegraphics[scale=0.105]{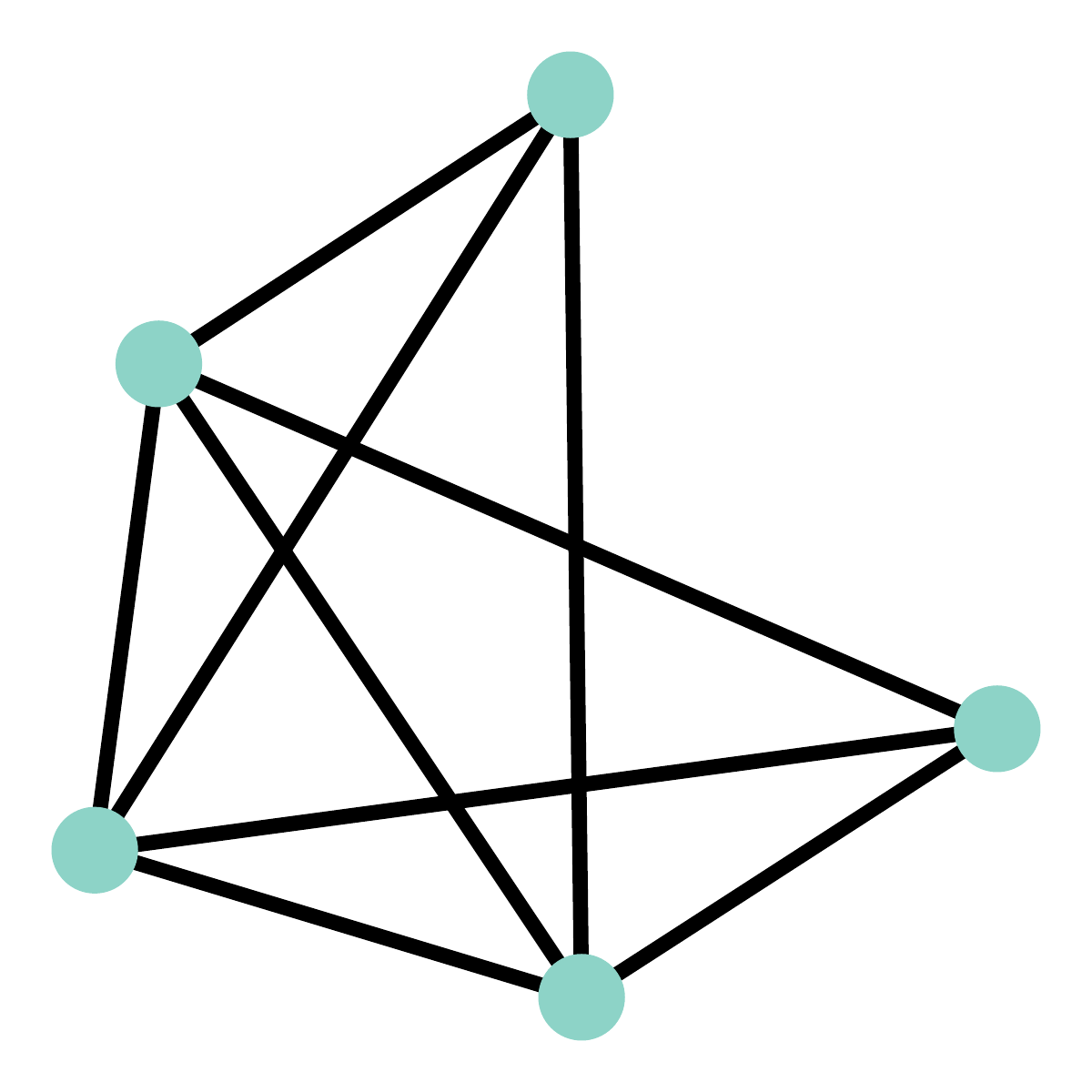}} \\ \midrule
\textbf{Node Feature} & \multicolumn{1}{c|}{\textbf{MSE/ACC(\%)}} & 0.3836/33.33 & 2.4573/33.33 & 2.1012/50.00 & 1.4808/50.00 & 0.2096/66.67 & \textbf{0.1744}/50.00 & 0.5025/16.67 & 1.7374/0.00 & 0.2007/\textbf{66.67} \\
\textbf{Adjacency Matrix} & \multicolumn{1}{c|}{\textbf{AUC/AP}} & 0.4514/0.6142 & 0.4444/0.7148 & 0.7604/0.8430 & 0.6354/0.6981/63.89 & 0.4583/0.6488 & 0.5000/0.6667 & 0.8264/0.9009 & 0.6250/0.7302 & \textbf{0.9306/0.9677} \\
\textbf{\begin{tabular}[c]{@{}c@{}}Ground-truth/\\ Reconstructed graph\end{tabular}} & \multicolumn{1}{c|}{\adjustbox{valign=c}{\includegraphics[scale=0.105]{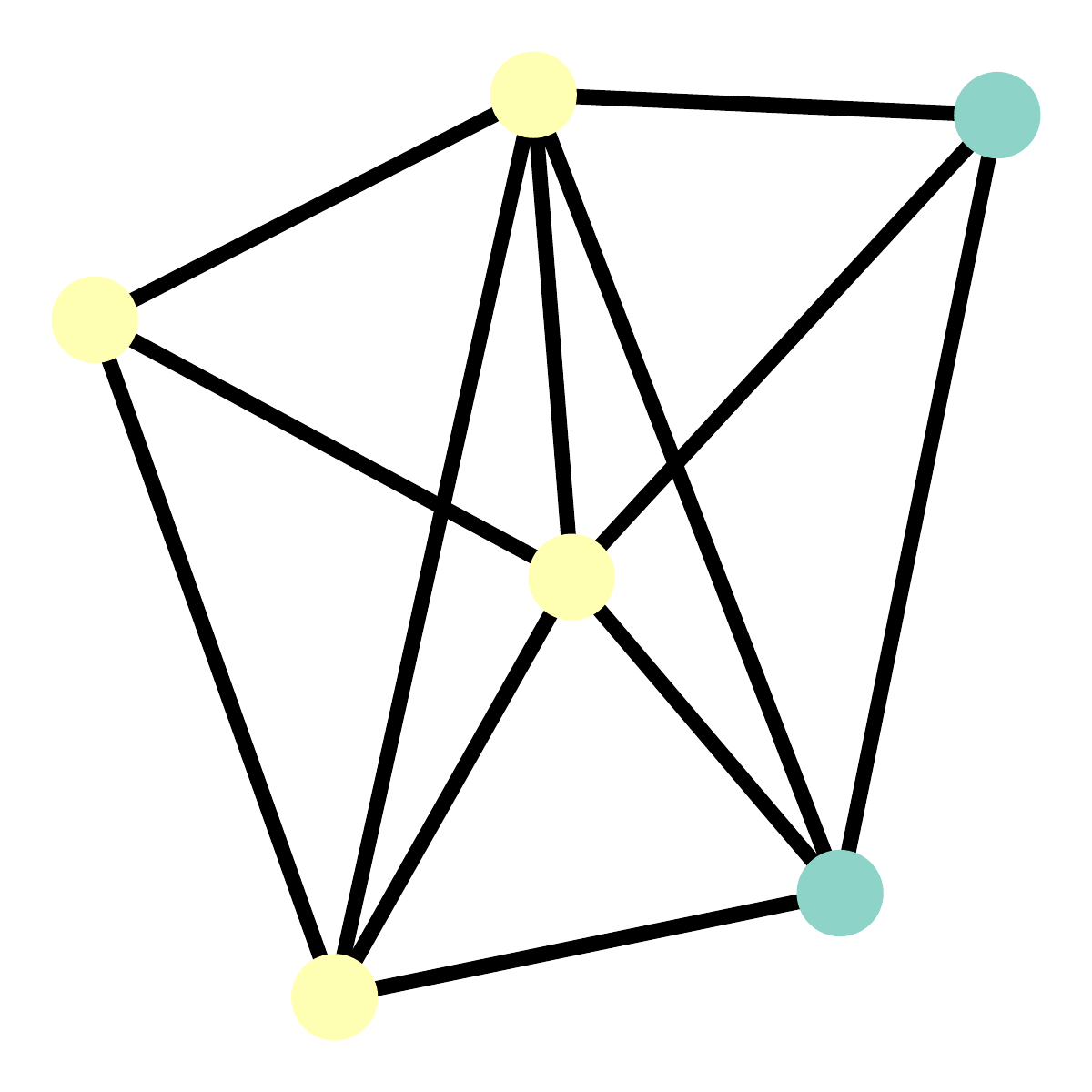}}} & \adjustbox{valign=c}{\includegraphics[scale=0.105]{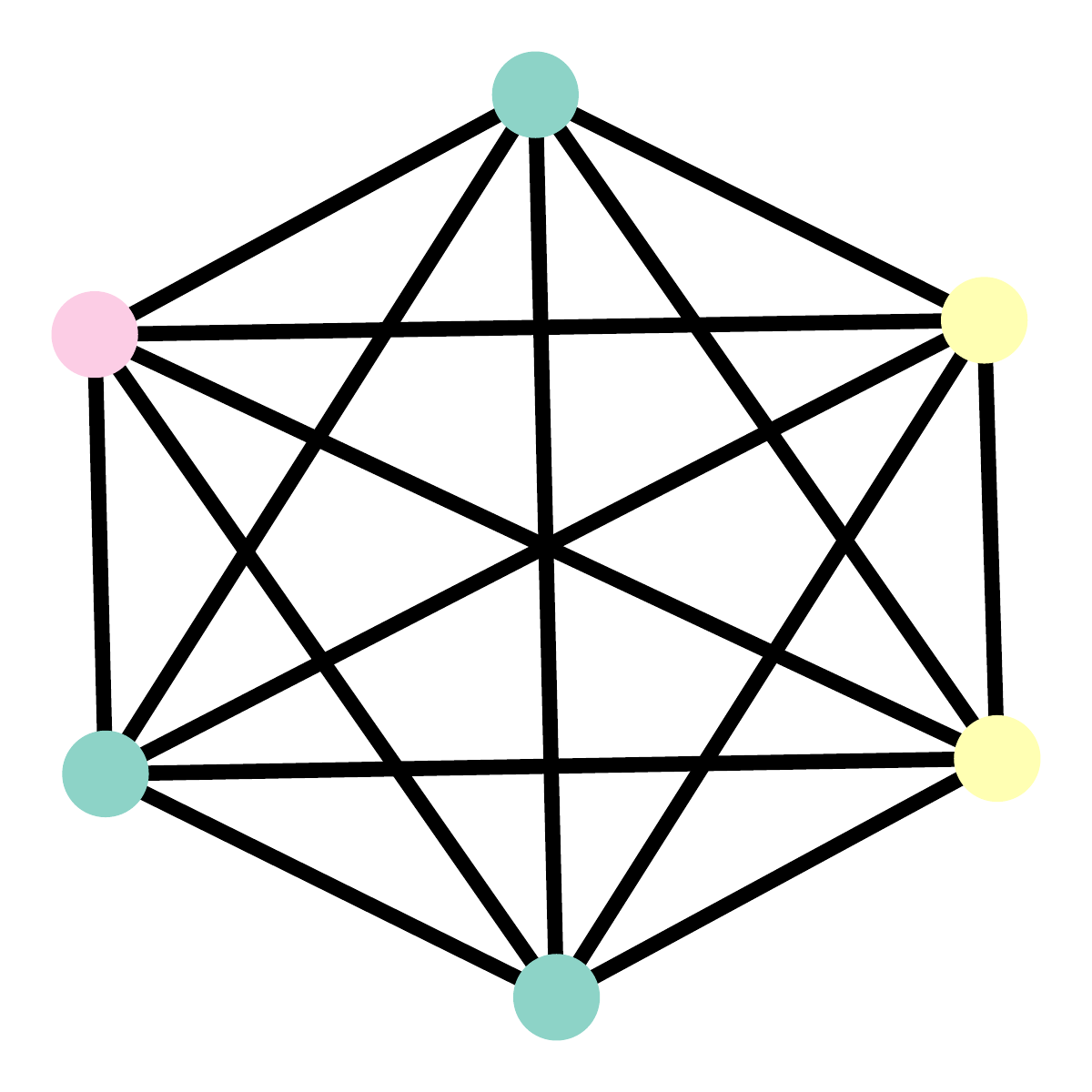}} & \adjustbox{valign=c}{\includegraphics[scale=0.105]{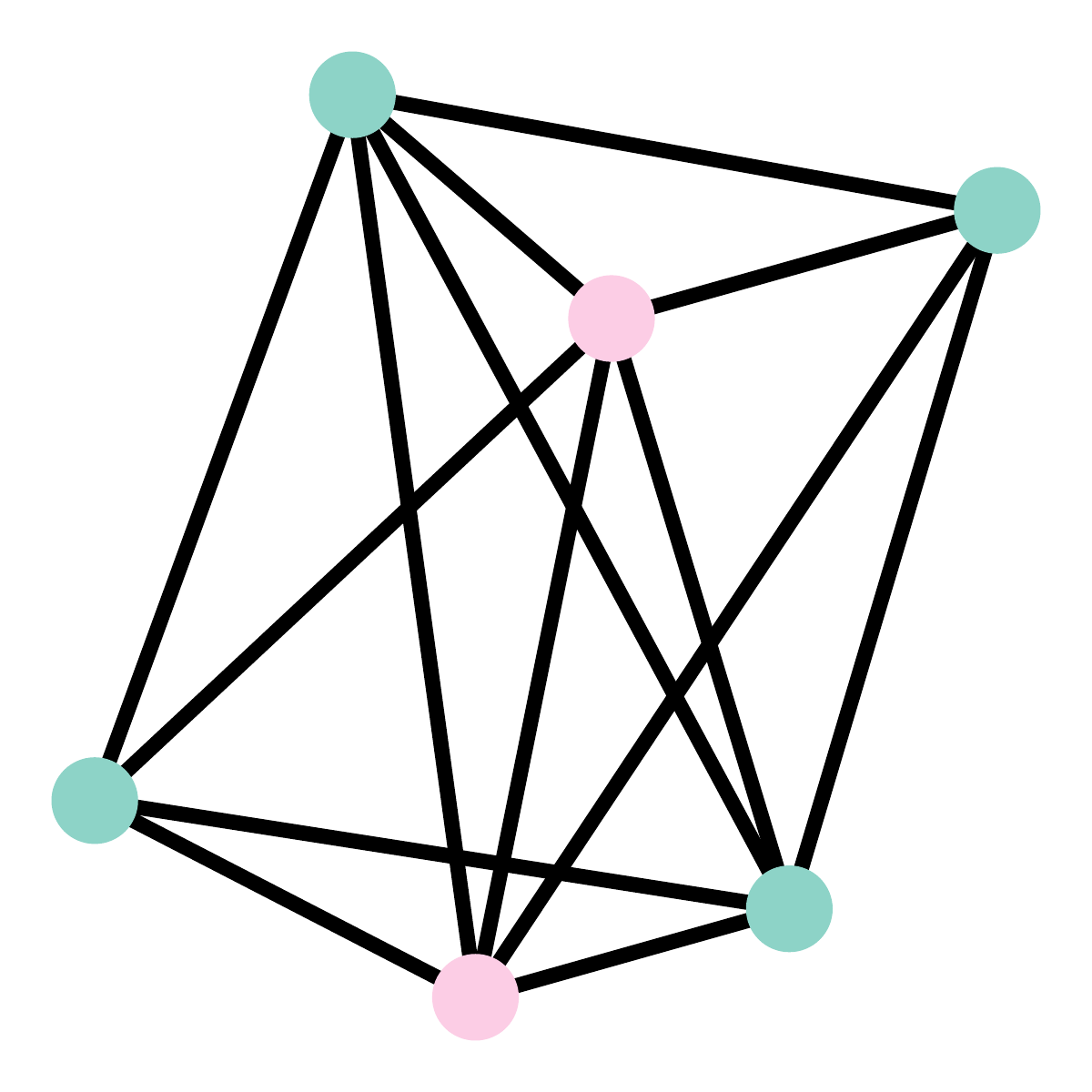}} & \adjustbox{valign=c}{\includegraphics[scale=0.105]{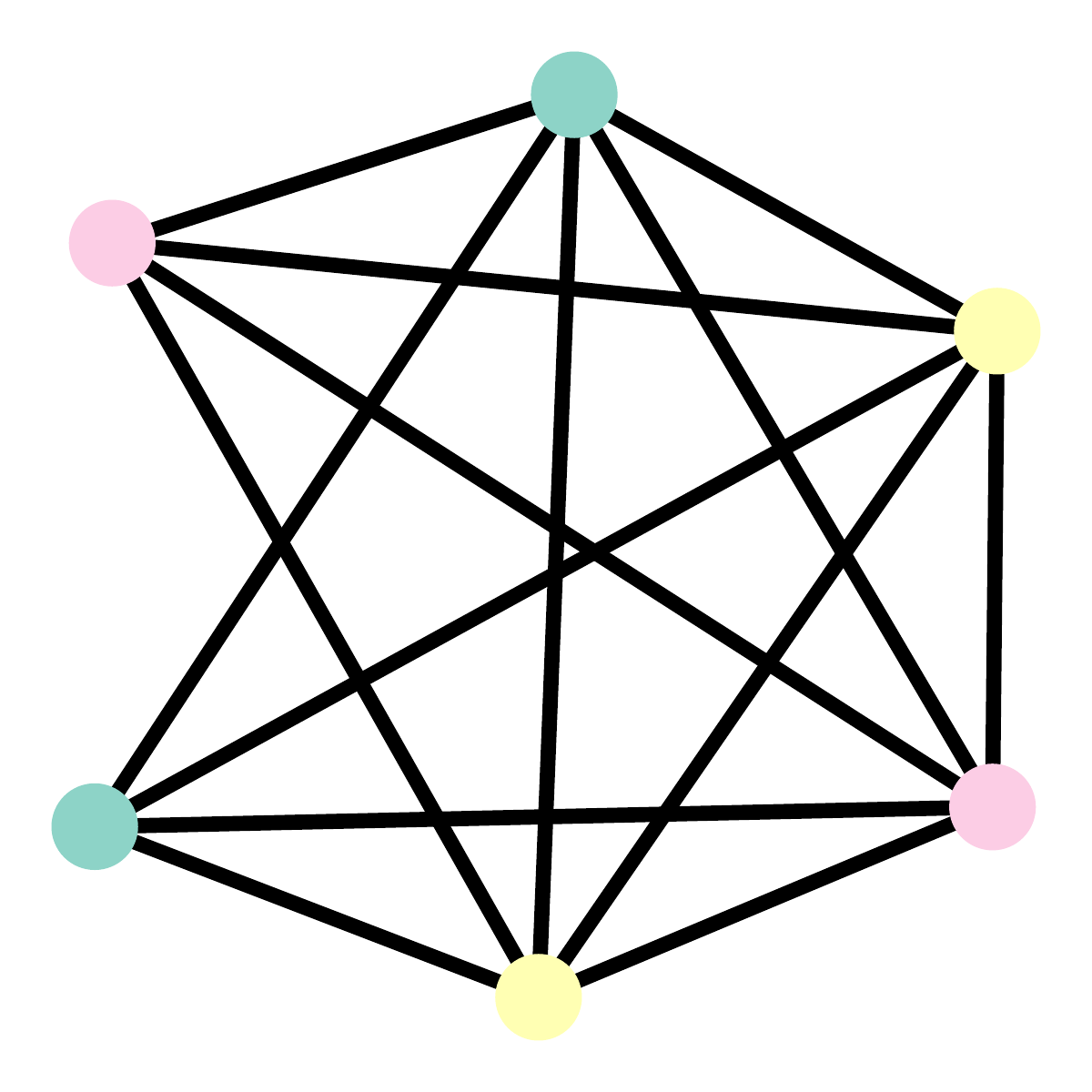}} & \adjustbox{valign=c}{\includegraphics[scale=0.105]{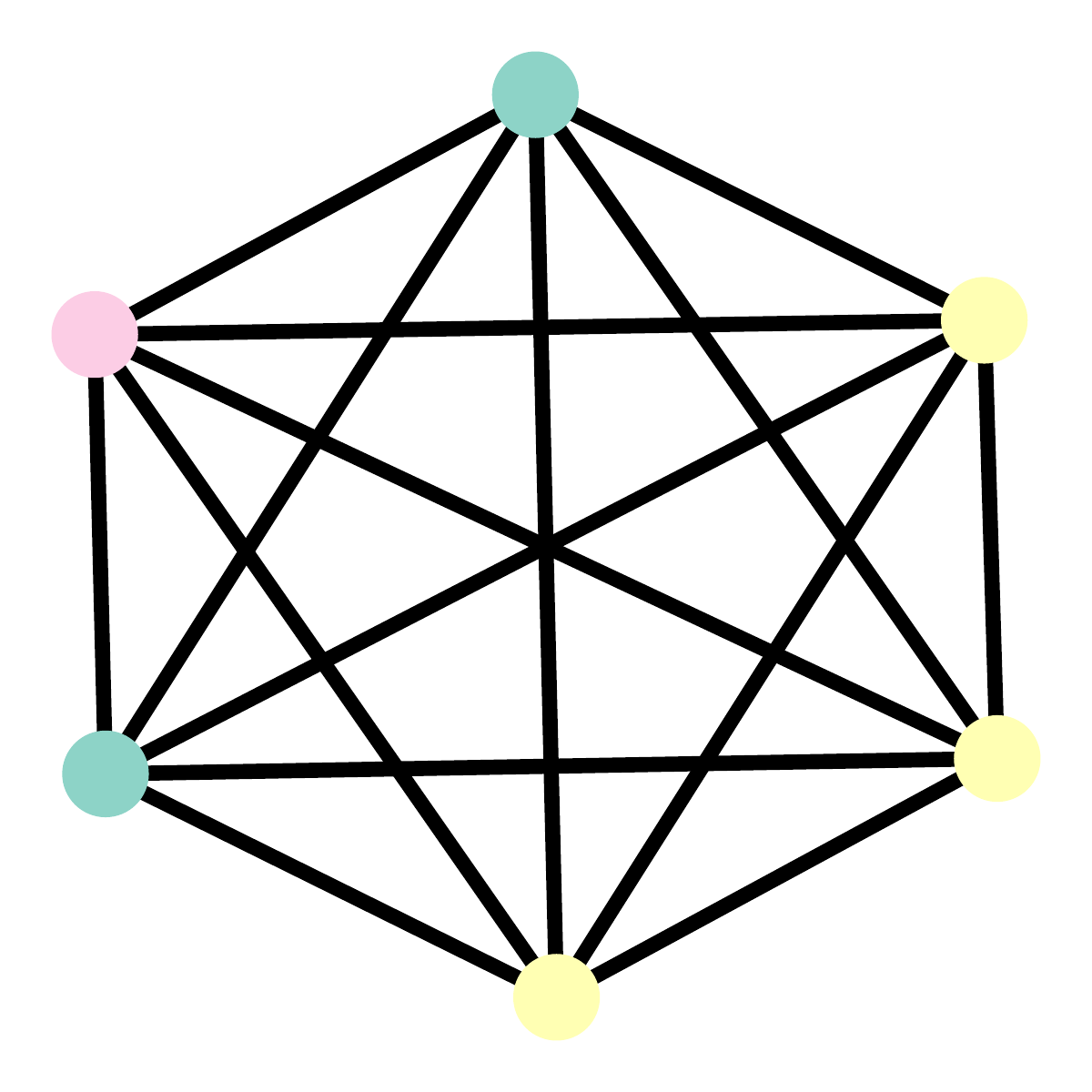}} & \adjustbox{valign=c}{\includegraphics[scale=0.105]{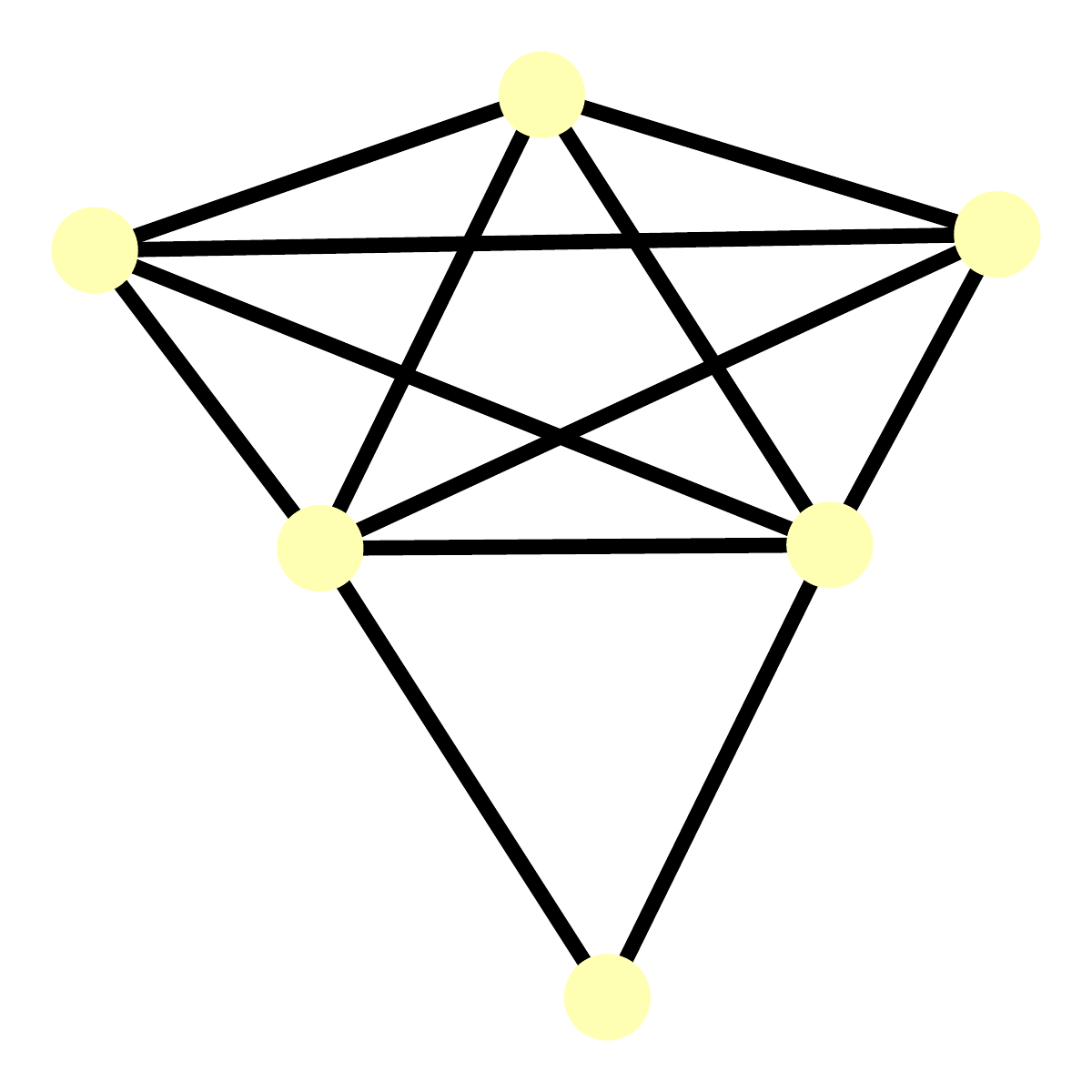}} & \adjustbox{valign=c}{\includegraphics[scale=0.105]{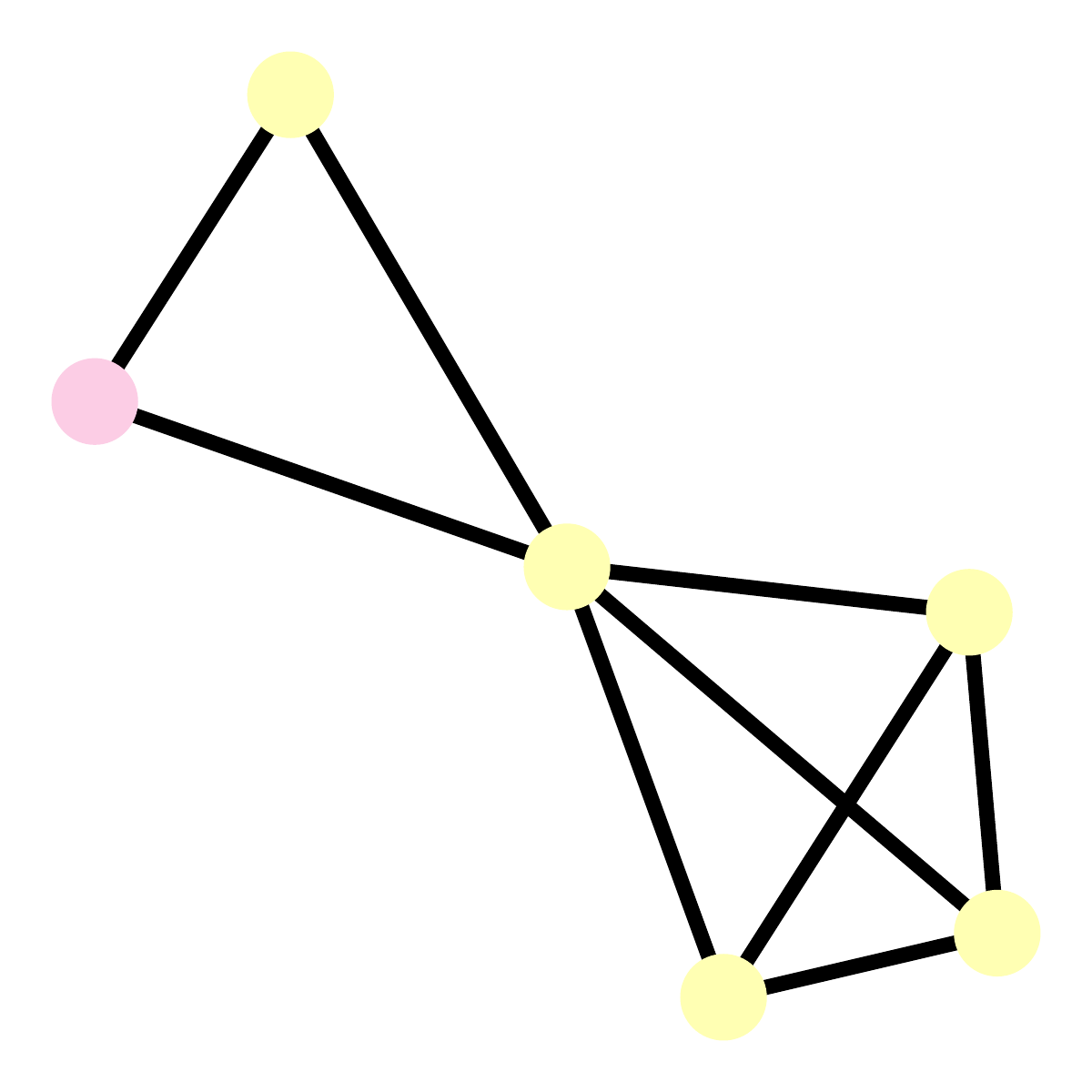}} &
\adjustbox{valign=c}{\includegraphics[scale=0.105]{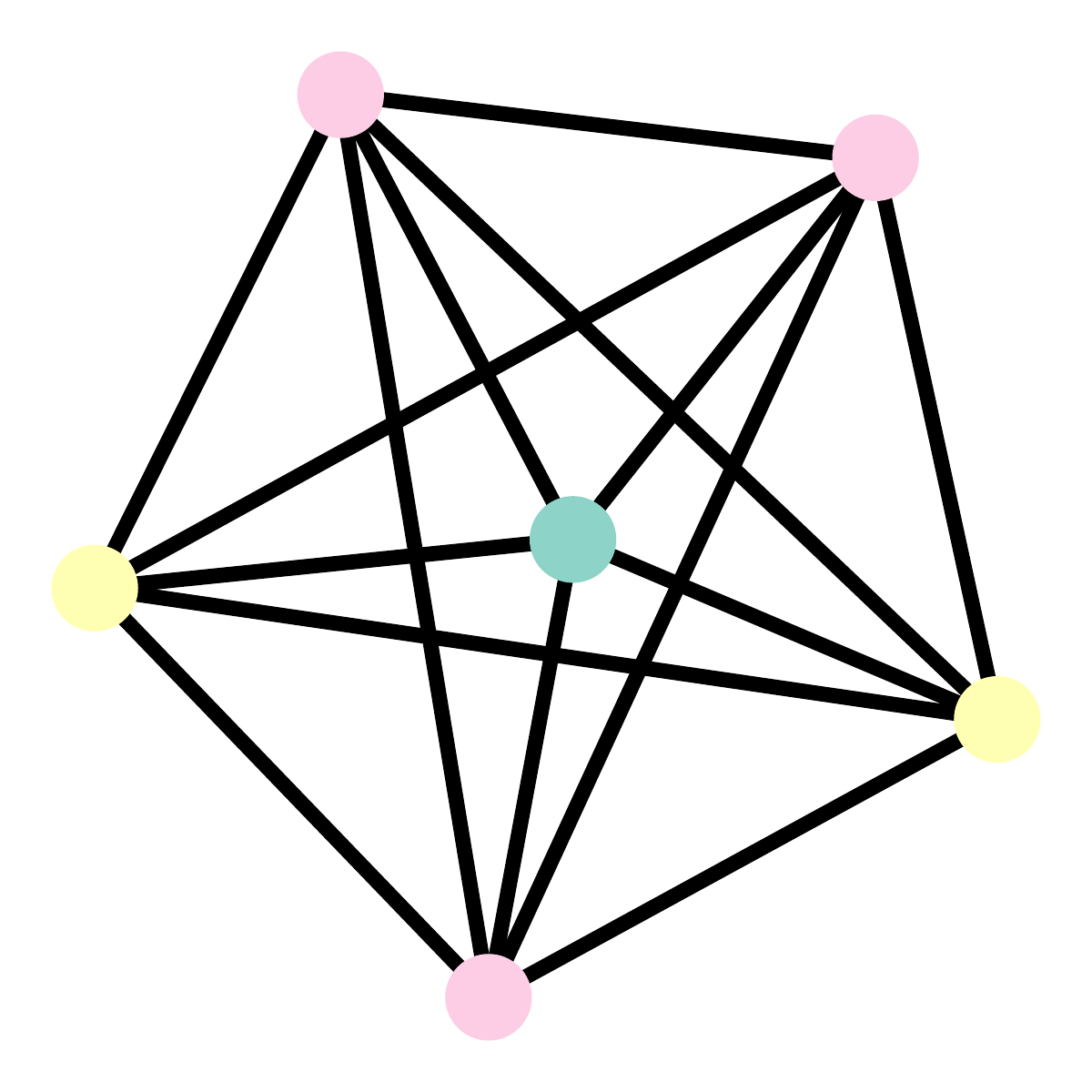}} &
\adjustbox{valign=c}{\includegraphics[scale=0.105]{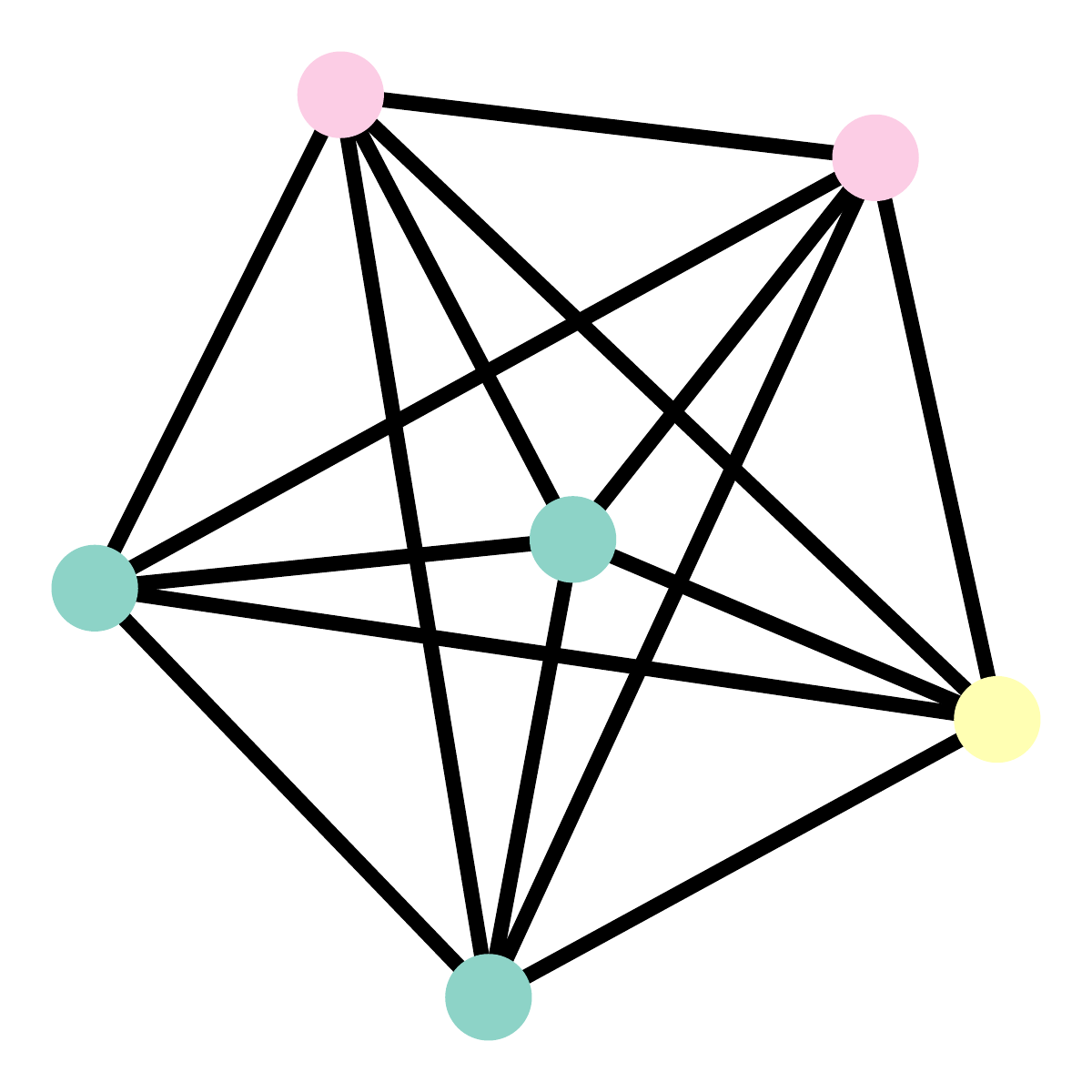}} &
\adjustbox{valign=c}{\includegraphics[scale=0.105]{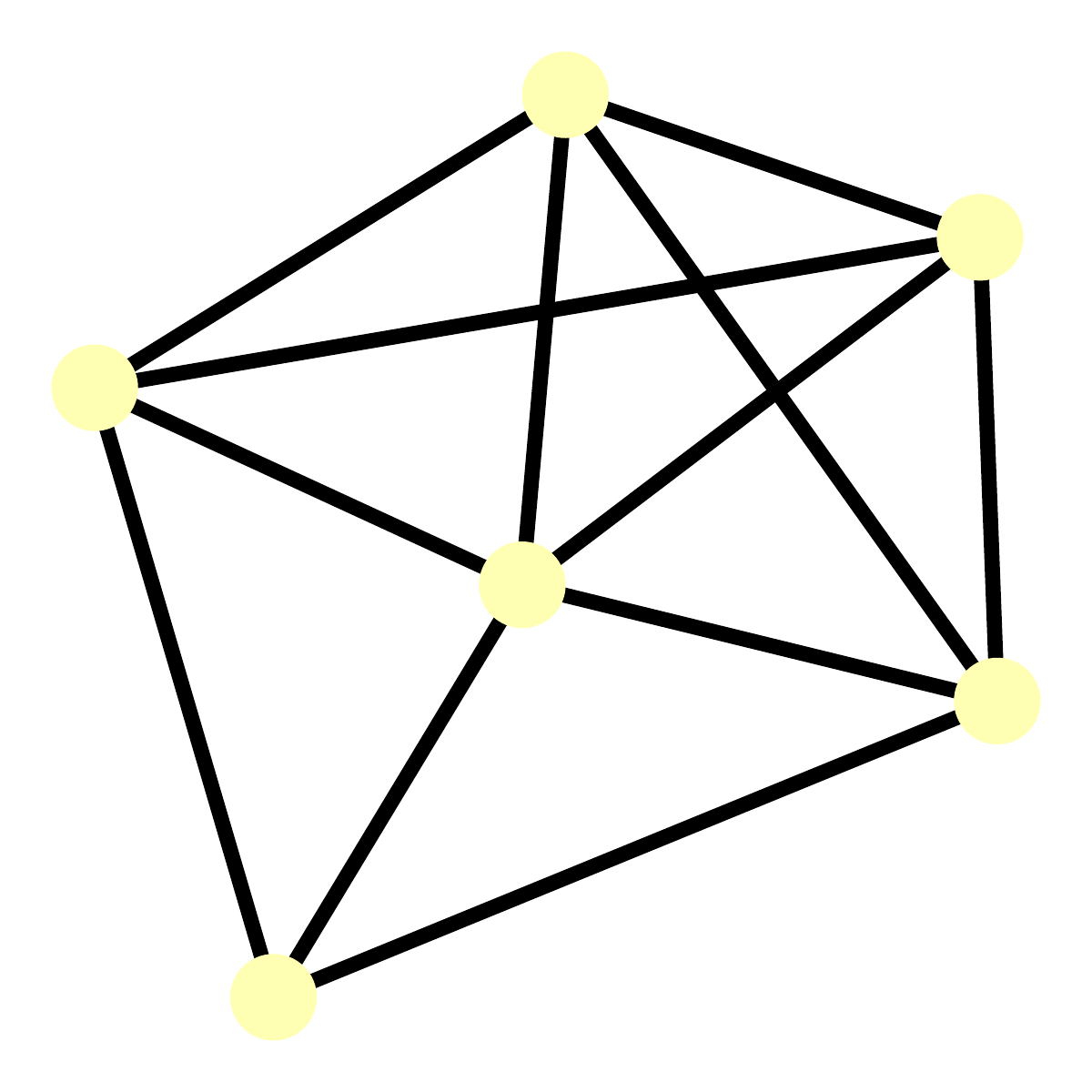}} \\ \midrule
\textbf{Node Feature} & \multicolumn{1}{c|}{\textbf{MSE/ACC(\%)}} & 0.3224/26.32 & 0.9537/21.05 & 1.2781/15.79 & 0.7949/31.58 & 0.0994/57.89 & 0.1072/\textbf{84.21} & 0.3175/26.32 & 0.8379/21.05 & \textbf{0.0745/84.21} \\
\textbf{Adjacency Matrix} & \multicolumn{1}{c|}{\textbf{AUC/AP}} & 0.4494/0.1083 & 0.5209/0.1221 & 0.5515/0.1661 & 0.5374/0.1504 & 0.5110/0.1243 & 0.4209/0.1101/49.86 & 0.5528/0.1396 & 0.5119/0.1245 & \textbf{0.9865/0.9203} \\
\textbf{\begin{tabular}[c]{@{}c@{}}Ground-truth/\\ Reconstructed graph\end{tabular}} & \multicolumn{1}{c|}{\adjustbox{valign=c}{\includegraphics[scale=0.105]{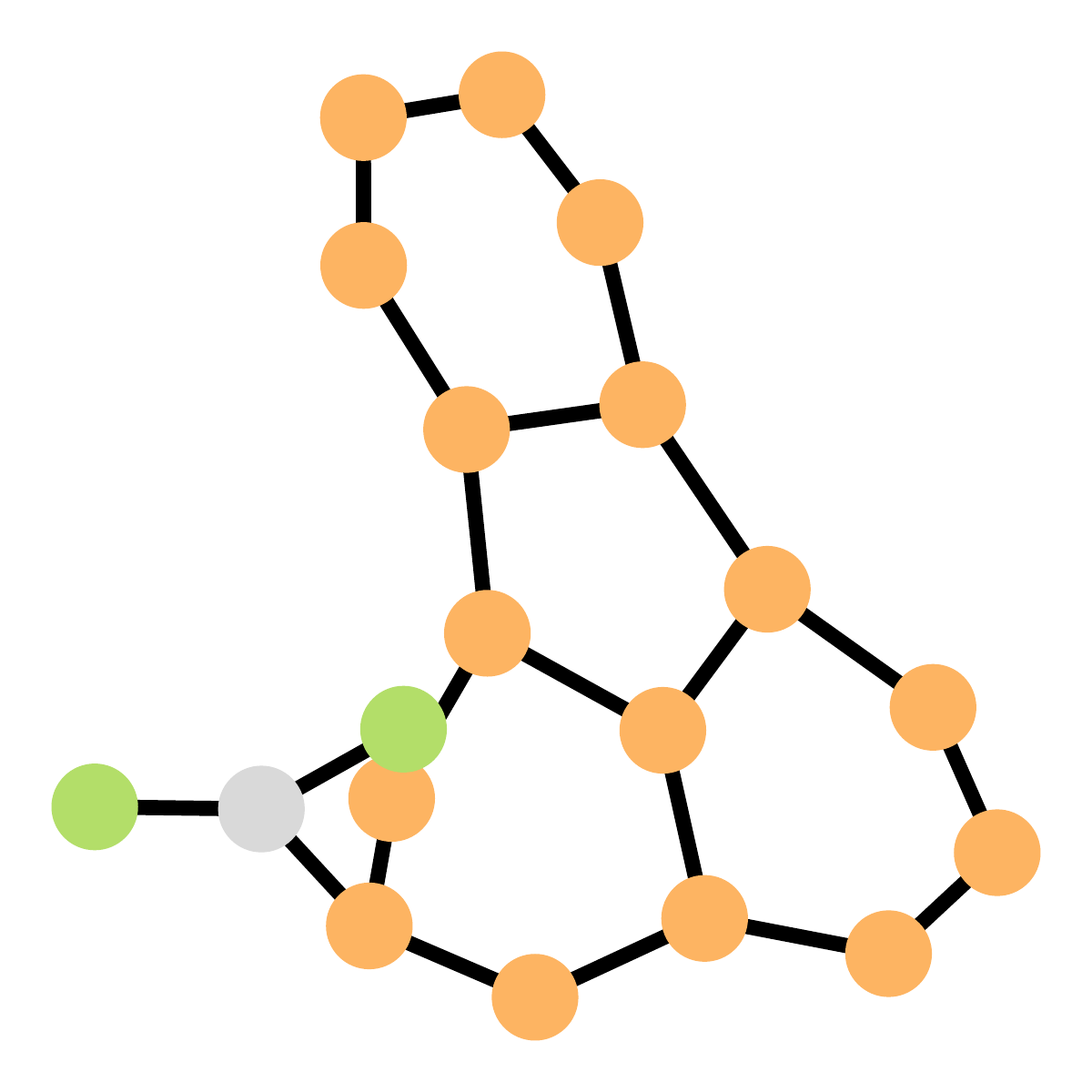}}} & \adjustbox{valign=c}{\includegraphics[scale=0.105]{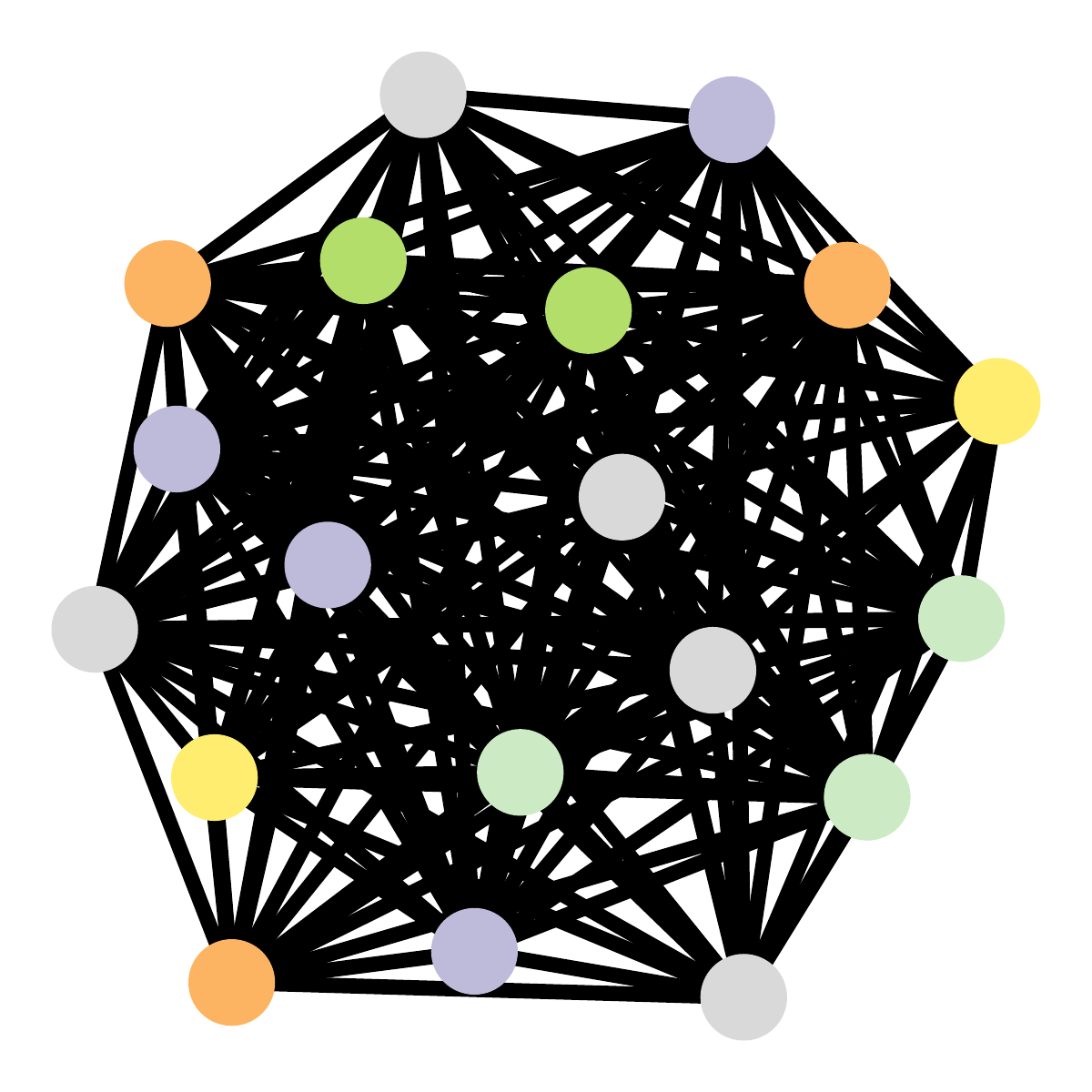}} & \adjustbox{valign=c}{\includegraphics[scale=0.105]{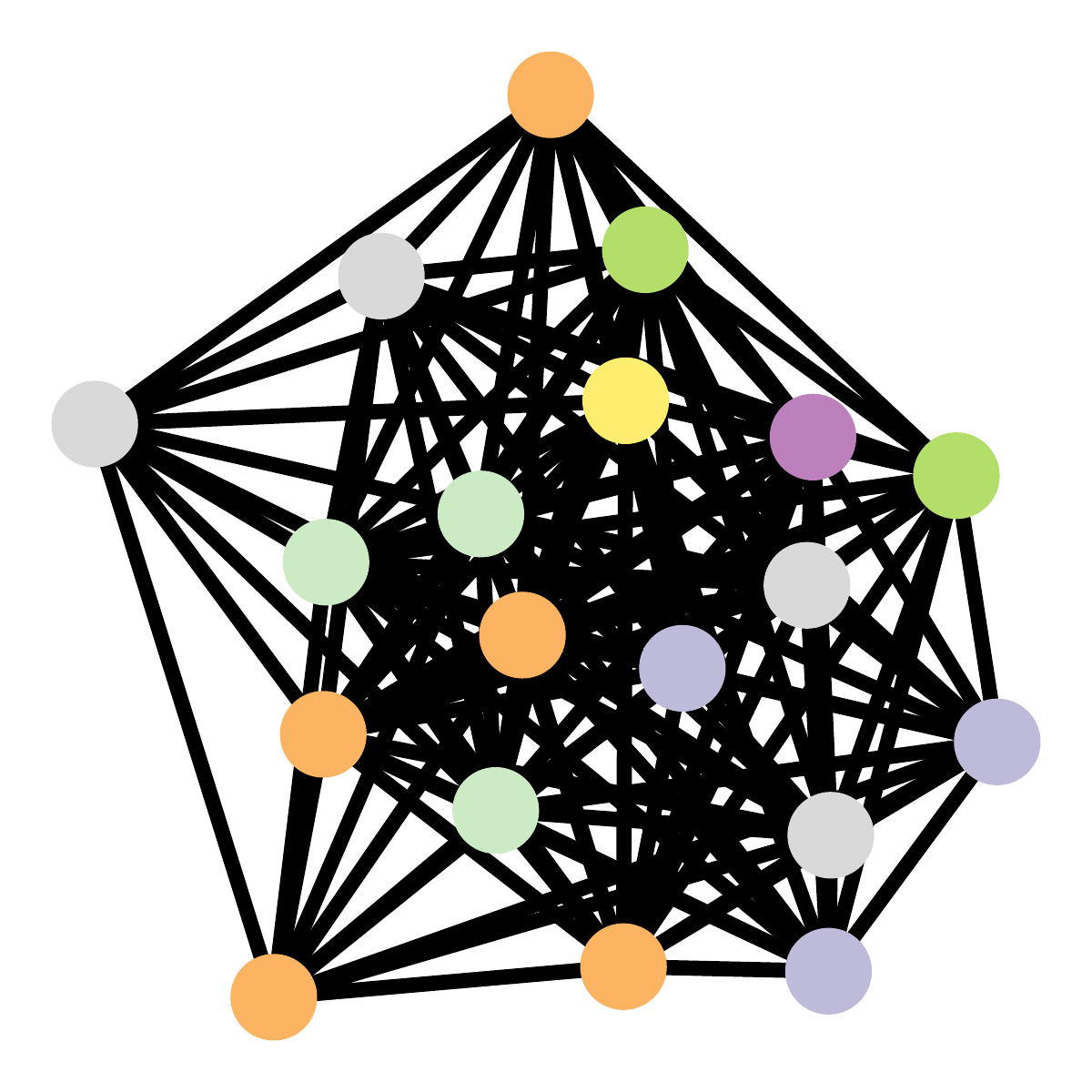}} & \adjustbox{valign=c}{\includegraphics[scale=0.105]{./figures/visualization/MUTAG/30_idlg.pdf}} & \adjustbox{valign=c}{\includegraphics[scale=0.105]{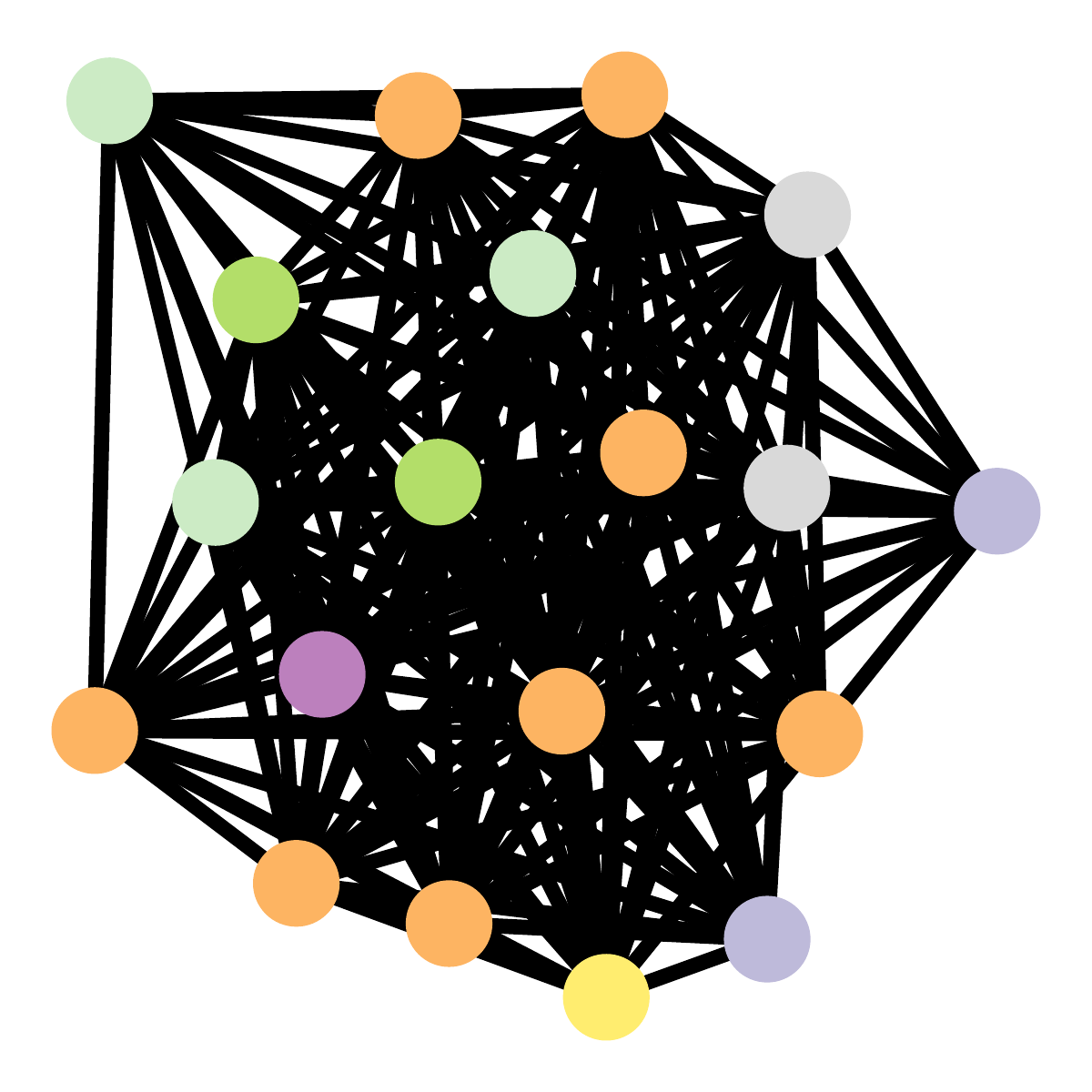}} & \adjustbox{valign=c}{\includegraphics[scale=0.105]{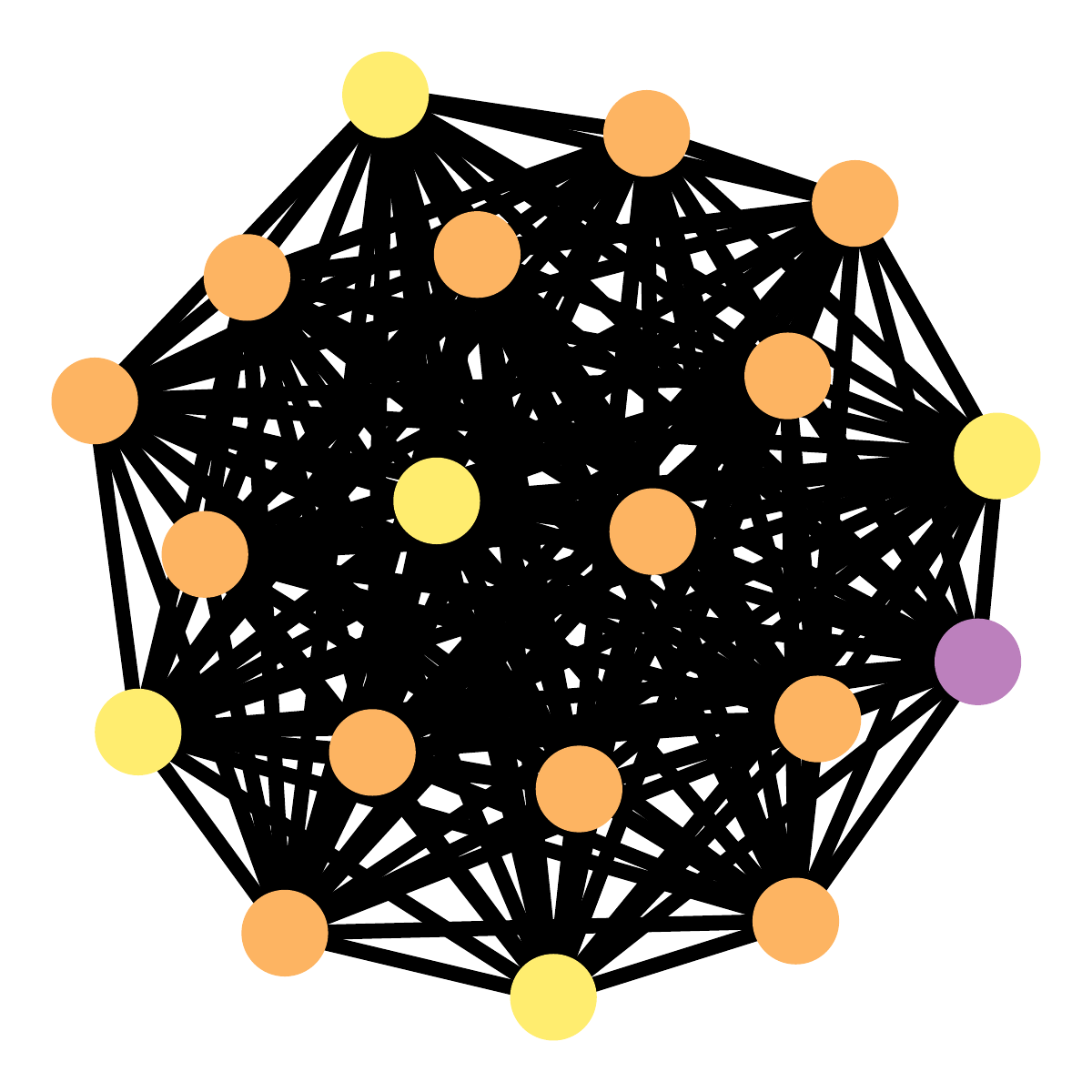}} & \adjustbox{valign=c}{\includegraphics[scale=0.105]{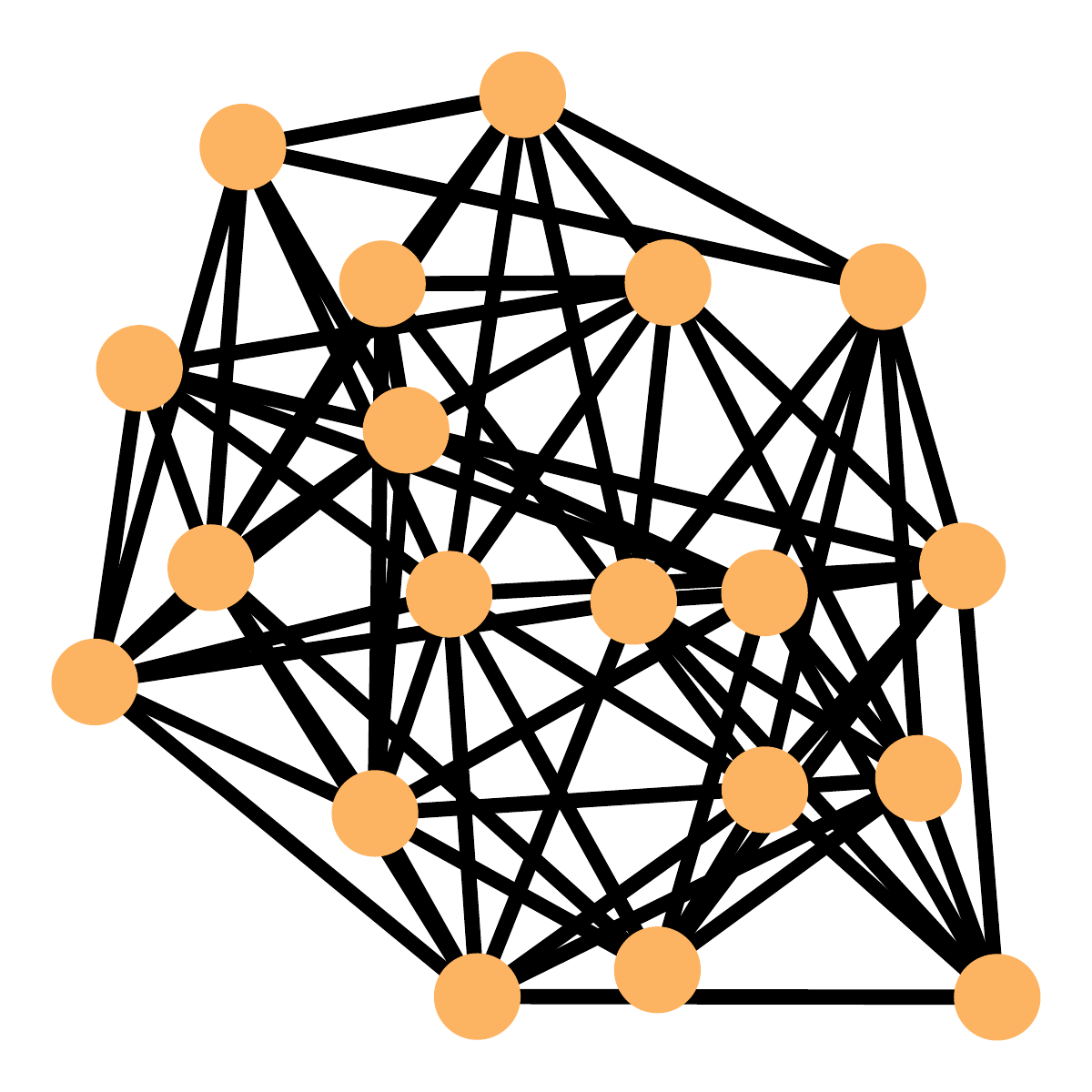}} & 
\adjustbox{valign=c}{\includegraphics[scale=0.105]{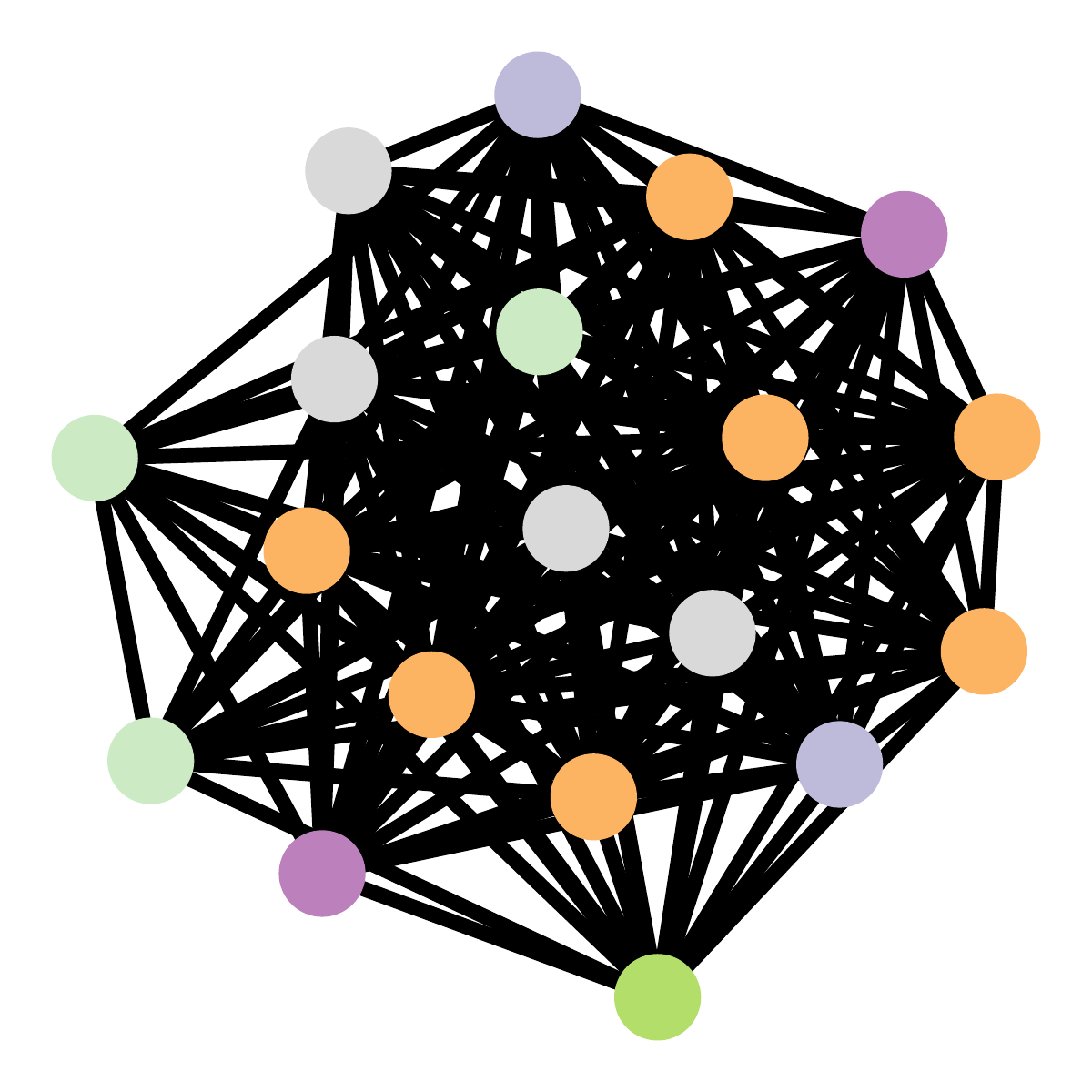}} &
\adjustbox{valign=c}{\includegraphics[scale=0.105]{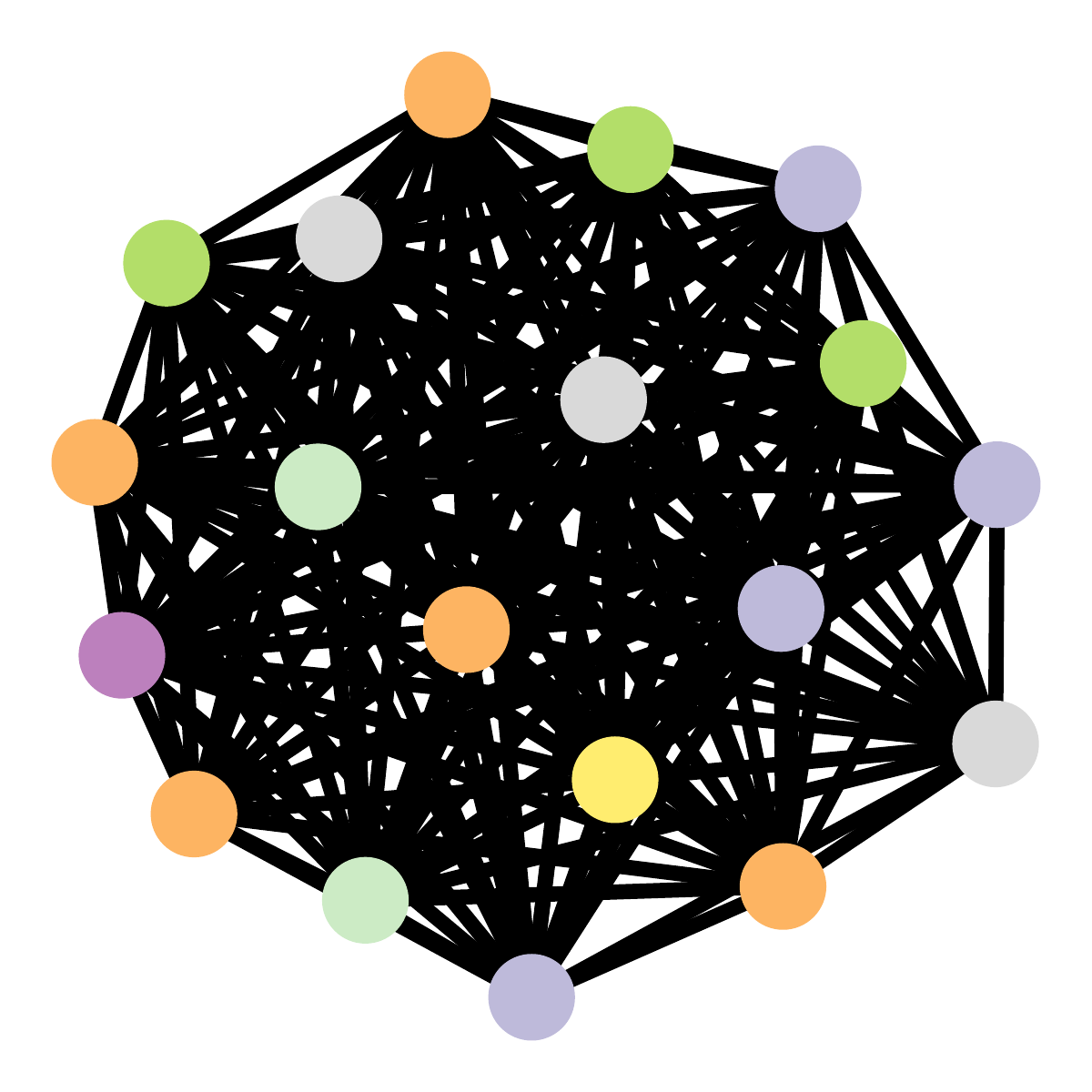}} &
\adjustbox{valign=c}{\includegraphics[scale=0.105]{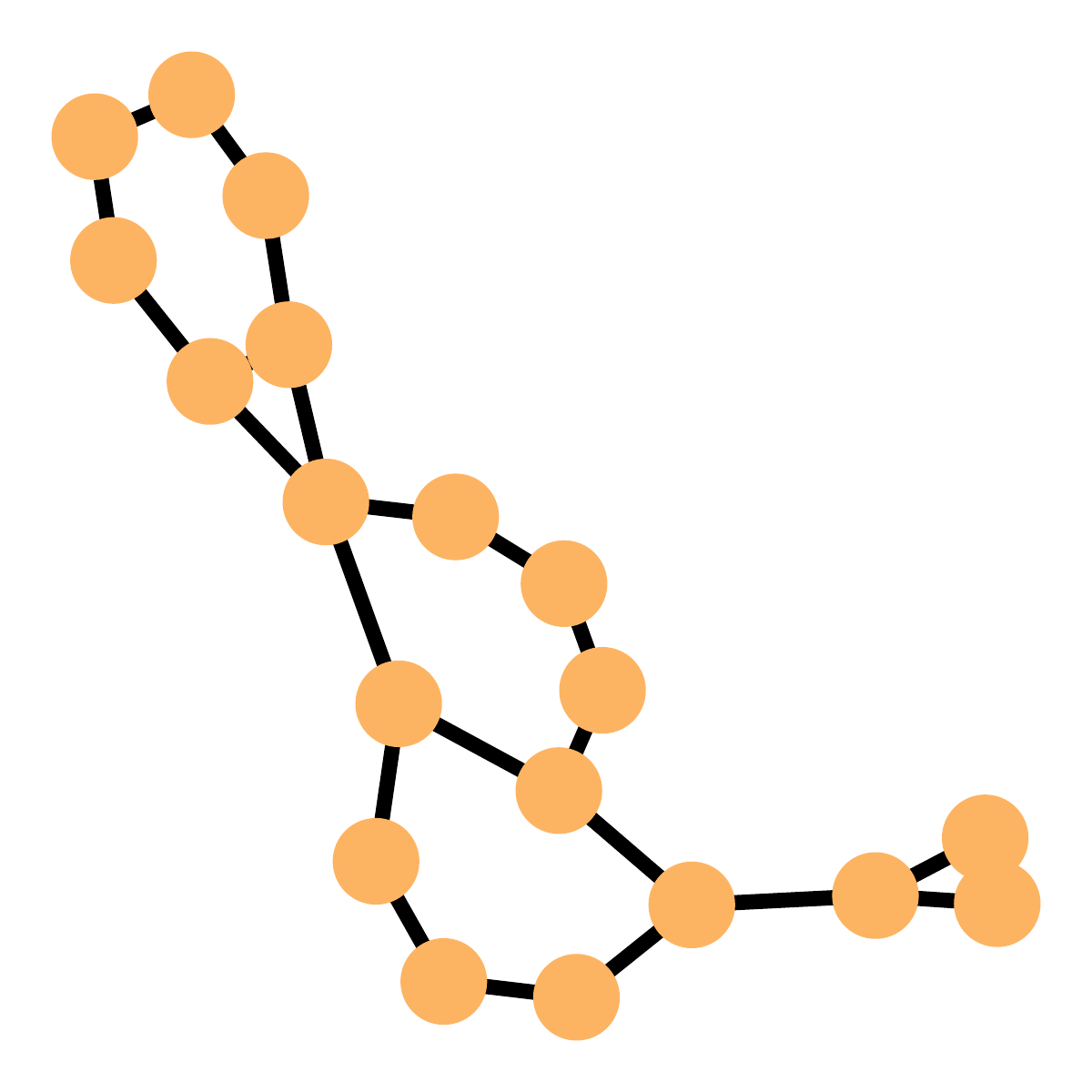}} \\ \midrule
\textbf{Node Feature} & \multicolumn{1}{c|}{\textbf{MSE/ACC(\%)}} & 0.3282/0.00 & 1.0487/3.33 & 0.9469/16.67 & 0.9277/13.33 & 0.0274/0.00 & 0.0154/\textbf{80.00} & 0.3174/6.67 & 1.0269/10.00 & \textbf{0.0144/80.00} \\
\textbf{Adjacency Matrix} & \multicolumn{1}{c|}{\textbf{AUC/AP}} & 0.5070/0.0736 & 0.4598/0.0641 & 0.5070/0.0695 & 0.4840/0.0682 & 0.5470/0.0780 & 0.4796/0.0686 & 0.4445/0.0662 & 0.5319/0.0756 & \textbf{0.9238/0.4942} \\
\textbf{\begin{tabular}[c]{@{}c@{}}Ground-truth/\\ Reconstructed graph\end{tabular}} & \multicolumn{1}{c|}{\adjustbox{valign=c}{\includegraphics[scale=0.105]{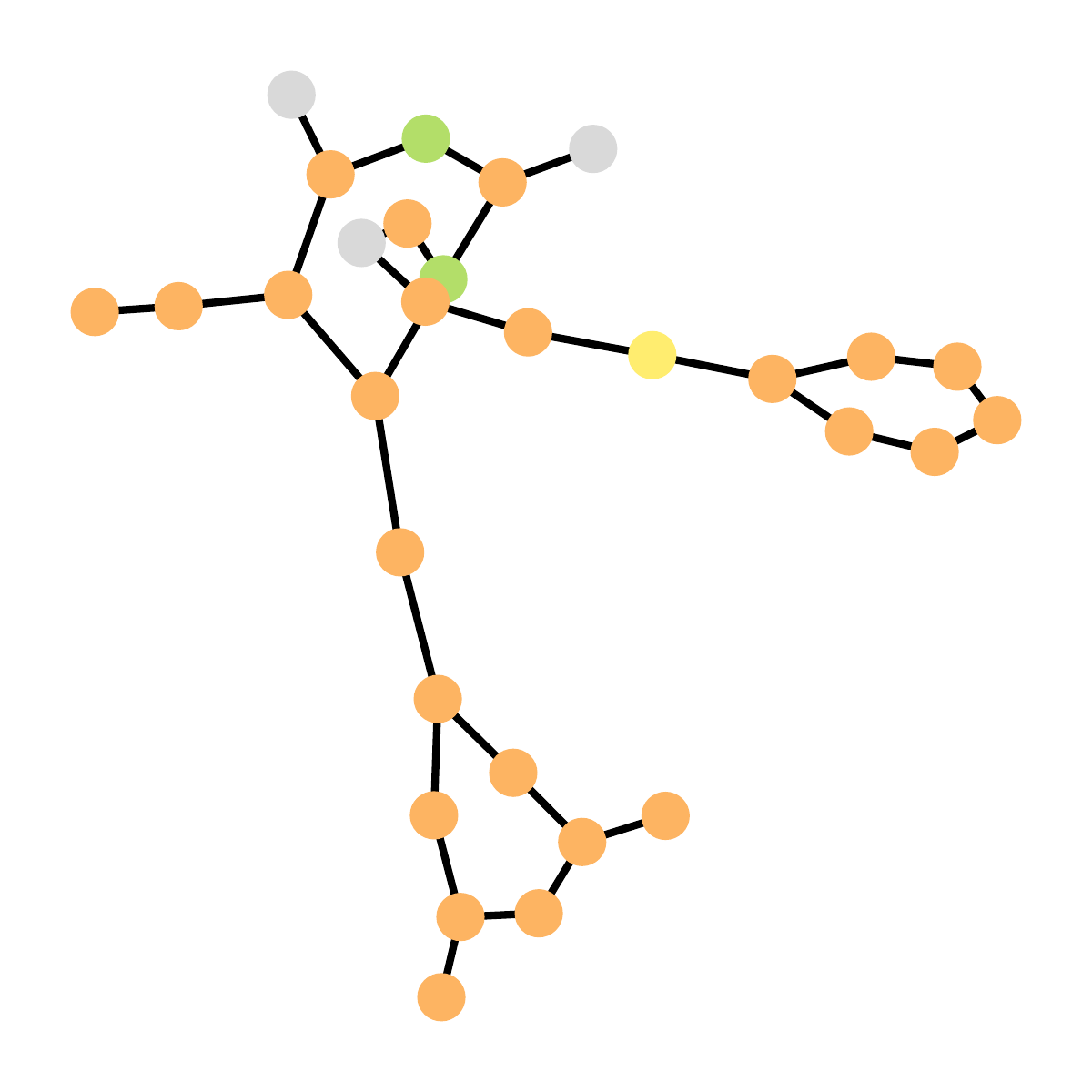}}} & \adjustbox{valign=c}{\includegraphics[scale=0.105]{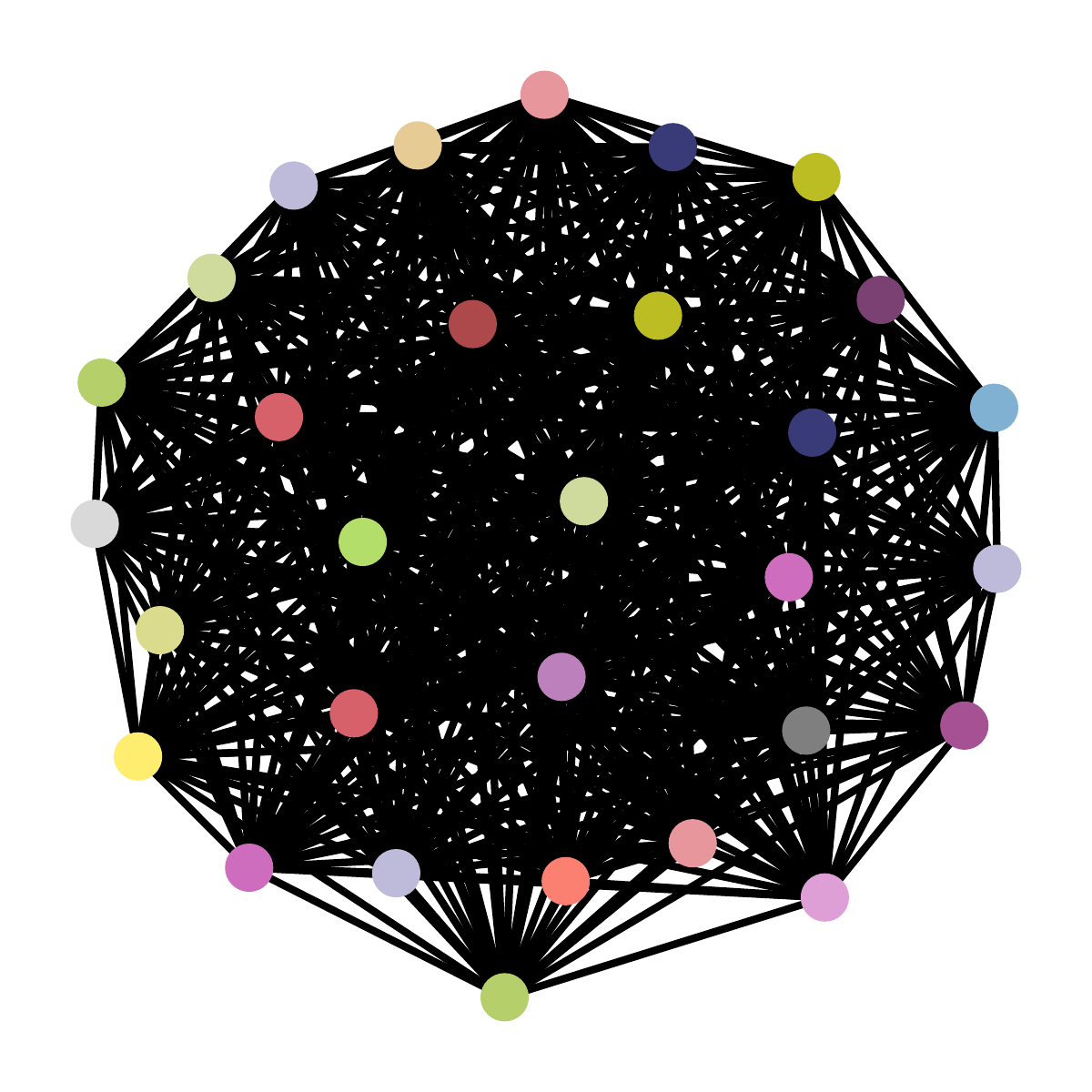}} & \adjustbox{valign=c}{\includegraphics[scale=0.105]{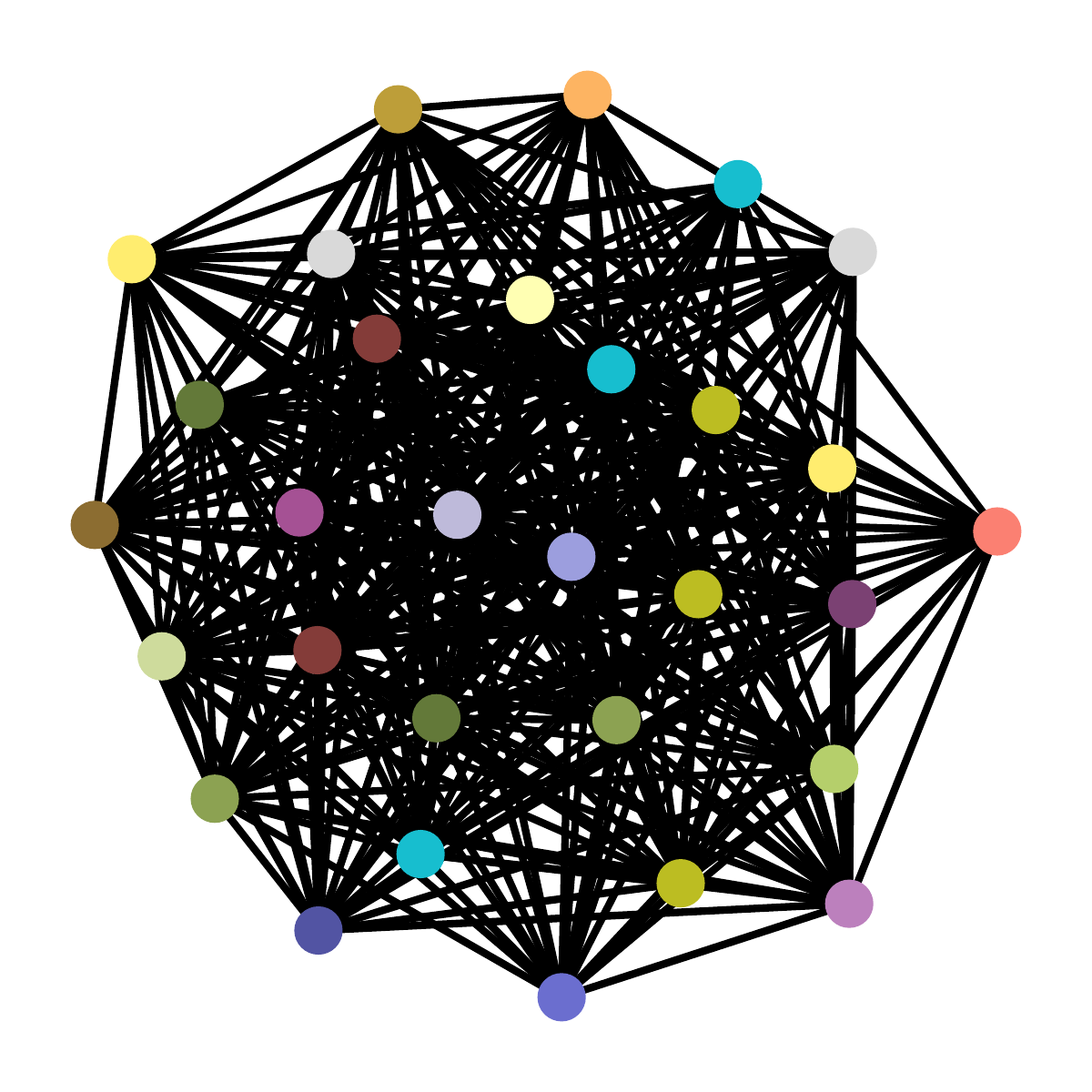}} & \adjustbox{valign=c}{\includegraphics[scale=0.105]{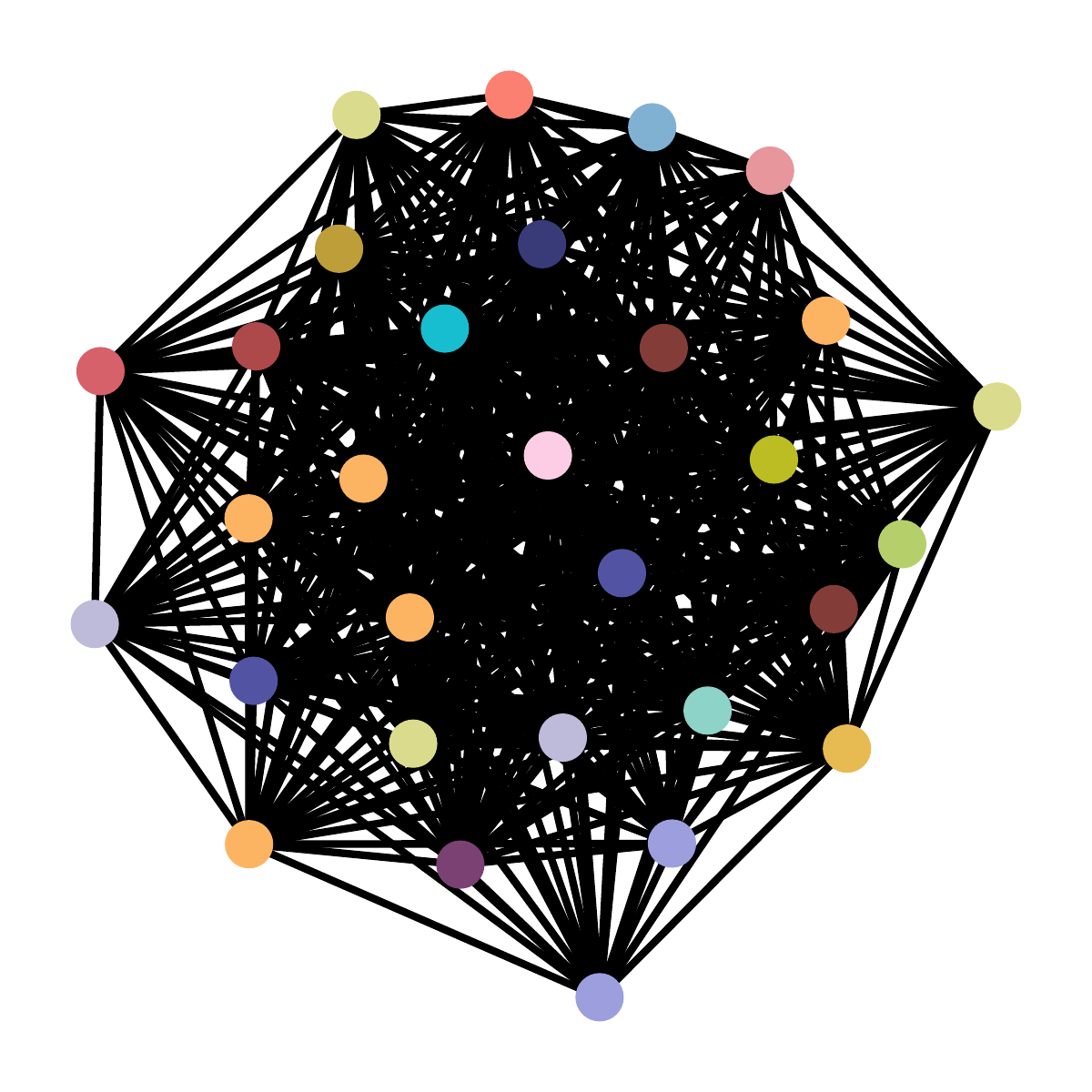}} & \adjustbox{valign=c}{\includegraphics[scale=0.105]{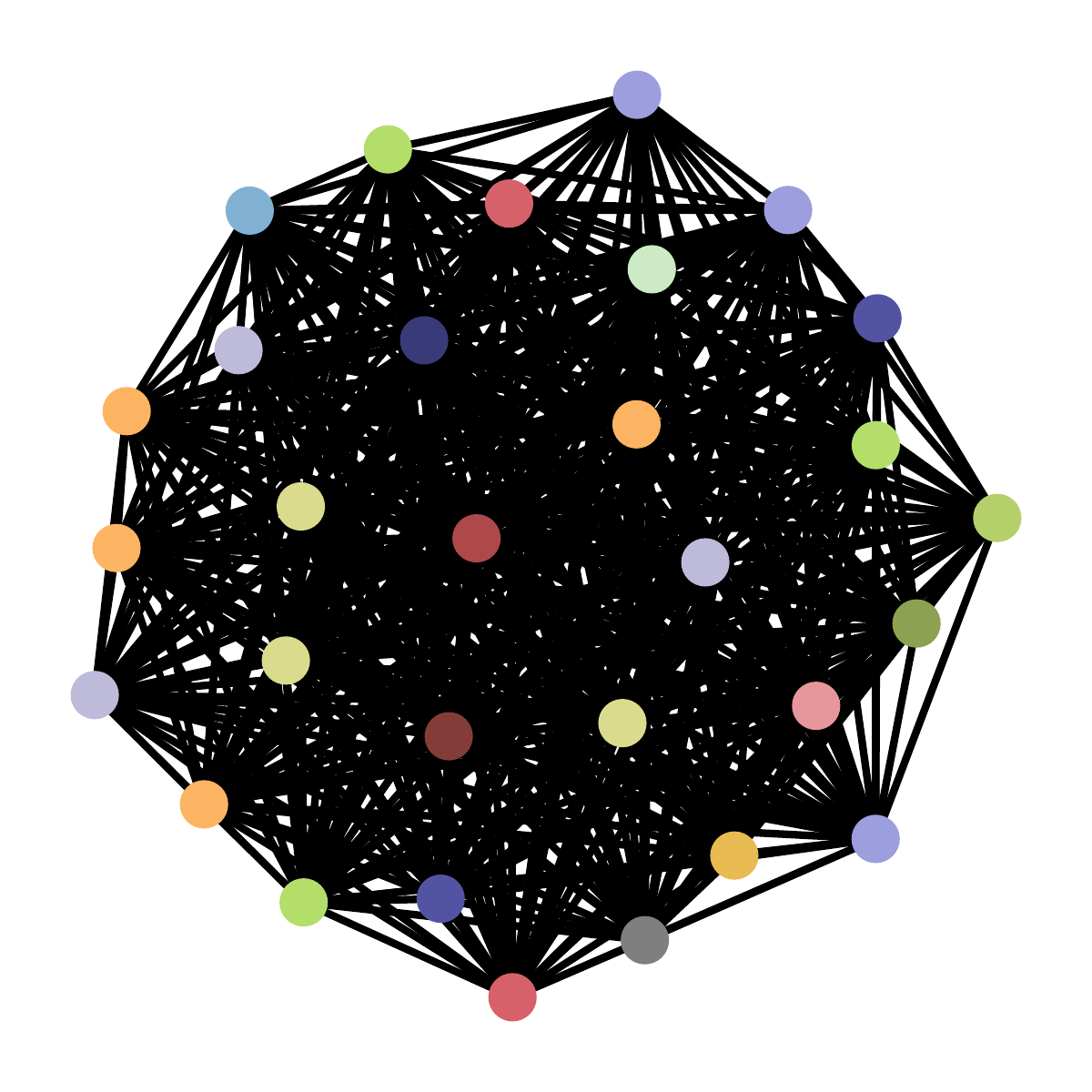}} & \adjustbox{valign=c}{\includegraphics[scale=0.105]{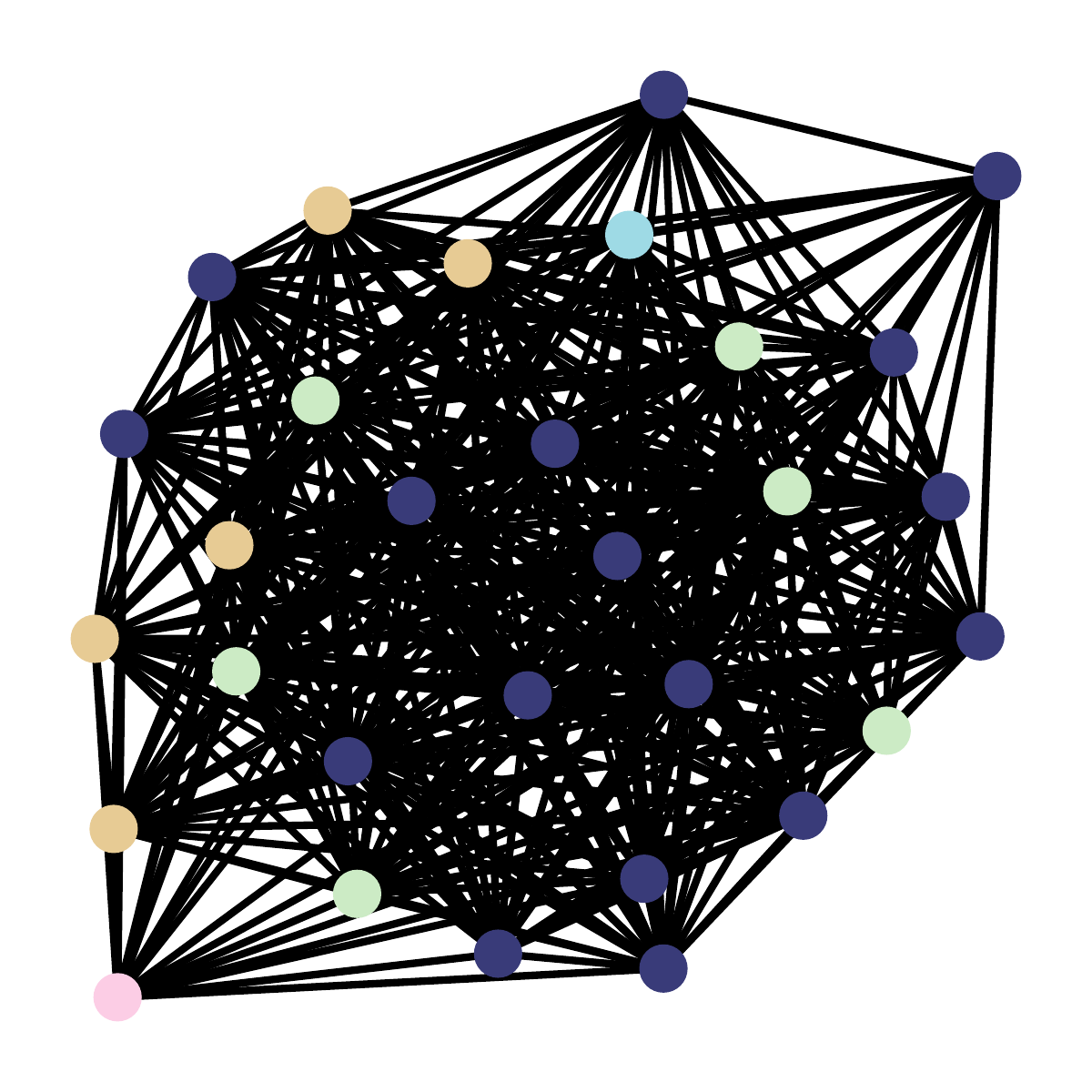}} & \adjustbox{valign=c}{\includegraphics[scale=0.105]{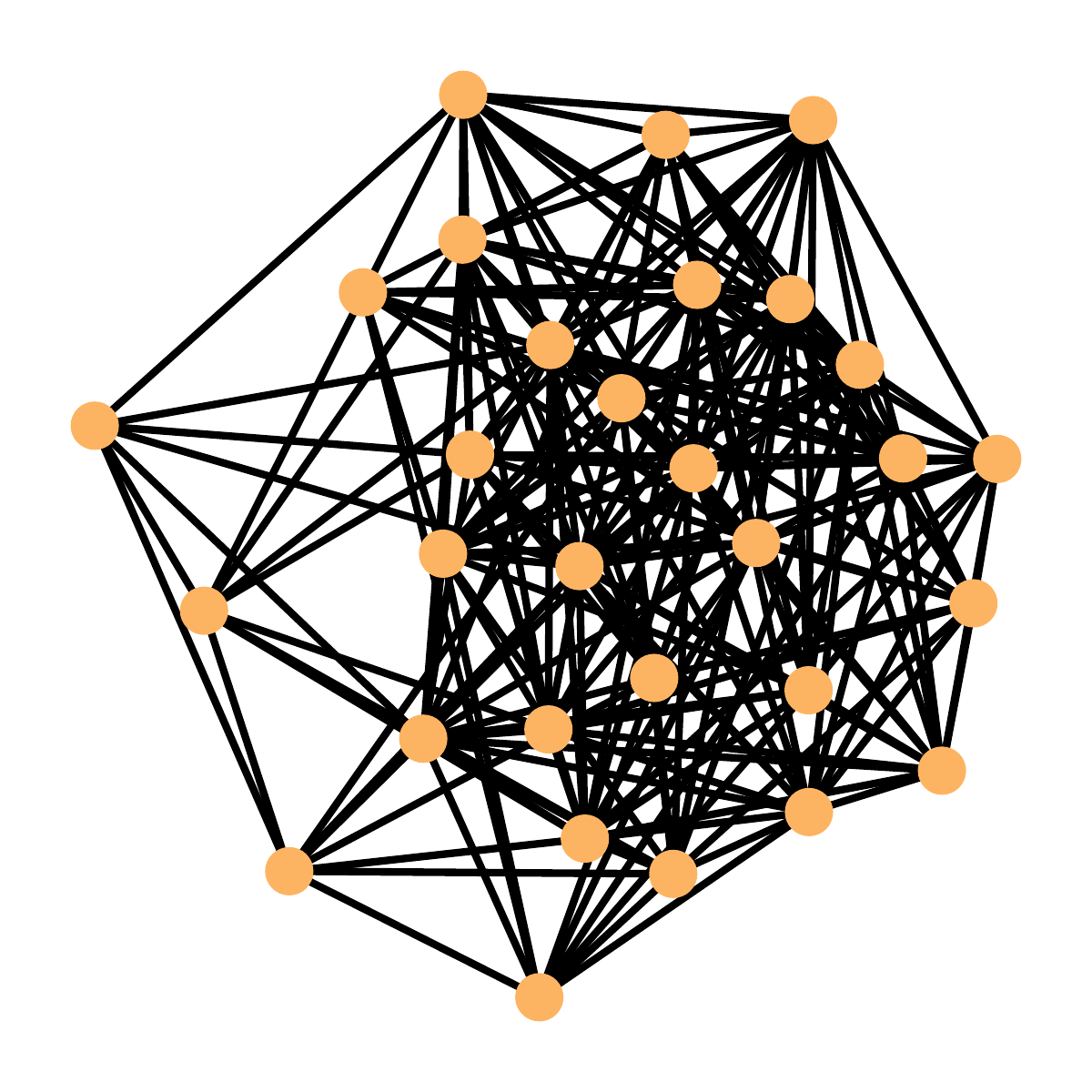}} & 
\adjustbox{valign=c}{\includegraphics[scale=0.105]{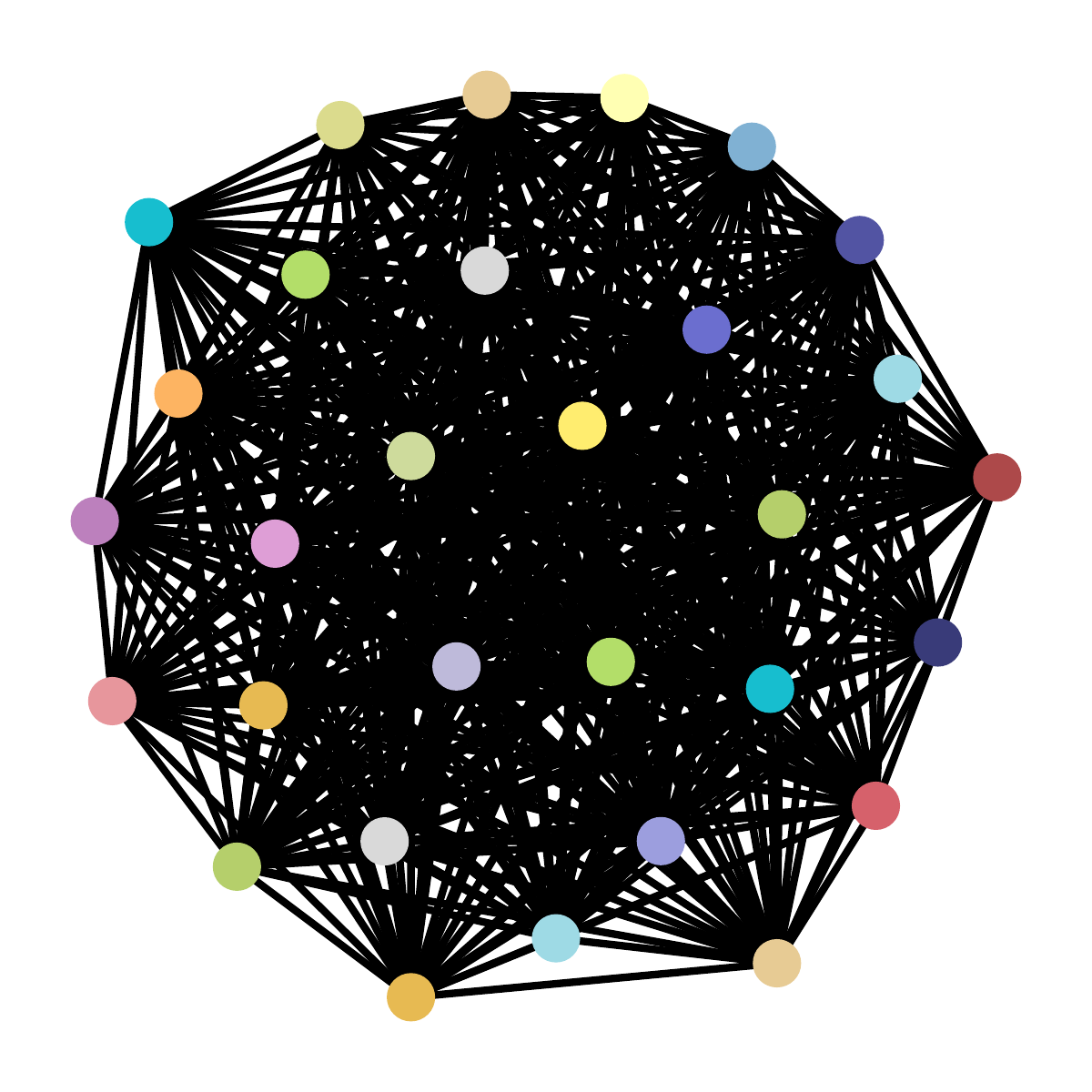}} & 
\adjustbox{valign=c}{\includegraphics[scale=0.105]{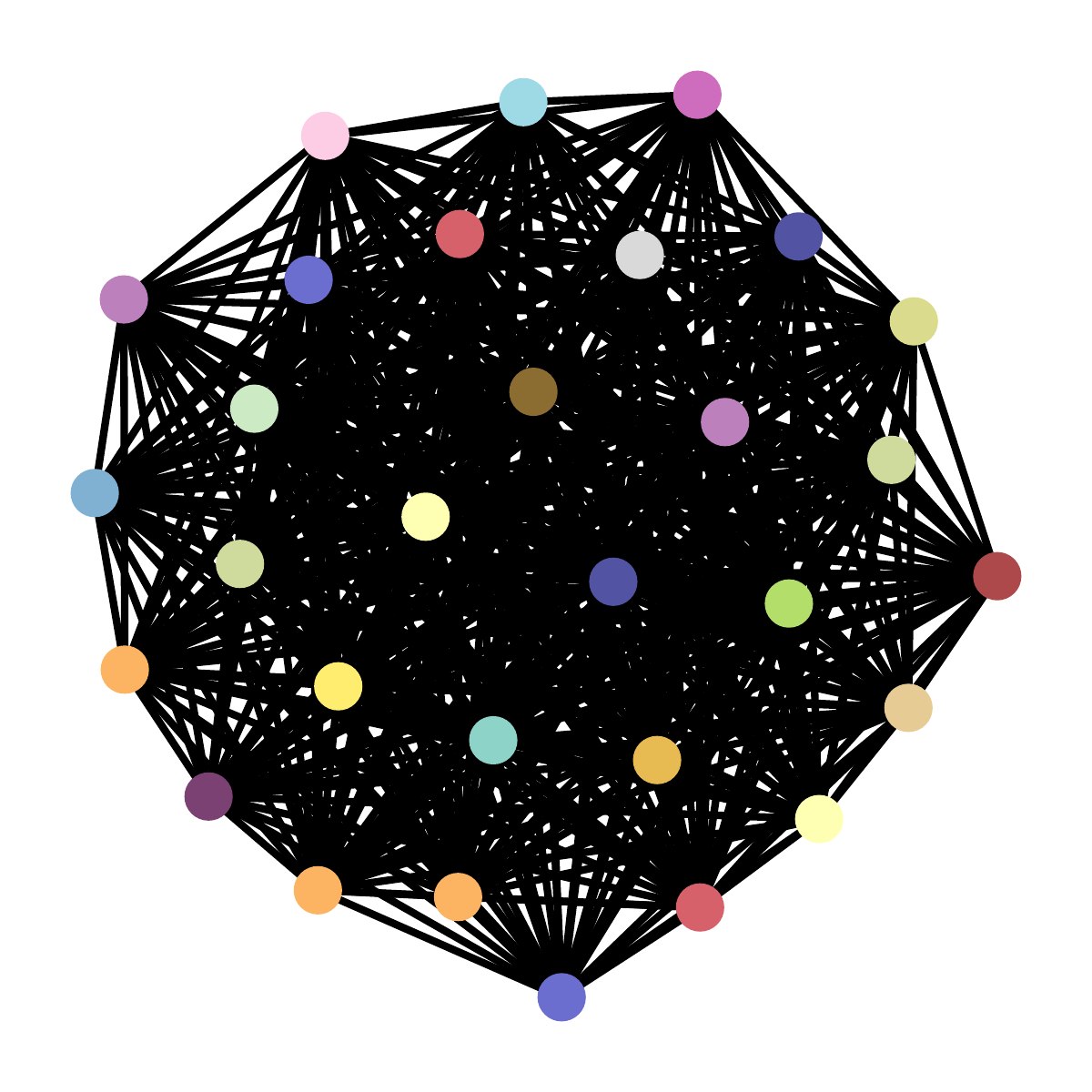}} &
\adjustbox{valign=c}{\includegraphics[scale=0.105]{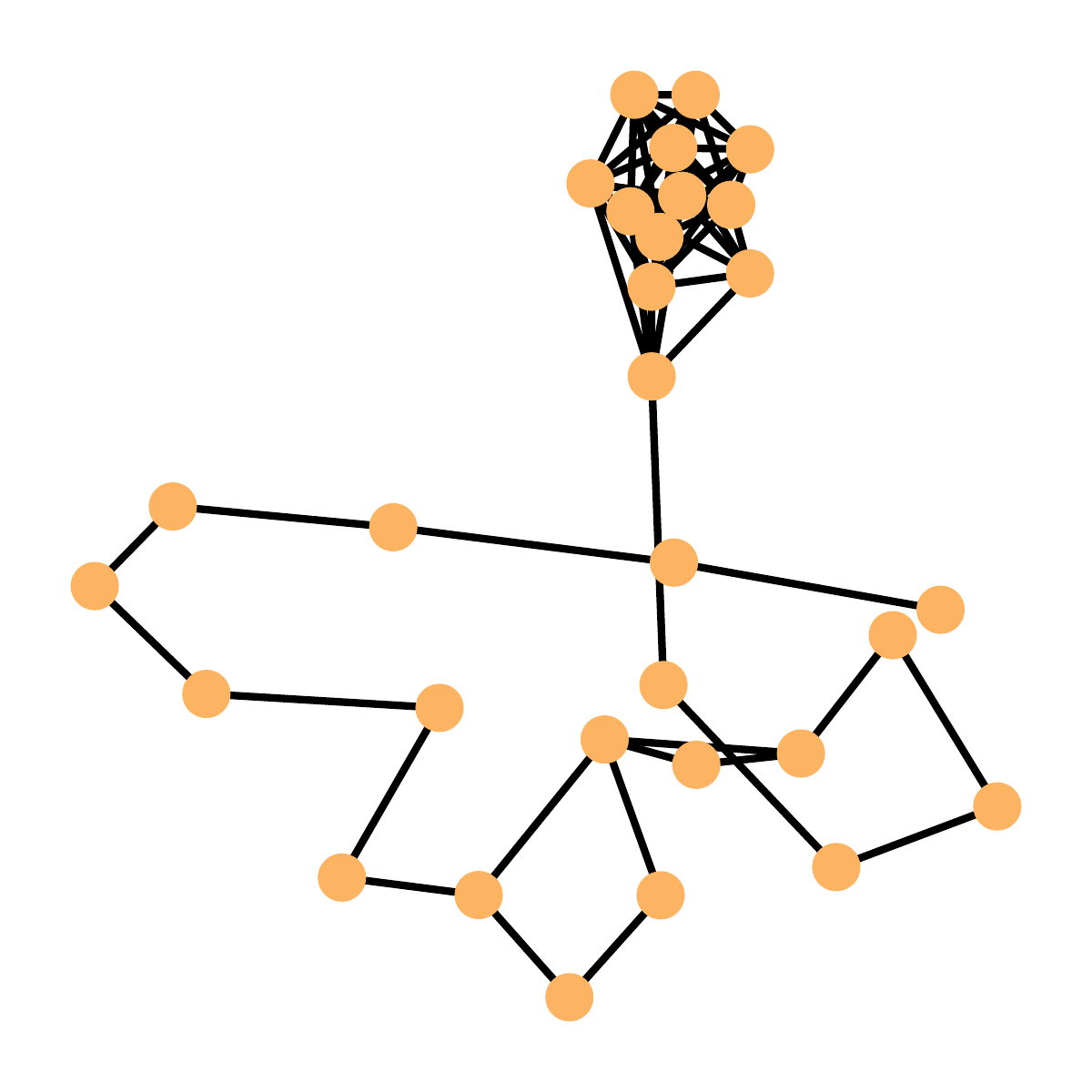}} \\ \midrule
\textbf{Node Feature} & \multicolumn{1}{c|}{\textbf{MSE/ACC(\%)}} & 0.3218/4.35 & 0.9578/13.04 & 1.0178/10.87 & 0.9178/4.35 & 0.0274/0.00 & \textbf{0.0193}/69.57 & 0.3077/2.17 & 1.0710/6.52 & 0.3580/71.74 \\
\textbf{Adjacency Matrix} & \multicolumn{1}{c|}{\textbf{AUC/AP}} & 0.4541/0.0425 & 0.4241/0.0407 & 0.4967/0.0471 & 0.4780/0.0448 & 0.4638/0.0452 & 0.4790/0.0464 & 0.5886/0.0644 & 0.5053/0.0487 & \textbf{0.8740/0.2678} \\
\textbf{\begin{tabular}[c]{@{}c@{}}Ground-truth/\\ Reconstructed graph\end{tabular}} & \multicolumn{1}{c|}{\adjustbox{valign=c}{\includegraphics[scale=0.105]{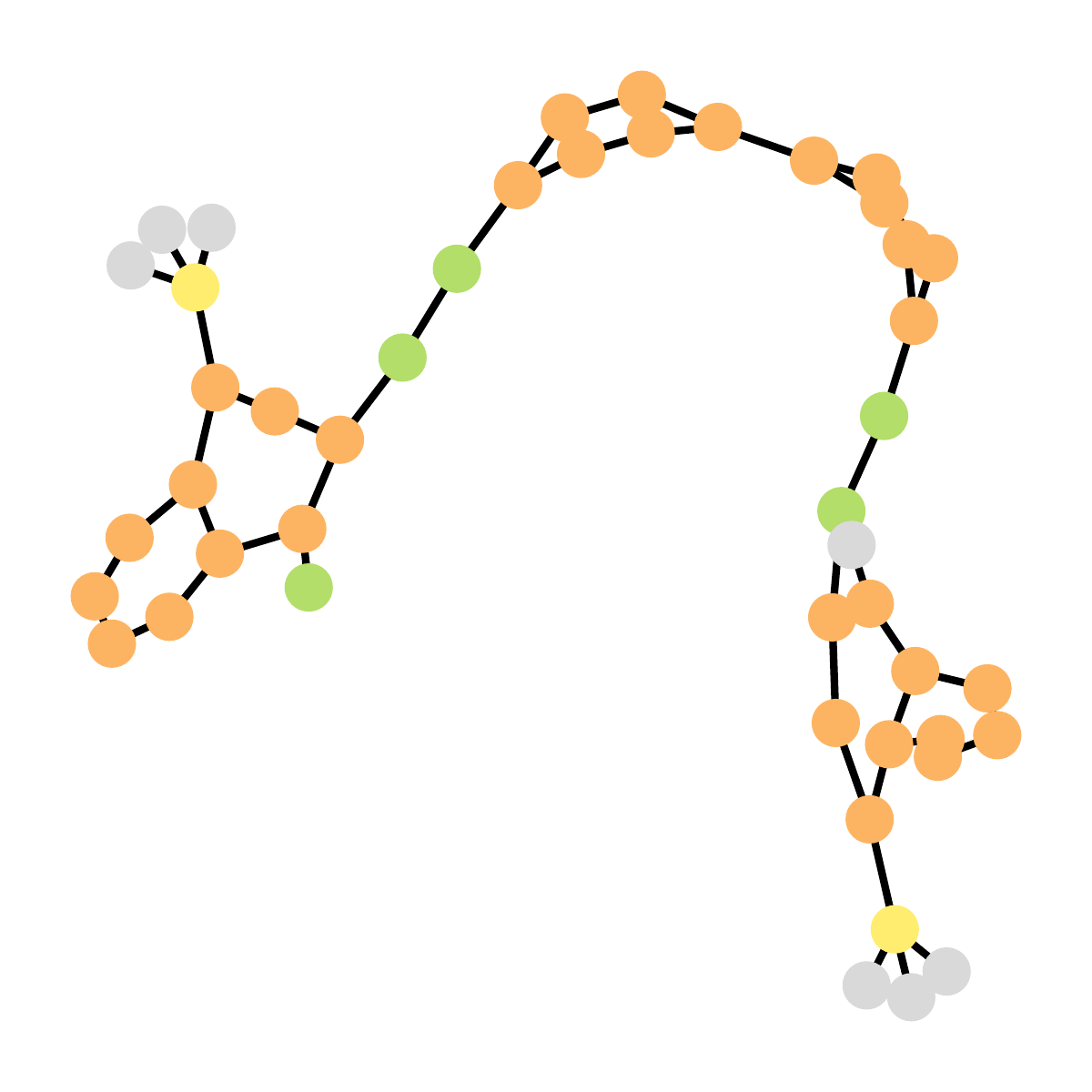}}} & \adjustbox{valign=c}{\includegraphics[scale=0.105]{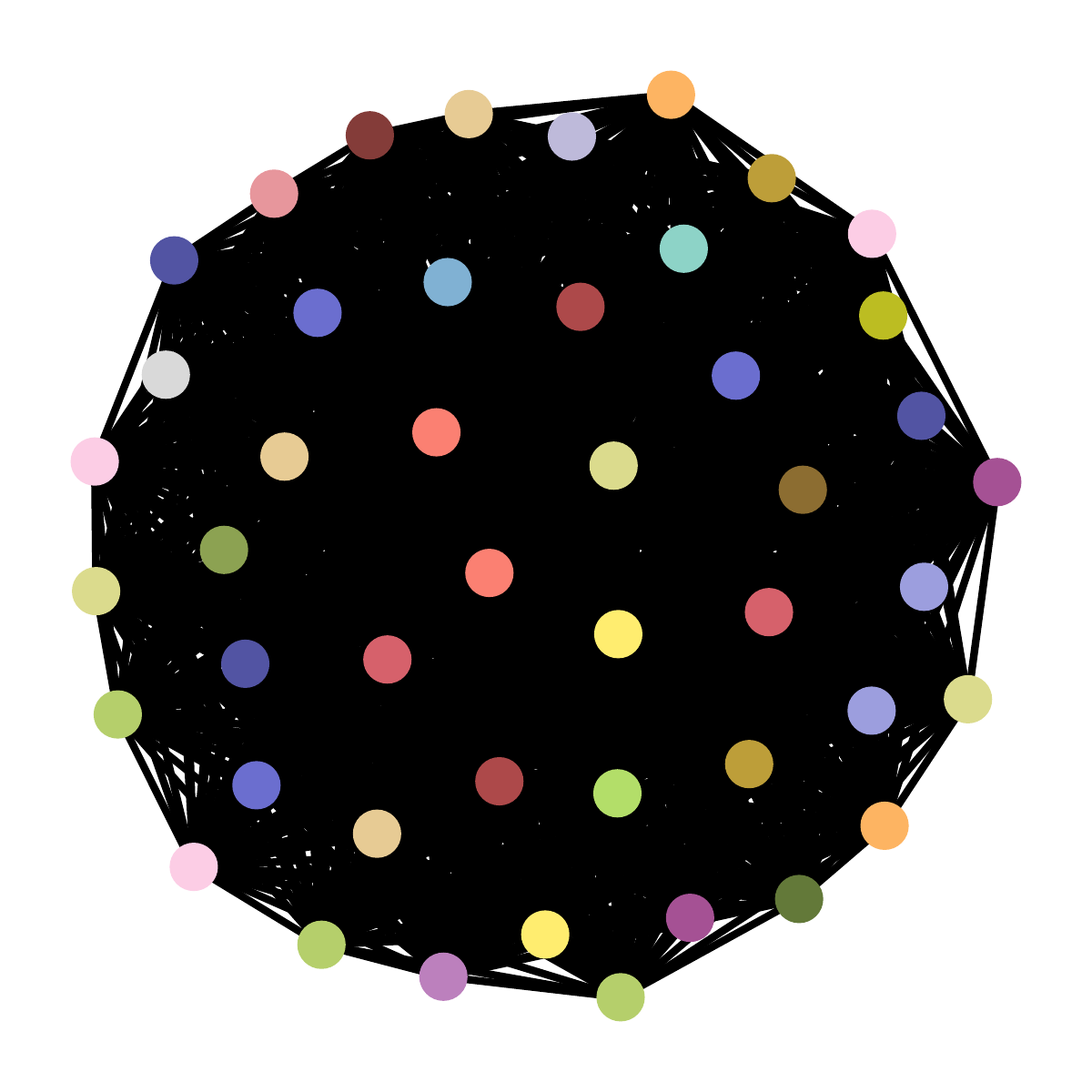}} & \adjustbox{valign=c}{\includegraphics[scale=0.105]{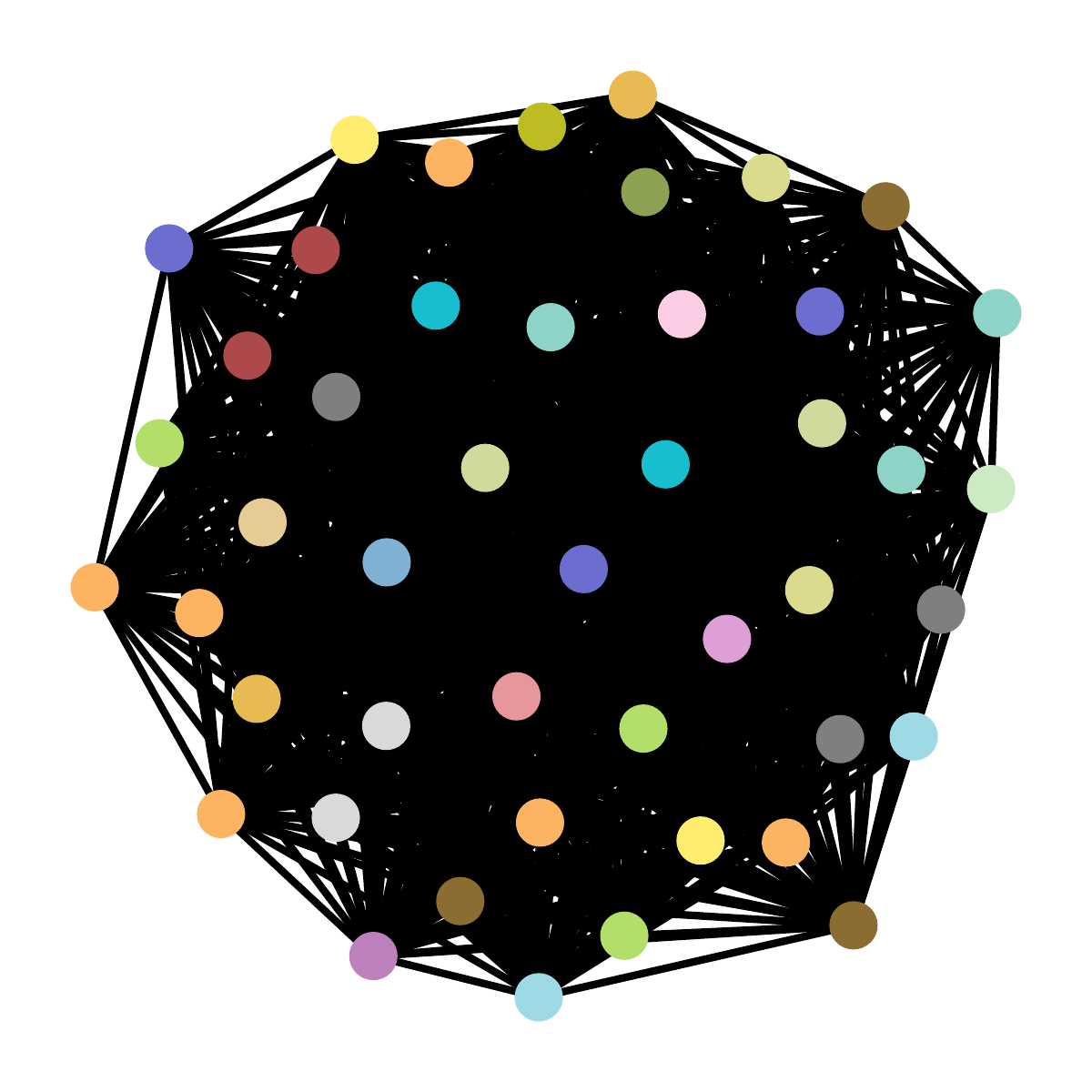}} & \adjustbox{valign=c}{\includegraphics[scale=0.105]{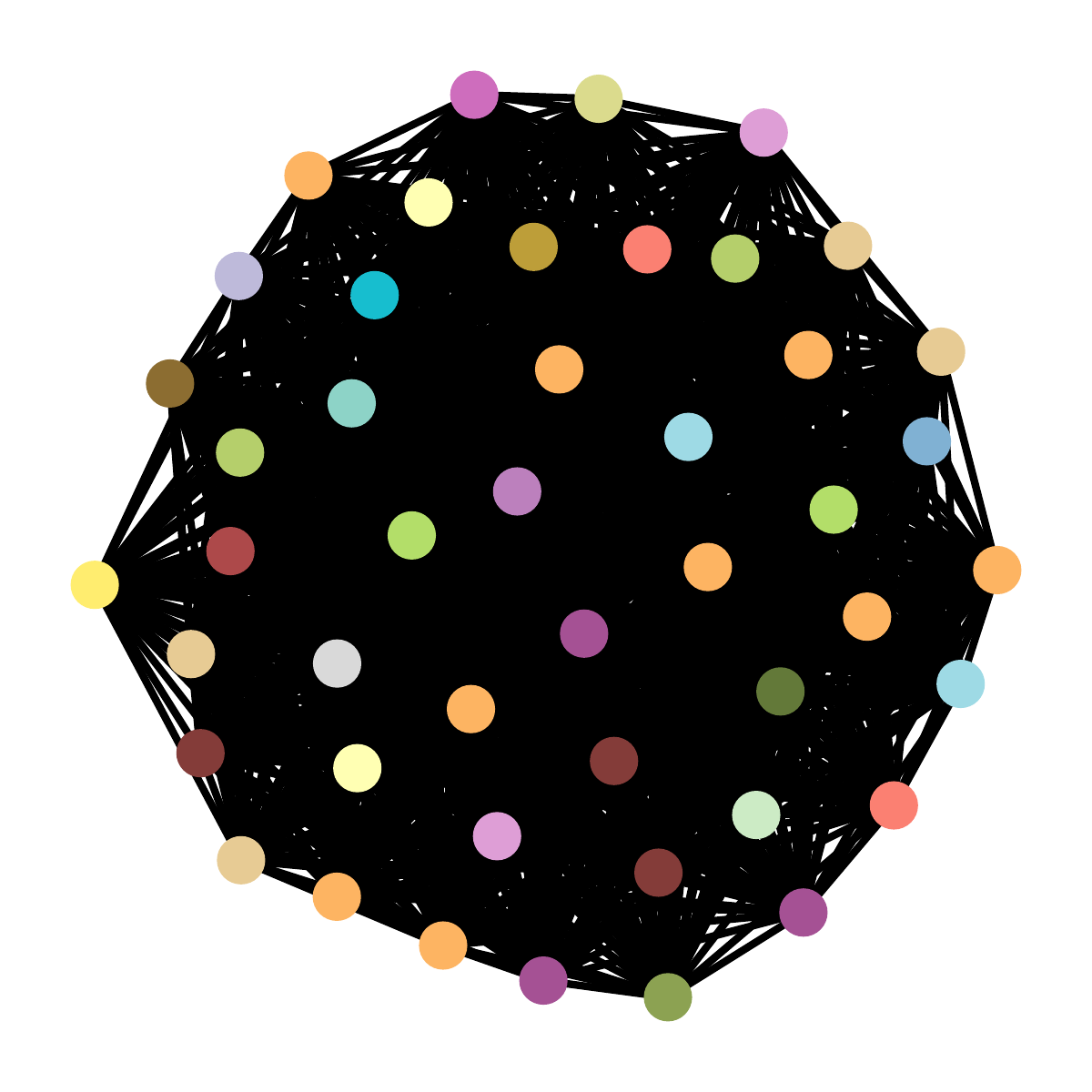}} & \adjustbox{valign=c}{\includegraphics[scale=0.105]{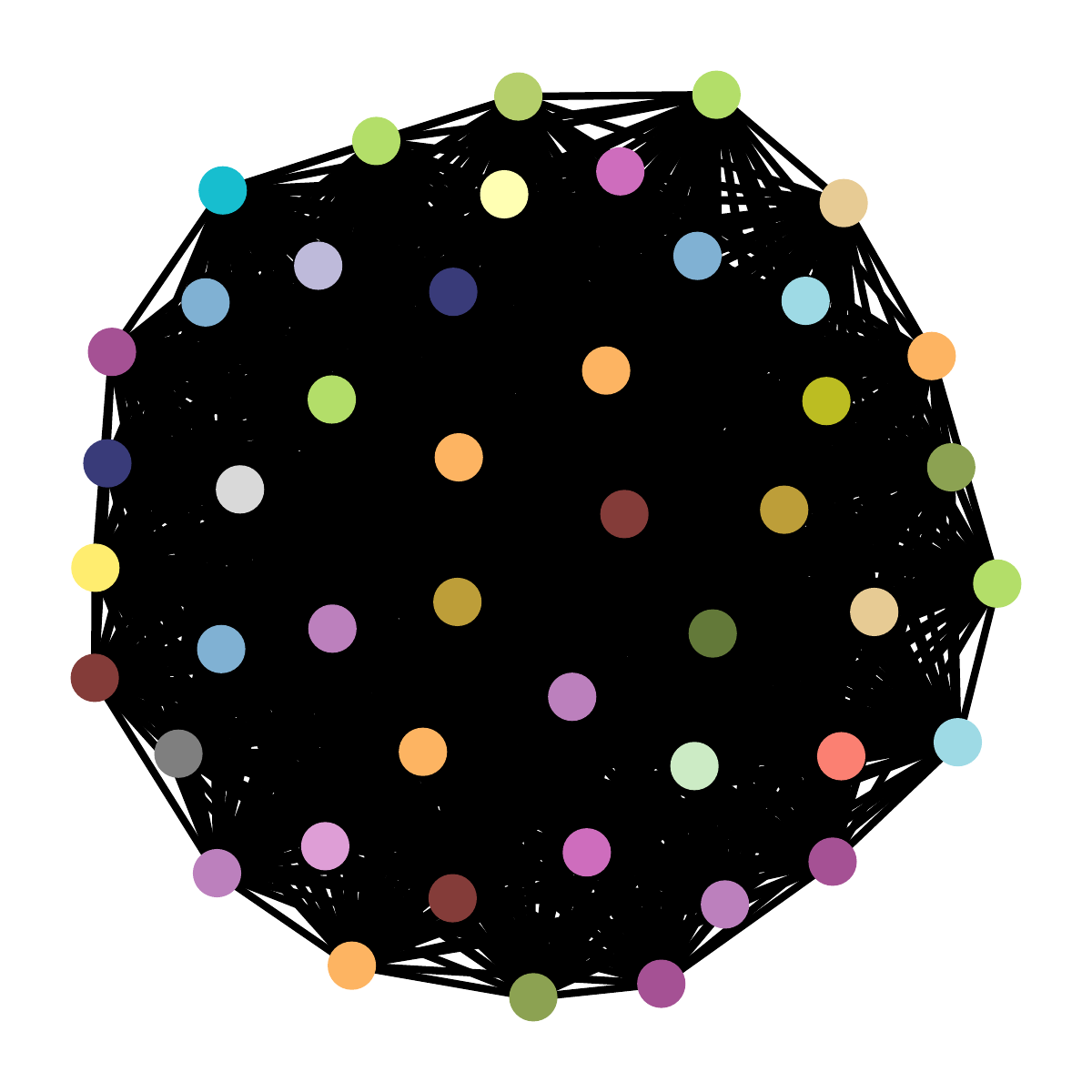}} & \adjustbox{valign=c}{\includegraphics[scale=0.105]{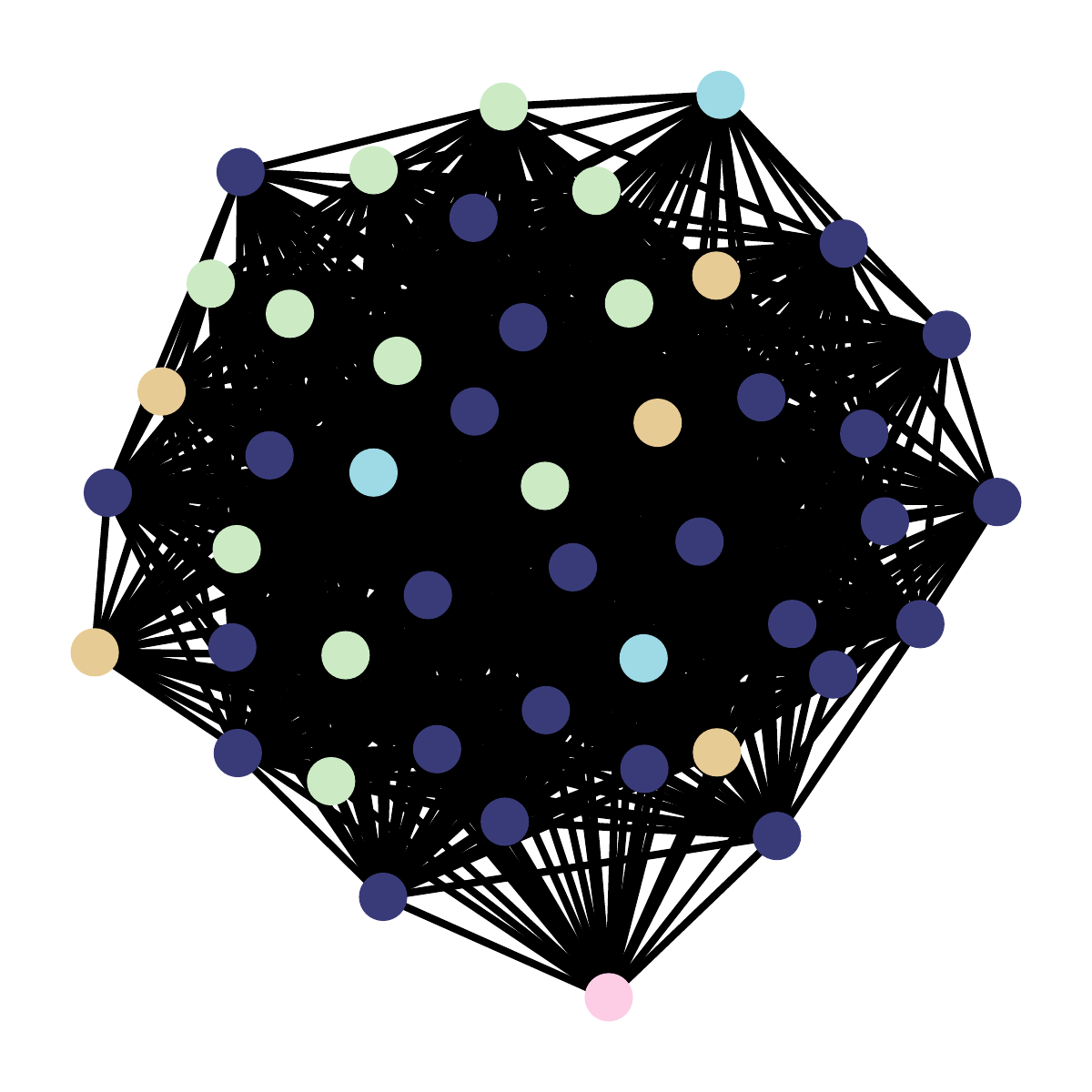}} & \adjustbox{valign=c}{\includegraphics[scale=0.105]{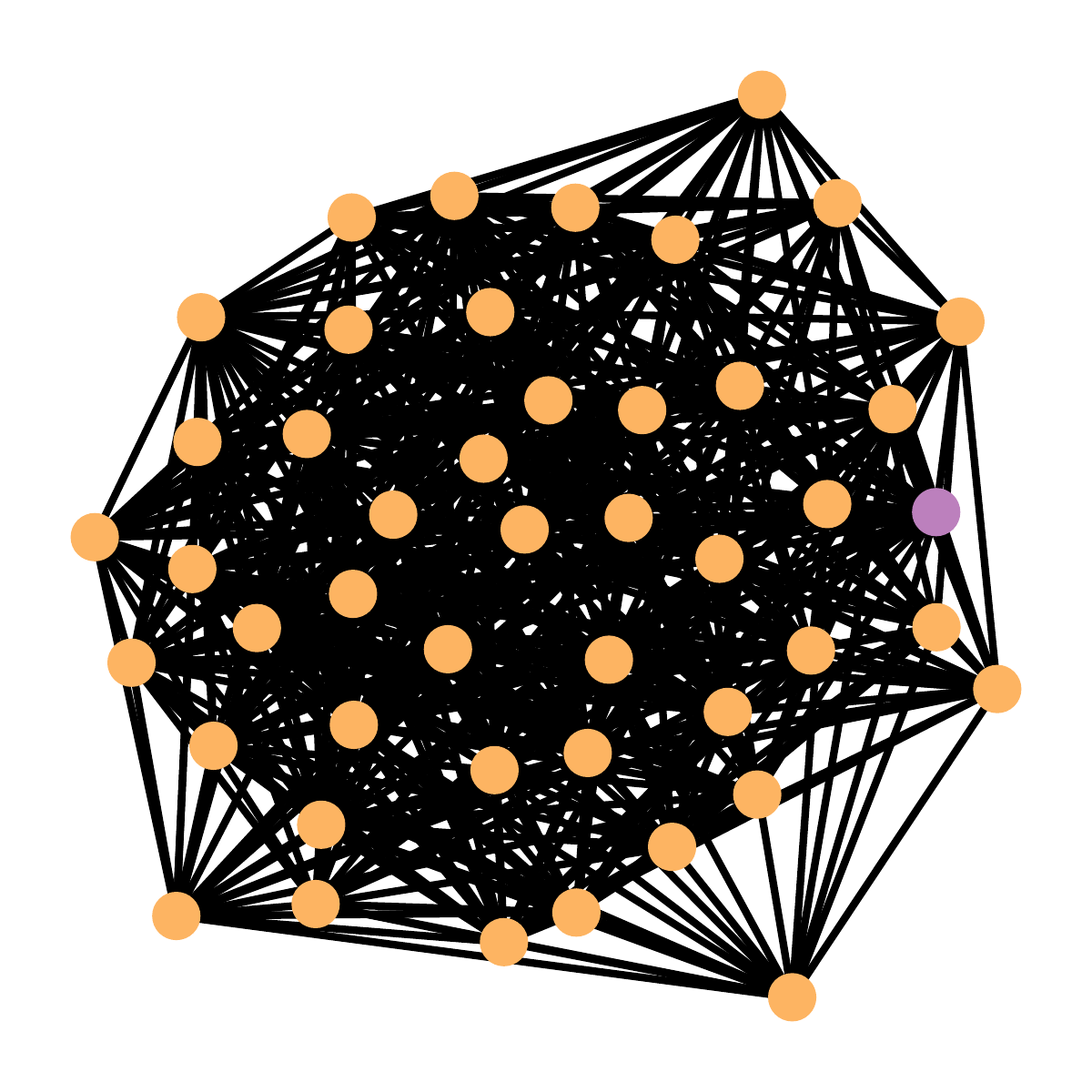}} & 
\adjustbox{valign=c}{\includegraphics[scale=0.105]{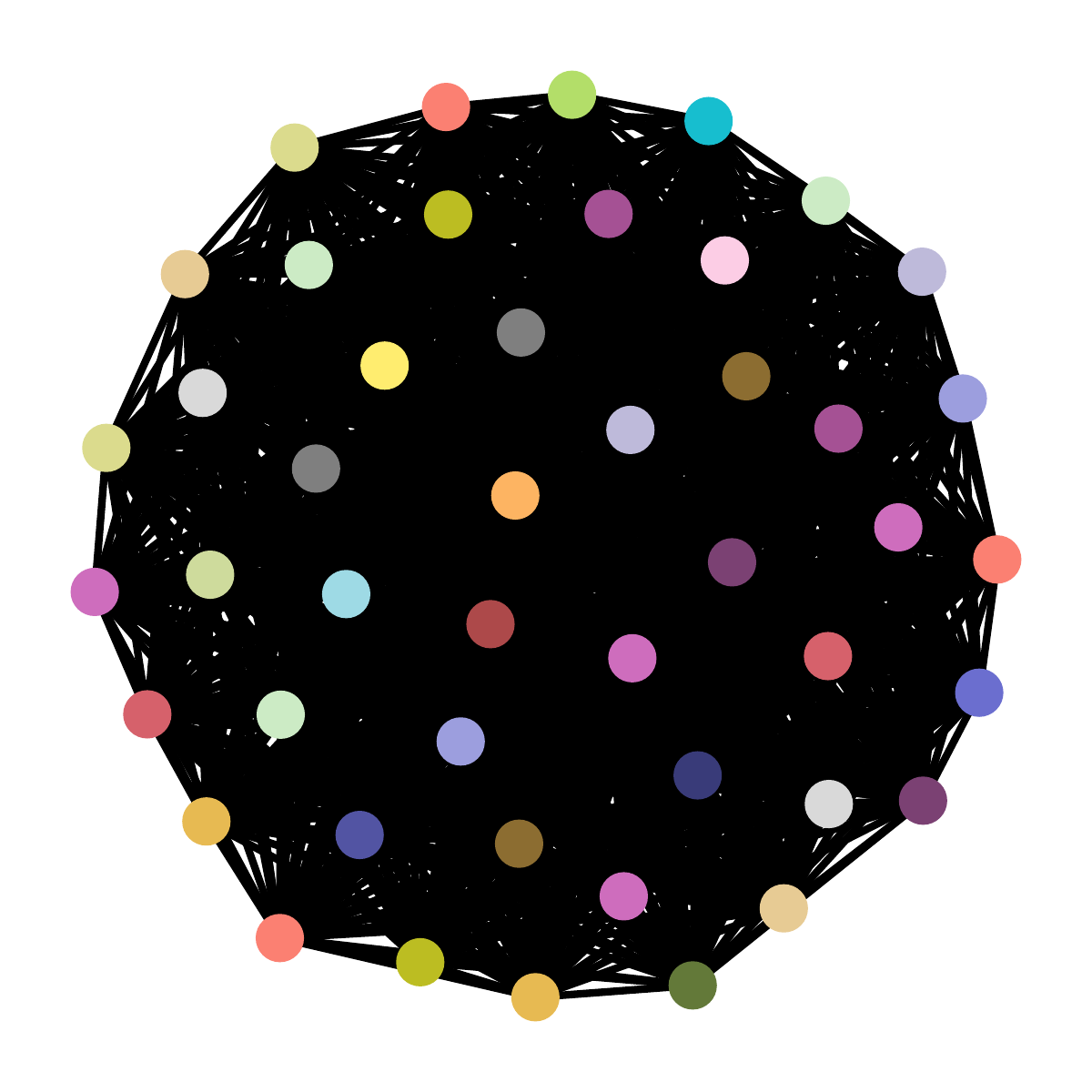}} & 
\adjustbox{valign=c}{\includegraphics[scale=0.105]{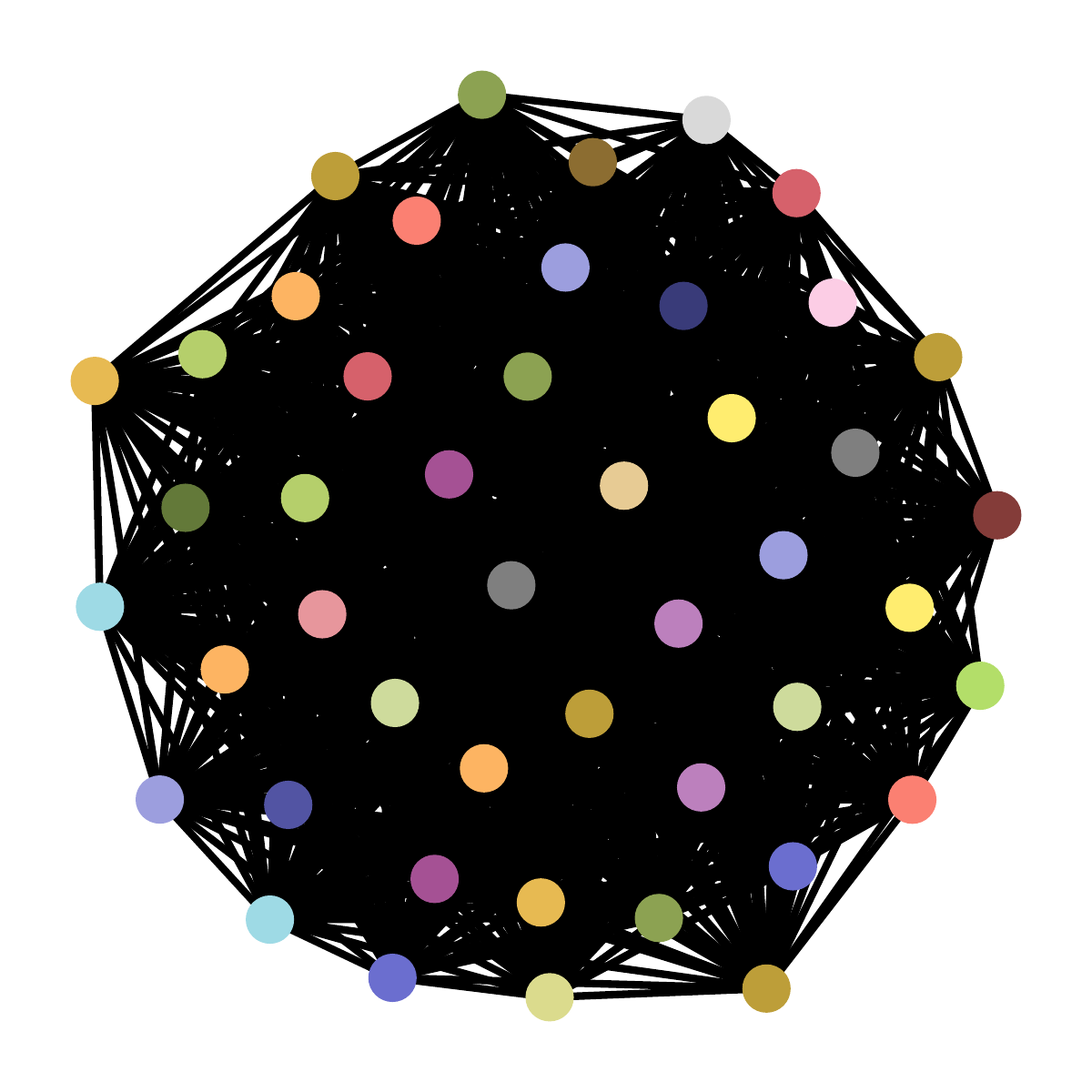}} &
\adjustbox{valign=c}{\includegraphics[scale=0.105]{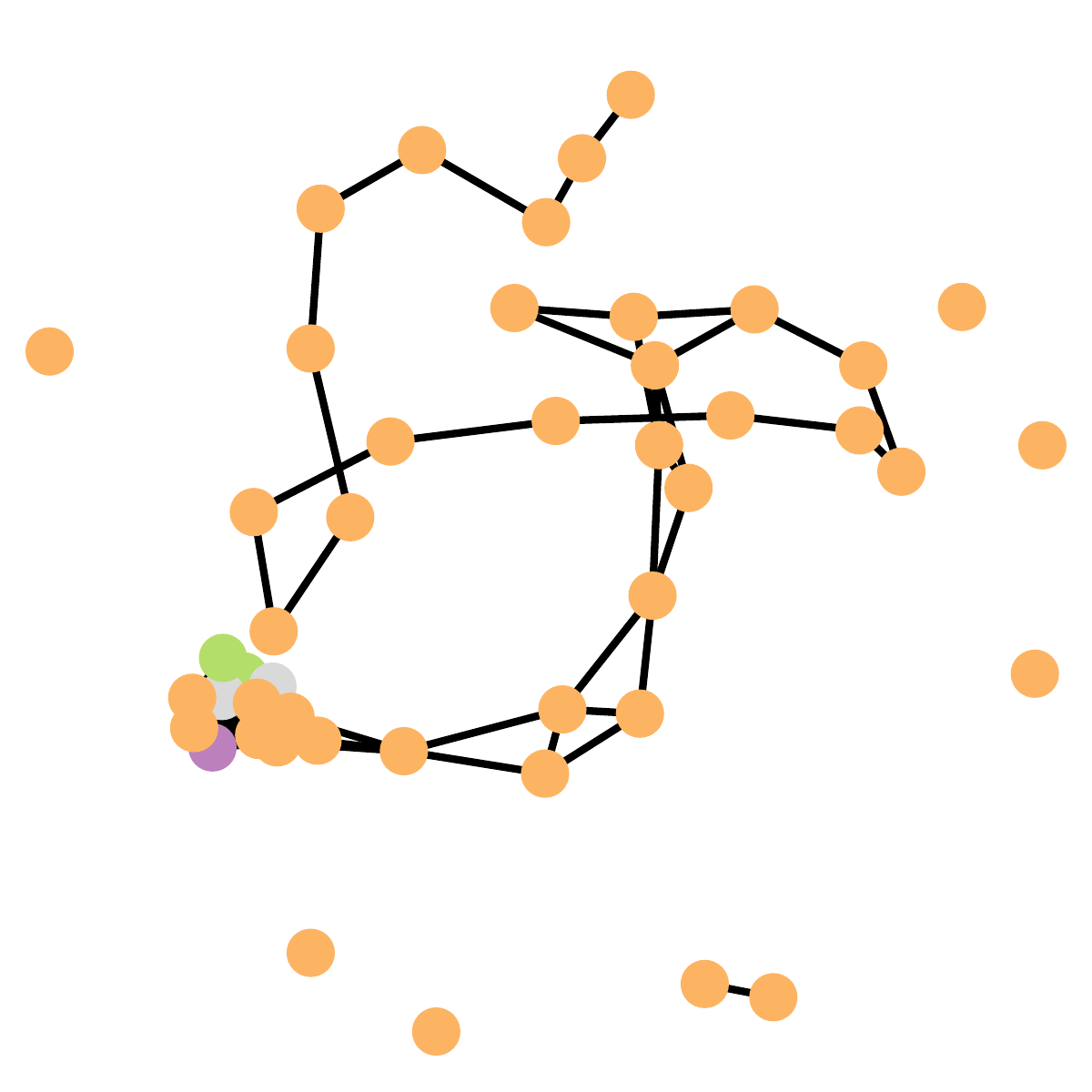}} \\ \bottomrule[1.5pt]
\end{tabular}%
}
\vspace{-1em}
\end{table*}

\setcounter{theorem}{\value{savelemma}}
\begin{theorem} [Recursive Rule for Gradients]
    The gradients $\nabla W_i$ in GCN can be depicted by its input embedding $H_{l-1}$ and a coefficient matrix $\mathbf{r_i}$ in a closed-form recursive equation, \ie,
           $\nabla{W_{i}} = H_{i-1}^\top \mathbf{r}_{i}$,
    where the coefficient matrix $\mathbf{r}_i$ is defined as follows,
    \begin{equation}  
        \begin{aligned} 
        \medop{
            \mathbf{r}_{i}
        }&
        \medop{
            =
        }
                \left \{
                \begin{array}{ll}
                     \medop{
                        \bar{A}^\top(M_p^\top\frac{\partial L}{\partial \hat Y}W_{fc}\odot\sigma_i'),
                     } & \medop{i=l}\\
                     \medop{
                        \bar{A}^\top (\mathbf{r}_{i+1} W_{i+1}^\top\odot \sigma_{i}'),
                    } & \medop{i=1,\dots, l-1}
                \end{array} \right.
        \end{aligned}
    \end{equation}
\end{theorem}

\begin{proof}
When $i=l$, it can be proved by rewriting  \equref{eq:gcn_g_l_app} in \lemmaref{lemma:gradients} as $\nabla{W_{l}} = H_{l-1}^\top \mathbf{r}_{l}$, where $\mathbf{r}_l=\bar{A}^\top(M_p^\top\frac{\partial L}{\partial \hat Y}W_{fc}\odot\sigma_l')$.

When $i=l-1, l-2,\dots,1$, we will utilize mathematical induction to prove the equation
        $dL = tr(\mathbf{r}_{i}W_{i}^\top dH_{i-1}^\top),\ 
        \medop{i=1,\dots,l-1}$.
Then, we can also establish the proof of \thmref{thm:gradients} by induction.
Our proof consists of two steps: the base case and the induction step.

\noindent \textbf{Step 1: Base Case.} When $i=l-1$, we have $r_{l-1}=\bar{A}^\top (\mathbf{r}_{l} W_{l}^\top\odot \sigma_{l}')=\bar{A}^\top (\bar{A}^\top(M_p^\top\frac{\partial L}{\partial \hat Y}W_{fc}\odot\sigma_l') W_{l}^\top\odot \sigma_{l-1}')$.
By \equref{eq:gcn_g_l-1_app} and \equref{eq:base_case_trace}, $\nabla{W_{l-1}}=H_{l-2}^\top \mathbf{r}_{l-1}$ and $dL=tr(\mathbf{r}_{l-1}W_{l-1}^\top dH_{l-2}^\top)$ are true.

\noindent  \textbf{Step 2: Induction Step.} Next, we assume that the formulas are true for any arbitrary positive integer $k(1 < k \le l-1)$.
We aim to prove that when $i=k-1$, $dL = tr(\mathbf{r}_{i}W_{i}^\top dH_{i-1}^\top)$ also hold true. 

According to the induction hypothesis, we have:
\begin{equation}
        \medop{dL = tr(\mathbf{r}_{k}W_{k}^\top dH_{k-1}^\top)\\
           = tr(\mathbf{r}_{k}W_{k}^\top d(\sigma_{k-1}(W_{k-1}^\top H_{k-2}^\top \bar{A}^\top))}
\end{equation}
By the property of Hadamard operation in \equref{eq:induc_hadam}, we have,
\begin{equation}\label{eq:trace_k-1}
    \begin{aligned}
        &\medop{dL 
        = \text{tr}(\mathbf{r}_{k} W_{k}^\top ((\sigma_{k-1}^\top)'\odot d(W_{k-1}^\top H_{k-2}^\top \bar A^\top)) )
        = \text{tr}((\mathbf{r}_{k} W_{k}^\top \odot \sigma_{k-1}')d(W_{k-1}^\top H_{k-2}^\top \bar A^\top) )}\\
        &\medop{
        = \text{tr}(H_{k-2}^\top \bar A^\top(\mathbf{r}_{k} W_{k}^\top \odot \sigma_{k-1}')dW_{k-1}^\top ) 
        = \text{tr}(H_{k-2}^\top \mathbf{r}_{k-1} dW_{k-1}^\top )
        }
        \\
    \end{aligned}
\end{equation}
where $\mathbf{r}_{k-1} $ is as its definition, \ie, $\mathbf{r}_{k-1}=\bar{A}^\top (\mathbf{r}_{k} W_{k}^\top\odot \sigma_{k-1}')$. 

Thus, we have $\nabla{W_{k-1}}=
            \frac{\partial L}{\partial W_{k-1}}
               =H_{k-2}^\top \mathbf{r}_{k-1}$.
Then, we can compute the full differential w.r.t. $H_{k-2}$ based on \equref{eq:trace_k-1}:
\begin{equation}
        \medop{
            dL 
            = \text{tr}((\mathbf{r}_{k} W_{k}^\top \odot \sigma_{k-1}')W_{k-1}^\top dH_{k-2}^\top \bar A^\top ) \\
            = \text{tr}(\mathbf{r}_{k-1}W_{k-1}^\top dH_{k-2}^\top  )
        } \\
\end{equation}
As previous assumptions also hold, we establish the induction step.

\noindent\textbf{Conclusion.} 
Finally, by mathematical induction, we finish the proof.

\end{proof}

\subsection{More Implementation Details}
\label{sec:app_setup}

\fakeparagraph{Experimental Environment}
We implement all methods using PyTorch 2.2.1 and the PyTorch Geometric (PyG) library.
All experiments were conducted on an NVIDIA A100-PCIE-40GB server with CUDA 11.8 and Intel(R) Xeon(R) Gold 6230R CPU @ 2.10GHz.
We utilize the \textsf{PyG} library to handle the mentioned graph datasets, which involves removing isomorphic graphs.
During partitioning, we divide the datasets into two parts based on the graph labels employing Dirichlet distribution with $\alpha=1.0$, widely used for simulating a non-i.i.d (heterogeneous) setting in federated learning.

\subsection{Baseline Descriptions} \label{sec:app_baselines}

The main ideas of the baselines is presented as follows.
\begin{itemize}[leftmargin=*, noitemsep, topsep=0pt]
    \item \textbf{Random:} It randomly generates both graph structures and features from a uniform distribution. 
    \item \textbf{DLG}~\cite{nips19:dlg}: It is the seminar work that introduces an optimization process based on gradient matching using L2 distance and the L-BFGS optimizer.   For extension to graphs, we initialize both the adjacency matrix and node features as dummy data.
    \item \textbf{iDLG}~\cite{corr20:idlg}: It improves DLG by introducing label inference from gradients to enhance recovery.
    \item \textbf{InverGrad}~\cite{nips20:ig}: It employs an optimization framework with cosine distance and the Adam optimizer, along with a total variation loss originally tailored for images. As its total variation loss is unsuitable for graphs, we omit it in our implementation.
    \item \textbf{GI-GAN:} It represents the GAN-based approaches~\cite{nips21:gias,cvpr22:ggl} that optimize a dummy latent vector and feed it into a pre-trained GAN for recovery, thereby leveraging GAN-learned priors. Specifically, we adopt MolGAN~\cite{icml18:molgan} as an extension for graph data.
    \item \textbf{GRA-GFA}~\cite{tnse23:empirical_graph}: It extends DLG to reconstruct graph data from optimized node embeddings $H$. Specifically, it recovers the adjacency matrix as $\sigma(HH^T)$ and reconstructs node features using a generator that takes the inferred adjacency matrix and a randomly initialized feature matrix as input.
    \item \textbf{TabLeak}~\cite{icml23:tableak}: It is originally designed for tabular data.  We extend it for graph data by treating adjacency matrices as continuous features and node attributes as discrete features, while employing its entropy-based uncertainty quantification to assess reconstruction quality.
    \item \textbf{Graph Attacker}~\cite{arxiv24:Graph_Attacker}: it reconstructs graph data by matching dummy gradients with actual gradients while enforcing graph-specific properties through feature smoothness and sparsity regularization terms in the optimization objective.
\end{itemize}

\subsection{More Case Study Results}
\label{sec:app_visualization}
\tabref{tab:app_visualization} shows visualization results for various concrete biometric or chemical molecules in the performance comparison experiment, including the reconstruction results for \sysname and baselines.
We can draw a similar conclusion to that mentioned in \secref{sec:exp_baseline}: Our method outperforms in recovering both the atom (node) type and the molecule structure, while other baselines tend to reconstruct a fully connected graph and incorrect node types, especially as the number of nodes increases, failing to recover the atom types and molecule structure.
However, it remains challenging to reconstruct graphs with complex structures and a large number of nodes.

\subsection{More Experimental Studies}
\label{sec:app_others}

\fakeparagraph{Studies on Larger Batch Sizes} 
We further compare the proposed method with baselines under larger batch sizes to validate the effectiveness.
As shown in \tabref{tab:batch_size} on the MUTAG dataset, GraphDLG consistently outperforms the baselines across different batch sizes, achieving an oblivious improvement of 16.35\% in node feature ACC and 66.58\% in graph structure AUC, thereby demonstrating the effectiveness and robustness of our approach.

\begin{table}[htb]
\vspace{-1em}
\caption{Performance with larger batch sizes.}
\vspace{-1em}
\label{tab:batch_size}
\centering
\resizebox{0.9\columnwidth}{!}{
\setlength{\tabcolsep}{0.5pt}
\begin{tabular}{@{}c|cc|cc|cc@{}}
\toprule[1.5pt]
\multirow{3}{*}{\textbf{Method}} & 
\multicolumn{2}{c|}{\textbf{Batch size=4}} & 
\multicolumn{2}{c|}{\textbf{Batch size=8}} & 
\multicolumn{2}{c}{\textbf{Batch size=16}} \\
\cmidrule(lr){2-3} \cmidrule(lr){4-5} \cmidrule(lr){6-7}

& 
\multicolumn{1}{c}{\textbf{Feature}} & 
\multicolumn{1}{c|}{\textbf{Structure}} & 
\multicolumn{1}{c}{\textbf{Feature}} & 
\multicolumn{1}{c|}{\textbf{Structure}} & 
\multicolumn{1}{c}{\textbf{Feature}} & 
\multicolumn{1}{c}{\textbf{Structure}} \\

& 
\multicolumn{1}{c}{\textbf{ACC (\%)}} & 
\multicolumn{1}{c}{\textbf{AUC} \quad \textbf{AP}} & 
\multicolumn{1}{c}{\textbf{ACC (\%)}} & 
\multicolumn{1}{c}{\textbf{AUC} \quad \textbf{AP}} & 
\multicolumn{1}{c}{\textbf{ACC (\%)}} & 
\multicolumn{1}{c}{\textbf{AUC} \quad \textbf{AP}} \\ 
\midrule[1pt]

Random & 
13.82 & 0.4917\quad0.1423 & 
13.70 & 0.5009\quad0.1357 & 
14.57 & 0.5054\quad0.1390 \\

DLG & 
14.95 & 0.5208\quad0.1465 & 
18.06 & 0.5149\quad0.1410 & 
14.33 & 0.5243\quad0.1440 \\

iDLG & 
17.20 & 0.5339\quad0.1498 & 
15.13 & {\ul 0.5394}\quad {\ul 0.1466} & 
14.70 & {\ul 0.5427}\quad0.1476 \\

InverGrad & 
23.10 & {\ul 0.5428}\quad {\ul 0.1532} & 
21.66 & 0.5367\quad0.1504 & 
22.03 & 0.5379\quad {\ul 0.1488} \\

GI-GAN & 
{\ul 45.52} & 0.5059\quad0.1240 & 
{\ul38.11} & 0.5046\quad0.1237 & 
{\ul44.88} & 0.5085\quad0.1251\\

GRA-GRF & 
19.50 &0.4893\quad0.1230 & 
20.89 & 0.4970\quad0.1250 & 
17.60 & 0.4852\quad0.1220 \\

\textbf{\sysname } &
\textbf{61.87} & \textbf{0.9042}\quad \textbf{0.7554} & 
\textbf{68.08} & \textbf{0.9043}\quad \textbf{0.7558} & 
\textbf{65.10} & \textbf{0.9042}\quad \textbf{0.7557} \\
\bottomrule[1.5pt]
\end{tabular}
}
\vspace{-1em}
\end{table}

\fakeparagraph{Studies on Deeper GCN Model}
\tabref{tab:layer} reports results using deeper GCN models with 3 and 4 layers on MUTAG dataset.
The results show that the proposed GraphDLG consistently outperforms other baselines in both feature and structure recovery, reducing the MSE of node features by over 19.54\% and improving the AUC of graph structures by more than 67.35\%, \ie our approach still achieves the best recovery performance for graph-tailored DLG.

\begin{table}[htb]
\vspace{-1em}
\caption{Performance with deeper GCN models.}
\vspace{-1em}
\label{tab:layer}
\centering
\resizebox{0.9\columnwidth}{!}{
\setlength{\tabcolsep}{0.5pt}
\begin{tabular}{@{}c|cc|cc@{}}
\toprule[1.5pt]
\multirow{3}{*}{\textbf{Method}} & 
\multicolumn{2}{c}{\textbf{3-layer-GCN}} & 
\multicolumn{2}{c}{\textbf{4-layer-GCN}} \\
\cmidrule(lr){2-3} \cmidrule(lr){4-5}

& 
\multicolumn{1}{c}{\textbf{Node Feature}} & 
\multicolumn{1}{c}{\textbf{Graph Structure}} & 
\multicolumn{1}{c}{\textbf{Node Feature}} & 
\multicolumn{1}{c}{\textbf{Graph Structure}} \\

& 
\multicolumn{1}{c}{\textbf{MSE} \quad \textbf{ACC (\%)}} & 
\multicolumn{1}{c}{\textbf{AUC} \quad \textbf{AP} \quad \textbf{ACC (\%)}} & 
\multicolumn{1}{c}{\textbf{MSE} \quad \textbf{ACC (\%)}} & 
\multicolumn{1}{c}{\textbf{AUC} \quad \textbf{AP} \quad \textbf{ACC (\%)}} \\ 
\midrule[1pt]

Random & 
0.3430 \quad 11.99 & 0.4896 \quad 0.1325 \quad 49.51 &
0.3416 \quad 13.29 & 0.4996 \quad 0.1382 \quad 49.89 \\

DLG & 
1.1412 \quad 21.45 & 0.5344 \quad 0.1423 \quad 65.06 &
1.2389 \quad 19.62 & {\ul 0.5269} \quad {\ul 0.1412} \quad {\ul 64.89} \\

iDLG & 
1.1519 \quad 16.76 & 0.5371 \quad 0.1481 \quad {\ul 65.67} &
1.2162 \quad 18.28 & 0.5176 \quad 0.1386 \quad 64.62 \\

InverGrad & 
1.0014 \quad 22.96 & {\ul 0.5397} \quad {\ul 0.1490} \quad 56.62 &
1.0559 \quad 22.51 & 0.5102 \quad 0.1396 \quad 58.43 \\

GI-GAN & 
1.1054 \quad 51.89 & 0.5091 \quad 0.1248 \quad 14.03 &
{\ul 0.1054} \quad 51.99 & 0.5092 \quad 0.1248 \quad 14.06 \\

GRA-GRF & 
{\ul 0.1079} \quad {\ul 71.06} & 0.4831 \quad 0.1231 \quad 49.19 &
0.1080 \quad {\ul 73.31} & 0.4922 \quad 0.1237 \quad 50.36 \\

\textbf{\sysname} & 
\textbf{0.0683} \quad \textbf{74.34} & \textbf{0.9032} \quad \textbf{0.7562} \quad \textbf{93.41} &
\textbf{0.0848} \quad \textbf{74.33} & \textbf{0.9022} \quad \textbf{0.7556} \quad \textbf{93.41} \\
\bottomrule[1.5pt]
\end{tabular}
}
\vspace{-1em}
\end{table}

\fakeparagraph{Studies on Larger Hidden Size}
\tabref{tab:hidden} shows the results using larger hidden sizes on MUTAG dataset.
We further conducted experiments with larger hidden sizes of 64 and 128, respectively.
The results are as follows.
It is evident that our proposed GraphDLG maintains the best feature and structure recovery performance compared with all baselines, reducing node features' MSE by over 8.07\% and increasing graph structure's AUC by over 64.57\%.

\begin{table}[htb]
\vspace{-1em}
\caption{Performance comparison using larger hidden sizes.}
\label{tab:hidden}
\vspace{-1em}
\centering
\resizebox{0.9\columnwidth}{!}{
\setlength{\tabcolsep}{0.5pt}
\begin{tabular}{@{}c|cc|cc@{}}
\toprule[1.5pt]
\multirow{3}{*}{\textbf{Method}} & 
\multicolumn{2}{c}{\textbf{Hidden size = 64}} & 
\multicolumn{2}{c}{\textbf{Hidden size = 128}} \\
\cmidrule(lr){2-3} \cmidrule(lr){4-5}

& 
\multicolumn{1}{c}{\textbf{Node Feature}} & 
\multicolumn{1}{c}{\textbf{Graph Structure}} & 
\multicolumn{1}{c}{\textbf{Node Feature}} & 
\multicolumn{1}{c}{\textbf{Graph Structure}} \\

& 
\multicolumn{1}{c}{\textbf{MSE} \quad \textbf{ACC (\%)}} & 
\multicolumn{1}{c}{\textbf{AUC} \quad \textbf{AP} \quad \textbf{ACC (\%)}} & 
\multicolumn{1}{c}{\textbf{MSE} \quad \textbf{ACC (\%)}} & 
\multicolumn{1}{c}{\textbf{AUC} \quad \textbf{AP} \quad \textbf{ACC (\%)}} \\ 
\midrule[1pt]

Random & 
0.3316 \quad 14.77 & 0.4789 \quad 0.1344 \quad 49.18 &
0.3337 \quad 15.05 & 0.4912 \quad 0.1331 \quad 49.67 \\

DLG & 
1.1304 \quad 20.46 & 0.5283 \quad 0.1469 \quad {\ul 65.88} &
1.0489 \quad 21.48 & 0.5236 \quad 0.1444 \quad 62.88 \\

iDLG & 
1.1228 \quad 20.29 & 0.5425 \quad 0.1577 \quad 65.57 &
1.1215 \quad 19.47 & 0.5265 \quad 0.1471 \quad {\ul 64.97} \\

InverGrad & 
0.9400 \quad 25.09 & {\ul 0.5492} \quad {\ul 0.1615} \quad 54.70 &
0.8657 \quad 25.41 & {\ul 0.5380} \quad {\ul 0.1561} \quad 49.11 \\

GI-GAN & 
{\ul 0.1053} \quad 50.33 & 0.5108 \quad 0.1254 \quad 14.85 &
{\ul 0.1050} \quad 52.37 & 0.5020 \quad 0.1234 \quad 14.42 \\

GRA-GRF & 
0.1078 \quad {\ul 71.78} & 0.4813 \quad 0.1210 \quad 52.19 &
0.1077 \quad {\ul 71.72} & 0.4831 \quad 0.1216 \quad 53.38 \\

\textbf{\sysname} & 
\textbf{0.0968} \quad \textbf{73.82} & \textbf{0.9038} \quad \textbf{0.7560} \quad \textbf{93.41} &
\textbf{0.0870} \quad \textbf{74.34} & \textbf{0.9041} \quad \textbf{0.7557} \quad \textbf{93.41} \\

\bottomrule[1.5pt]
\end{tabular}
}
\vspace{-1.5em}
\end{table}

\subsection{Potential Defense Methods}
\label{sec:app_defense}

We further conducted experiments to study the potential defense strategies against the proposed attack.
\textit{DP-Gradients}: it follows the classical differential privacy mechanism by adding Laplacian noise with variance $\sigma$ to model gradients
\textit{(ii) DP-Activations}: it injects Laplacian noise with variance $\sigma$ into the activations (\ie, layer outputs) of GCN layers instead of the gradients.
\textit{(iii) DP-Hybrid}: it is a combined scheme that applies Laplacian noise with variance $\sigma/2$ to both gradients and activations, respectively.
The experimental results summarized in \tabref{tab:defense}, where we highlight the best result in bold and the second best result with the symbol `*'.

\fakeparagraph{Experimental Observations}
Firstly, the DP-Gradients provides stronger protection for node features than for graph structures, which is consistent with the results shown in \figref{fig:defense}.
Nevertheless, when $\sigma$ is large (\eg, $\sigma=0.5$), it can also offer meaningful protection for the graph structure.
Secondly, the DP-Activations is more effective in protecting graph structure, likely because the graph structure plays a central role in graph convolution operations, where injecting noise into intermediate activations directly disrupts the structure-related signals. 
Finally, the DP-Hybrid appear to be a promising defense strategy, as it combines the advantages of both DP-Gradients and DP-Activations, \ie, the DP-Hybrid can provide comprehensive protection for both node features and graph structures.
\textit{In summary, injecting noise into both gradients and activations simultaneously can serve as an effective defense strategy against deep leakage from gradients in federated graph learning.}

\begin{table}[bht]
\vspace{-1em}
\caption{Results on Three Potential Defense Methods with Varying Noise Levels ($\sigma$).}
\label{tab:defense}
\vspace{-1em}
\centering
\resizebox{0.9\columnwidth}{!}{
\setlength{\tabcolsep}{0.5pt}
\begin{tabular}{@{}c|cc|cc|cc@{}}
\toprule[1.5pt]
\multirow{3}{*}{\textbf{Defense}} & 
\multicolumn{2}{c}{\textbf{DP-Gradients}} & 
\multicolumn{2}{c}{\textbf{DP-Activations}} & 
\multicolumn{2}{c}{\textbf{DP-Hybrid}} \\
\cmidrule(lr){2-3} \cmidrule(lr){4-5} \cmidrule(lr){6-7}

& 
\multicolumn{1}{c}{\textbf{Feature}} & 
\multicolumn{1}{c}{\textbf{Structure}} & 
\multicolumn{1}{c}{\textbf{Feature}} & 
\multicolumn{1}{c}{\textbf{Structure}} & 
\multicolumn{1}{c}{\textbf{Feature}} & 
\multicolumn{1}{c}{\textbf{Structure}} \\

& 
\multicolumn{1}{c}{\textbf{ACC (\%)}} & 
\multicolumn{1}{c}{\textbf{AUC} \quad \textbf{AP}} & 
\multicolumn{1}{c}{\textbf{ACC (\%)}} & 
\multicolumn{1}{c}{\textbf{AUC} \quad \textbf{AP}} & 
\multicolumn{1}{c}{\textbf{ACC (\%)}} & 
\multicolumn{1}{c}{\textbf{AUC} \quad \textbf{AP}} \\ 
\midrule[1pt]

w/o DP & 
74.35 & 0.9042\quad0.7556 & 
74.35 & 0.9042\quad0.7556 &
74.35 & 0.9042\quad0.7556 \\

$\sigma$=0.05 & 
\textbf{63.50} & 0.9045\quad0.7564 & 
74.34 & 0.8994*\quad0.7358* & 
67.99* & \textbf{0.8843}\quad\textbf{0.7079} \\

$\sigma$=0.1 & 
\textbf{36.50} & 0.8931\quad0.7296 & 
74.34 & 0.8829*\quad0.6926* & 
45.49* & \textbf{0.8546}\quad\textbf{0.6618} \\

$\sigma$=0.2 & 
\textbf{25.22} & 0.8820\quad 0.7204 & 
74.34 & 0.8485*\quad0.6351* & 
29.26* & \textbf{0.8328}\quad \textbf{0.6313} \\

$\sigma$=0.5 & 
\textbf{20.67} & 0.8272*\quad0.6230* & 
74.34 & \textbf{0.8131}\quad\textbf{0.5748} & 
20.75* & 0.8344\quad0.6404\\

\bottomrule[1.5pt]
\end{tabular}
}
\vspace{-1em}
\end{table}

\end{document}